\documentclass[12pt]{article}

\pdfoutput=1

\usepackage{latexsym,amsfonts,bm,epsfig,amsmath,natbib,authblk,amsthm,thmtools,amssymb}

\usepackage{float}
\usepackage{booktabs}
\usepackage{graphicx}
\usepackage{subcaption}
\usepackage[margin=.5cm]{caption}
\usepackage{xr-hyper}
\usepackage{color}
\usepackage[colorlinks,linkcolor=black,citecolor=black,filecolor=black,bookmarks=false,pagebackref]{hyperref}
\usepackage[algoruled, noend]{algorithm2e}

\usepackage{pgfplots}
\usepackage{tikz}
\usepackage{placeins}
\usepackage{setspace}
\newlength\figureheight
\newlength\figurewidth

\floatstyle{plain}
\makeatletter
\renewcommand{\algocf@captiontext}[2]{#1\algocf@typo. \AlCapFnt{}#2} % text of caption
\def\@algocf@capt@plain{top}
\renewcommand{\algocf@makecaption}[2]{%
	\addtolength{\hsize}{\algomargin}%
	\sbox\@tempboxa{\algocf@captiontext{#1}{#2}}%
	\ifdim\wd\@tempboxa >\hsize%     % if caption is longer than a line
	\hskip .5\algomargin%
	\parbox[t]{\hsize}{\algocf@captiontext{#1}{#2}}% then caption is not centered
	\else%
	\global\@minipagefalse%
	\hbox to\hsize{\box\@tempboxa}% else caption is centered
	\fi%
	\addtolength{\hsize}{-\algomargin}%
}
\makeatother

\def\tr{\text{\rm tr}}

\newcommand{\ignore}[1]{}

\newcommand{\wh}{\widehat}

\newcommand{\z}{z}
\newcommand{\dt}{\mathrm{d}}
 
\newcommand{\E}{\mathbb{E}}
\newcommand{\KL}{\mathrm{KLD}}

\newcommand{\argmin}{\operatornamewithlimits{argmin}}

\newtheorem{assumption}{Assumption}
\newtheorem{theorem}{Theorem}

\newtheorem{corollary}{Corollary}
\newtheorem{lemma}{Lemma}
\newtheorem{defn}{Definition}
\theoremstyle{remark}
\newtheorem{remark}{Remark}

\begin{document}
	\def\spacingset#1{\renewcommand{\baselinestretch}%
		{#1}\small\normalsize} \spacingset{1}
	
	\title{The Statistical Accuracy of Neural Posterior and Likelihood Estimation}
	\date{\empty}
%	\author{}
\author[1]{David T. Frazier\thanks{Corresponding author:  david.frazier@monash.edu}}
	\author[4]{Ryan Kelly}
	\author[4]{Christopher Drovandi}
	\author[4]{David J. Warne}
\affil[1]{Department of Econometrics and Business Statistics, Monash University, Clayton VIC 3800, Australia}
	%\affil[2]{Department of Statistics and Applied Probability, National University of Singapore, Singapore 117546}
	%\affil[3]{Operations Research and Analytics Cluster, National University of Singapore, Singapore 119077}
	\affil[4]{School of Mathematical Sciences, Queensland University of Technology, Brisbane 4000 Australia}
	
	\maketitle

	\begin{abstract}
Neural posterior estimation (NPE) and neural likelihood estimation (NLE) are machine learning approaches that provide accurate posterior, and likelihood, approximations in complex modeling scenarios, and in situations where conducting amortized inference is a necessity. While such methods have shown significant promise across a range of diverse scientific applications, the statistical accuracy of these methods is so far unexplored. In this manuscript, we give, for the first time, an in-depth exploration on the statistical behavior of NPE and NLE. We prove that these methods have similar theoretical guarantees to common statistical methods like approximate Bayesian computation (ABC) and Bayesian synthetic likelihood (BSL). While NPE and NLE methods are just as accurate as ABC and BSL, we prove that this accuracy can often be achieved at a vastly reduced computational cost, and will therefore deliver more attractive approximations than ABC and BSL in certain problems. We verify our results theoretically and in several examples from the literature.
			\vspace{1cm}

	%	\noindent \textbf{Keywords.}  Approximate Bayesian computation. Model misspecification. Likelihood tempering. 
	\end{abstract}
	\spacingset{1.9} % DON'
	\section{Introduction}

    Neural conditional density approximation approaches have emerged as powerful machine learning techniques for conducting simulation-based Bayesian inference (SBI) in complex models with intractable or computationally expensive likelihoods.
    These methods can approximate the likelihood through neural likelihood estimation (NLE) \citep{papamakarios2019sequential} or directly target the posterior distribution with neural posterior estimation (NPE) \citep{greenberg2019automatic, lueckmann2017flexible, papamakarios2016fast}, with NLE requiring subsequent Markov Chain Monte Carlo (MCMC) steps to produce posterior samples. The hallmark of these neural methods is their ability to accurately approximate complex posterior distributions using only forward simulations from the assumed model. While sequential methods iteratively refine the posterior estimate through multiple rounds of simulation, one-shot NPE and NLE methods perform inference in a single round, enabling amortized inference where a trained model can be reused for multiple datasets without retraining (see, e.g., \citealp{radev2020bayesflow}; \citealp{pmlr-v235-gloeckler24a}). In particular, like the statistical methods of approximate Bayesian computation (ABC), see, e.g., \cite{sisson2018handbook} for a handbook treatment, and \cite{martin2023approximating} for a recent summary, and Bayesian synthetic likelihood (BSL), see, e.g., \cite{wood2010statistical}, \cite{price2018bayesian} and \cite{frazier2019bayesian}, NPE and NLE first reduce the data down to a vector of statistics and then build an approximation to the resulting partial posterior by substituting likelihood evaluation with forward simulation from the assumed model. 
	
	In contrast to the statistical methods for likelihood-free inference like ABC and BSL, NPE (respectively, NLE) approximates the posterior (resp., the likelihood) directly by fitting flexible conditional density estimators, usually neural- or flow-based approaches, using training data that is simulated from the assumed model space. The approximation that results from this training step is then directly used as a posterior in the context of NPE or as a likelihood in the case of NLE, with MCMC for this trained likelihood then used to produce draws from an approximate posterior. %In contrast to NPE and NLE, which can train these approximations on data that lie anywhere in the assumed model space, ABC and BSL only learn the behavior learn about the unknown values of the parameters from simulated summaries that are a ``close approximation'' to the observed summaries.}
	
	While NPE/NLE methods have been shown to be just as accurate as ABC and BSL in certain experiments, while requiring a greatly reduced computational cost \citep{lueckmann2021benchmarking}, virtually nothing is known about the underlying statistical behavior of these methods. Herein, we make three fundamental contributions to the literature on SBI methods. Our first contribution is to provide the first rigorous study on the statistical behavior of NPE and NLE methods, and to rigorously show that these methods can deliver posterior approximations that have well-behaved statistical properties such as posterior concentration and Bernstein-von Mises phenomena. 
 
 Our second contribution is to use our theoretical results to analyze the validity of the existing hypothesis in the machine learning literature that NPE/NLE methods are more accurate than statistical methods like ABC and BSL for a given computational budget (see, e.g., \citealp{lueckmann2021benchmarking}). We rigorously prove that, under certain tuning regimes, and similar regularity conditions to methods like ABC, asymptotically NPE/NLE  can be as accurate as ABC, but the accuracy of NPE/NLE can be far less impacted by the dimension of the summaries and parameters than methods like ABC.\footnote{We later show that it is due to the ability of NPE/NLE to directly control the volume of training data used and the specific approximation being employed that allows these approaches to partly offset the curse of dimensionality associated with likelihood-free methods.} This latter finding is particularly important as it is well-known that ABC encounters a curse of dimensionality in the number of summaries used in the analysis (see, e.g., \citealp{blum2010non} for a discussion). 

 Our last contribution is to use these theoretical results to compare the behavior of NPE and NLE. The current literature exhibits a preference for NPE methods in most empirical examples, but to date no theoretical comparison of the two methods exist. We show that, under similar regularity conditions, NPE and NLE methods have similar statistical behavior. Hence, recalling that NLE must first estimate the likelihood, and then use this within some form of MCMC sampling, our results suggest that,  NPE is preferable to NLE since it has similar theoretical guarantees but does not require further MCMC sampling to produce posterior draws. 
 
The remainder of the paper is organized as follows. Section \ref{sec:motiv} contains a motivating example that highlights the importance of studying the behavior of NPE/NLE methods. Section \ref{sec:one} gives the general setup needed to study NPE methods, while Section \ref{sec:snpe} gives an account of the theoretical behavior of NPE, and Section \ref{sec:specific} contains discussion on several features that drive the statistical behavior of NPE. Section \ref{sec:snl} gives the setup and theoretical results for NLE methods, and Section \ref{sec:discuss} concludes the paper with directions for further study. 

\section{Motivating Example}\label{sec:motiv}
Before presenting our theoretical results we briefly motivate our analysis by numerically illustrating one of our main findings with a simulated data example. To date, the accuracy of NPE methods as a function of the number of training datasets used in fitting the approximation, denoted generically as $N$, and relative to the number of observed data points, denoted by $n$, has not been explored in the literature. In most cases, applications of NPE use a few thousand training datasets regardless of the example under study. In the following example, we show if NPE is to deliver accurate posterior inferences, then the number of model simulations used in training, $N$, must diverge to infinity faster than the number of observed data points; in subsequent sections we formally show that the precise rate at which $N$ must diverge depends on $n$ and the dimension of the problem. However, before giving formal results on the link between $N$ and $n$, we explore this link empirically in a commonly encountered example from the SBI literature.

\subsection{Stereological Example}
During the production of steel or other metals, microscopic particles, called inclusions, can occur within the metal that fundamentally weaken the structure. As fatigue is often believed to develop from inclusions within the material, over time the number of inclusions within the metal has become a useful gauge with which to measure the quality of metal production. 

To model the number of inclusions in a given block of steel, \cite{bortot2007inference} used a model of stereological extremes that generates a random number of elliptical inclusions, where the number of inclusions is governed by a homogeneous Poisson process with rate $\lambda$, and where the shape and scale size of the inclusion is governed by a Pareto distribution with shape and size parameters ($\sigma,\xi$).\footnote{See Appendix \ref{stereo} for further details on this model.} 

While relatively simple to explain, the resulting likelihood for this model is intractable, and inference is generally carried out using SBI methods. In this example, we show that in order for NPE to deliver accurate posterior inferences, we must choose the number of model simulations $N$ such that as $n\rightarrow\infty$, $n/N\rightarrow0$. To this end, in the experiments that follow we consider implementing NPE according to four possible choices of $N$: $N\in\{n,n\log(n),n^{3/2},n^2\}$. The specific reasoning behind these choices are explained in detail within Section \ref{sec:specific}. 

To keep the analysis as simple as possible, we only present results for the rate parameter $\lambda$, while results for $\sigma$ and $\xi$ are reported in Appendix~\ref{stereo}. %{\color{red}parameters: xxxx and the number of inclusions within the dataset.} 
First, we compare the accuracy of NPE under our choices of $N$ against the ABC posterior on a single dataset, where the ABC posterior is sampled using the ABC-SMC algorithm in \citet{simola2021adaptive}.
These results are plotted visually in Figure \ref{fig:stereo_lambda_posterior}. 
On these results, the ABC-SMC algorithm required over ten million simulations, which is more than an order of magnitude greater than the highest number of simulations considered for NPE at $N=n^2$ with $n=1000$. Since, in this example, simulations dominate the computation time, this results in an order of magnitude longer runtime for ABC-SMC compared to NPE.
These results demonstrate that NPE is quite close to the benchmark ABC posterior for larger choices of $N$, but for $N< n$ the NPE is much more diffuse than the ABC posterior. However, once we let $N> n$, the accuracy of the NPE increases dramatically. 

To further illustrate the increase in accuracy of the NPE that results from taking $N>n$, in Figure~\ref{fig:stereo_lambda_posterior} we plot the bias of the posterior mean from NPE across 100 replicated datasets and analyze the behavior of this bias as $N$ and $n$ increase. The results suggest that both the bias and the variance of the posterior mean decrease as $N$ becomes larger for any fixed $n$. However,  the results are roughly similar for $N\in \{n^{3/2},n^2\}$. 

\begin{figure}[htbp]
    \centering
    \includegraphics[width=0.6\textwidth]{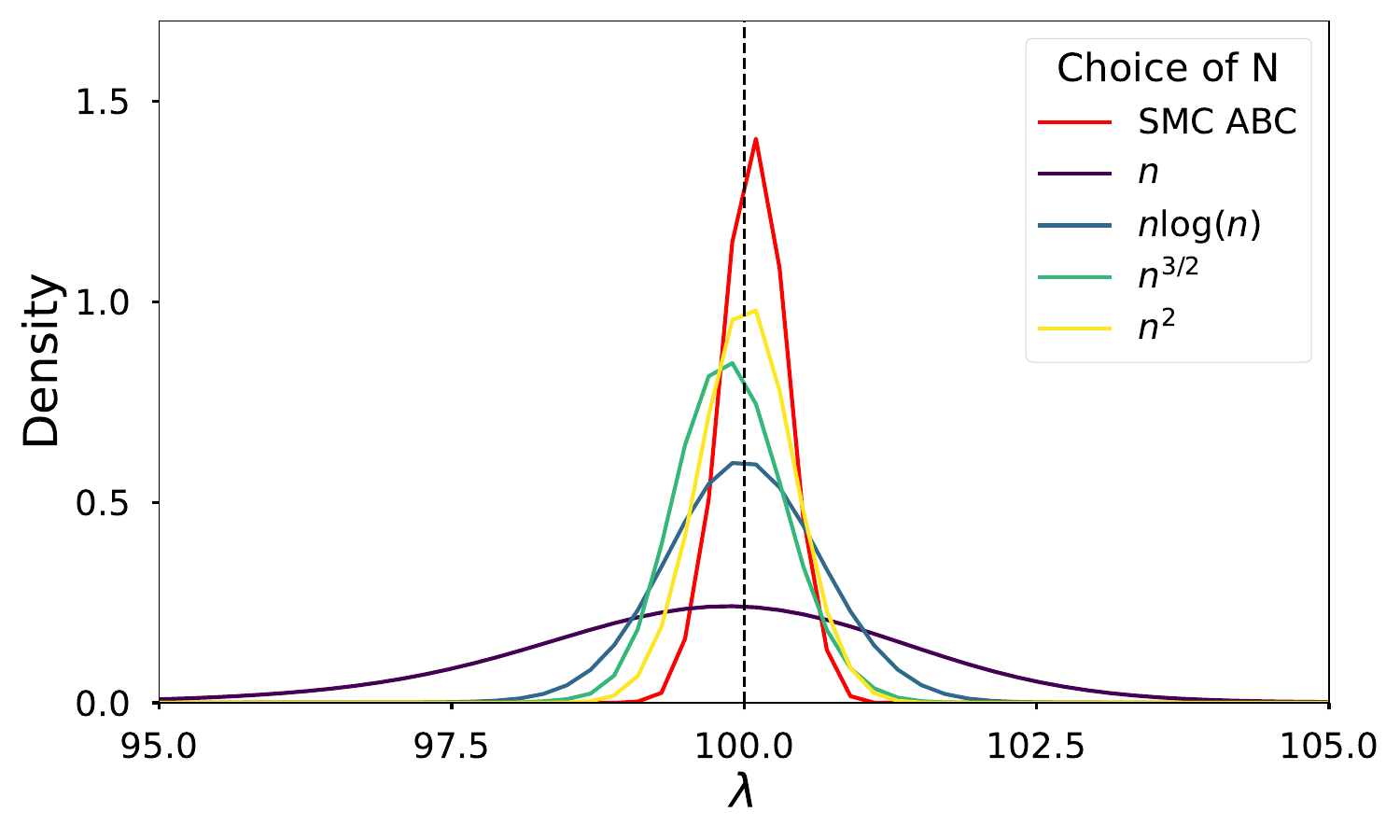}
    \caption{Univariate posterior approximations of the rate parameter \( \lambda \) (true value \( \lambda = 100 \), shown by the dashed vertical line) for a dataset generated from the stereological model with \( n = 1000 \) observations. The NPE approximations are compared across varying numbers of simulations against the posterior obtained via ABC-SMC.}
    \label{fig:stereo_lambda_posterior}
\end{figure}

In Table \ref{tab:stereo_coverage_lambda}, across 100 replicated runs, we compare the realized coverage of the corresponding NPE for $\lambda$ across different choices of $N$ and different sample sizes $n$. The results in Table \ref{tab:stereo_coverage_lambda} clearly show that the most accurate levels of coverage are obtained by choosing $N\ge n^{3/2}$ regardless of the value of $n$. 

Panels (a)-(d) in Figure \ref{fig:stereo_boxplot_all} and the results in Table \ref{tab:stereo_coverage_lambda} suggest that concentration of the NPE requires that $N$ diverges sufficiently fast, as a function of $n$, so the impact of using simulated data within the approximation is smaller than the natural variability in the posterior we are attempting to target. In the following sections, we formalize this link between $N$ and $n$ and discuss the requirements on $N$ needed to ensure that the NPE concentrates at the standard parametric rate. We also show that if the summaries are not chosen carefully, or if the NPE is not trained well-enough, then there is no reason to suspect that the accuracy of the NPE shown in this example will remain valid.

\begin{table}[htbp]
\centering
\begin{tabular}{@{}l l l l l @{} }\toprule
 & $N=n$ &  $N=n\log(n)$ & $N=n^{3/2}$ & $N=n^2$ \\ 
\hline
$n=100$ &1.00/1.00/1.00 &0.93/0.99/1.00 &0.81/0.94/0.99 &0.81/0.91/0.96 \\ 
$n=500$ &0.97/1.00/1.00 &0.87/0.96/0.99 &0.83/0.93/0.97 &0.82/0.91/0.96 \\ 
$n=1000$ &0.97/1.00/1.00 &0.87/0.95/0.99 &0.84/0.93/0.98 &0.80/0.90/0.95 \\ 
$n=5000$ &0.93/0.98/1.00 &0.89/0.96/0.99 &0.86/0.95/0.98 &0.81/0.92/0.96 \\ 
\bottomrule
\end{tabular}
\caption{
Monte Carlo coverage of 80\%, 90\%, and 95\% credible intervals for $\lambda$ across different choices of $N$ and sample sizes $n$, based on 100 repeated runs.}
\label{tab:stereo_coverage_lambda}

\end{table}

%Bias plots.
\begin{figure}[htbp]
    \centering
    \begin{subfigure}[b]{0.49\textwidth}
        \centering
        \includegraphics[width=\textwidth]{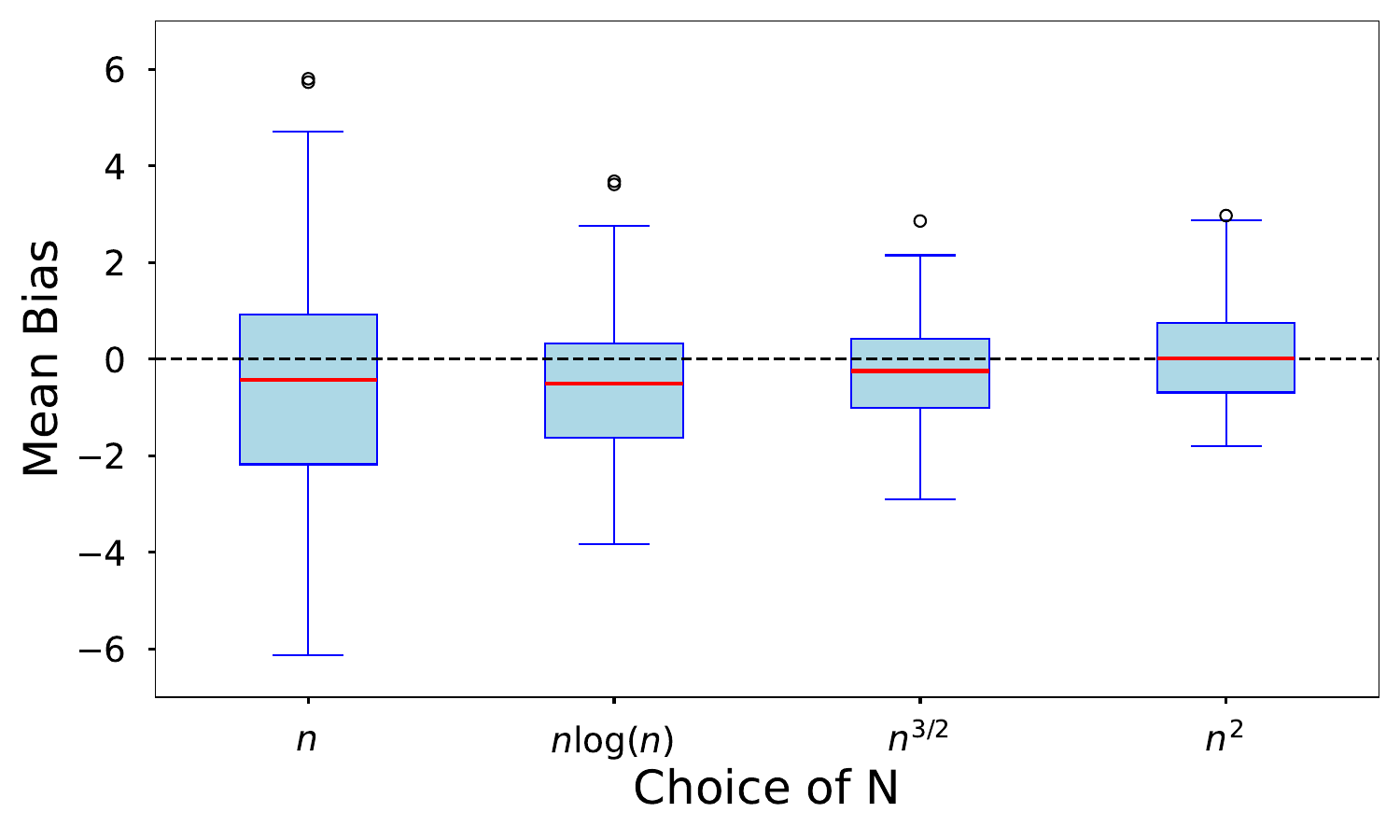}
        \caption{$n=100$.}
    \end{subfigure}
    \hfill
    \begin{subfigure}[b]{0.49\textwidth}
        \centering
        \includegraphics[width=\textwidth]{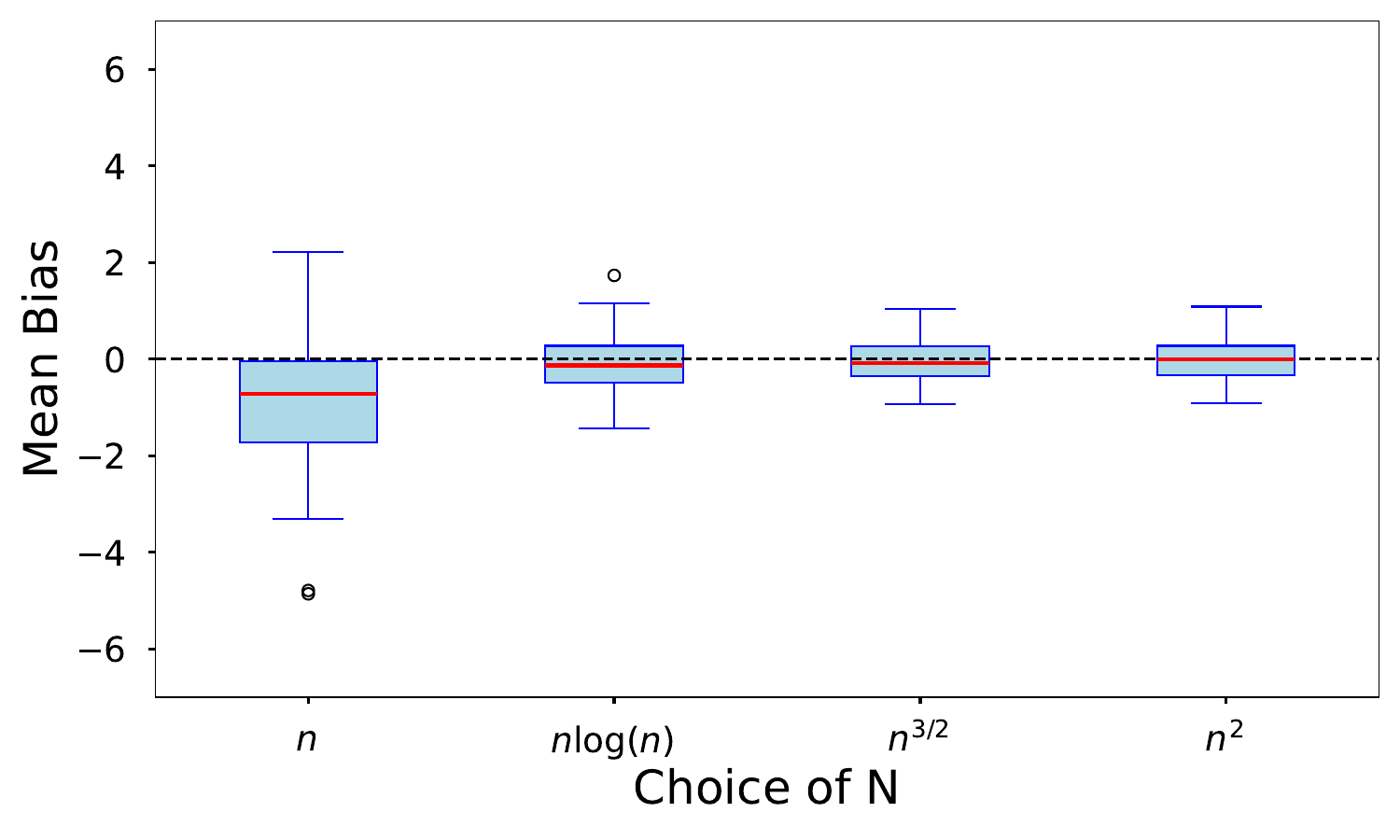}
        \caption{$n=500$.}
    \end{subfigure}
    
    \vspace{5pt}
    
    \begin{subfigure}[b]{0.49\textwidth}
        \centering
        \includegraphics[width=\textwidth]{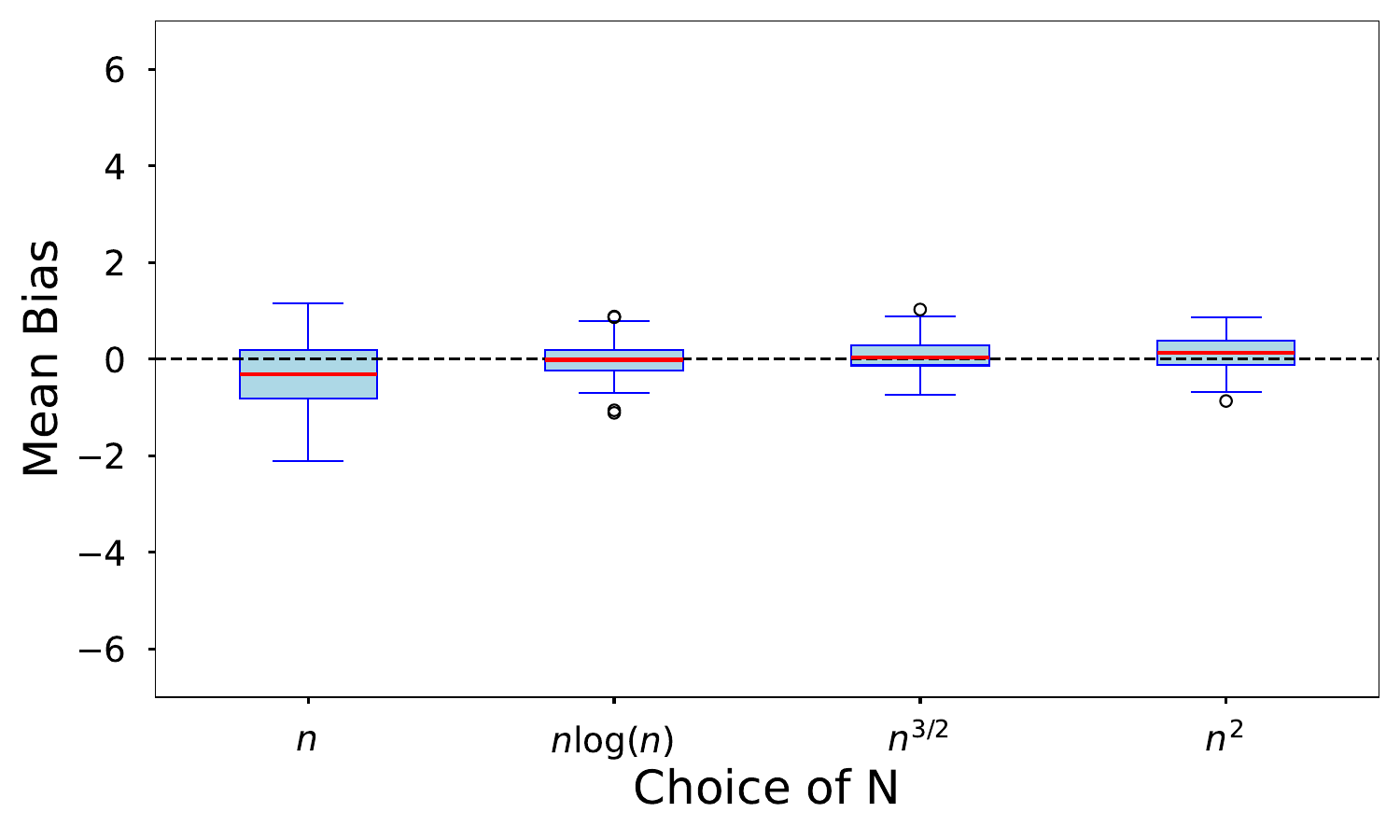}
        \caption{$n=1000$.}
    \end{subfigure}
    \hfill
    \begin{subfigure}[b]{0.49\textwidth}
        \centering
        \includegraphics[width=\textwidth]{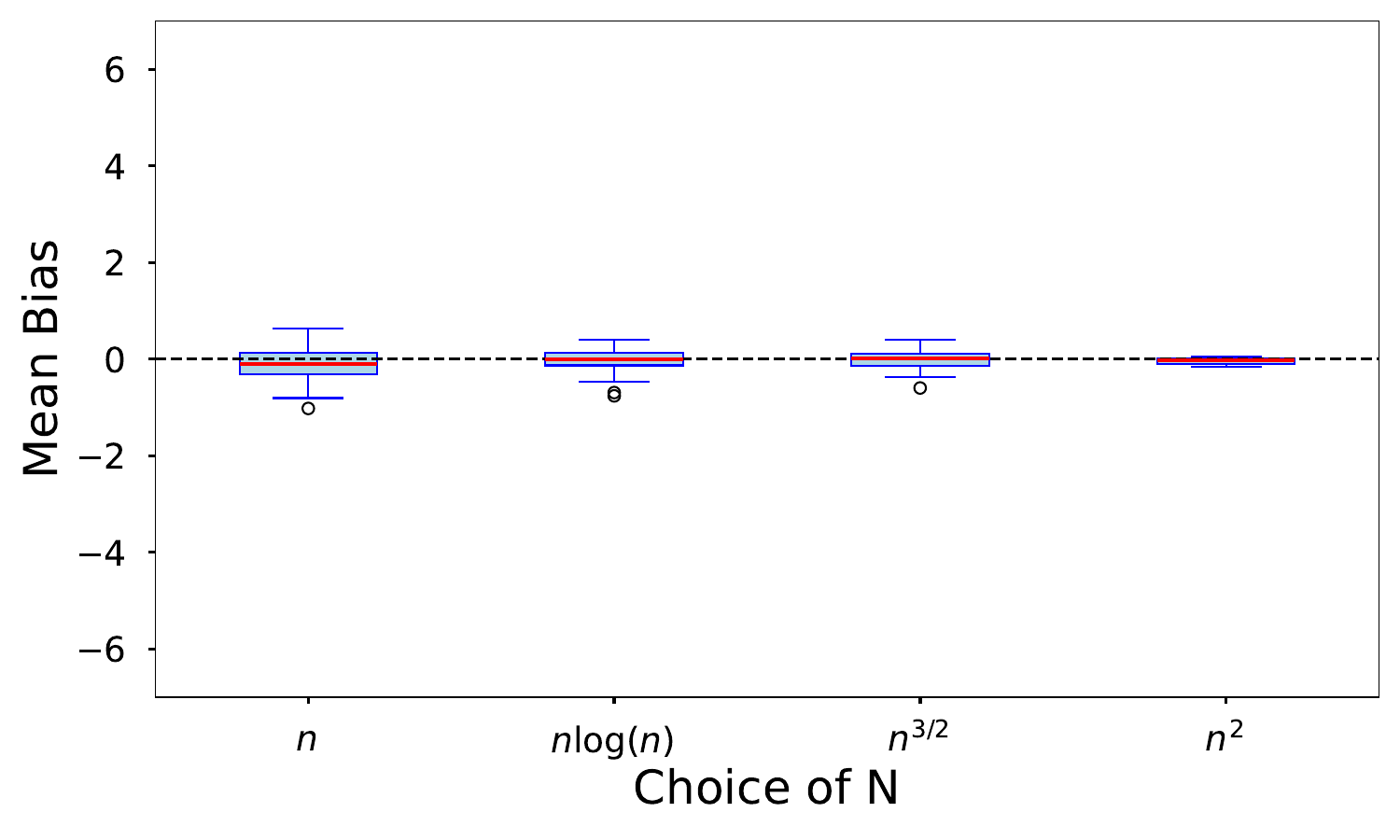}
        \caption{$n=5000$.}
    \end{subfigure}
    
    \caption{Bias of the posterior mean for $\lambda$ visualized through boxplots across varying $n$ and $N$.}
    \label{fig:stereo_boxplot_all}
\end{figure}

    \FloatBarrier

	\section{Setup: Neural Posterior Estimation}\label{sec:one}

	Let $y=(y_1,\dots,y_n)$ denote the observed data and define $P^{(n)}_{0}$ as the true distribution of $y$. We assume the observed data is generated from an intractable class of parametric models  $\{P^{(n)}_\theta:\theta\in\Theta\subseteq\mathbb{R}^{d_\theta}\}$ from which we can easily simulate pseudo-data $\z=(z_1,\dots,z_n)$ for any $\theta\in\Theta$. Denote the prior density for $\theta$ as $p({\theta })$.
	
	Since the likelihood is intractable, we conduct inference using approximate Bayesian methods, also known as simulation-based Bayesian inference (SBI). The main idea of SBI is to build an approximation to the posterior by leveraging our ability to simulate data under $P_\theta^{(n)}$. To make the problem computationally practical the approximation is generally carried out using summaries rather than the full dataset $y$. Let $S_n:\mathbb{R}^n\rightarrow\mathcal{S}$, $\mathcal{S}\subseteq\mathbb{R}^{d_s}$, $d\ge d_\theta$,  denote the vector summary statistic mapping employed in the analysis. When no confusion is likely to result, we write $S_n(\cdot)$ for the mapping and $S_n$ its value when $S_n(\cdot)$ is evaluated at the observed data $y$, i.e., $S_n=S_n(y)$. 
	
	Denote by $G_n(\cdot\mid \theta)$ the distribution of $S_n(z)$ under $P_\theta^{(n)}$. Since we have reduced the data down to $S_n$, the inferential ideal shifts from the standard Bayes posterior to the ``partial'' posterior:
	$$
\dt\Pi(\theta\mid S_n)=\frac{\dt G_n(S_n\mid \theta)p(\theta)}{\int_\Theta \dt G_n(S_n\mid \theta)p(\theta)\dt\theta}.
	$$When $G_n(\cdot\mid\theta)$ admits a density with respect to (wrt) the Lebesgue measure, we let $g_n(\cdot\mid\theta)$ denotes its density, and state the posterior more compactly as $\pi(\theta\mid S_n)\propto g_n(S_n\mid\theta)p(\theta)$. Throughout the remainder, to simplify notations, we use this more compact notation.\footnote{We note here that our analysis can be extended to cases where $G_n$ does not admit a density at the cost of additional notations and technicalities.} 
	
	Since $\Pi(\cdot\mid S_n)$ is intractable, we resort to conducting inference via an approximate posterior.  Two classes of powerful posterior approximations used in the machine learning literature are neural posterior estimation (NPE) (see, e.g., \citealp{lueckmann2017flexible}), and neural likelihood (NLE) (see, e.g., \citealp{papamakarios2019sequential}). 
	
	NPE learns an approximation to $\Pi(\cdot\mid S)$ by simulating data pairs $(\theta^\top,S_n(z)^\top)^\top\stackrel{iid}{\sim} p(\theta)g_n(S\mid\theta)$ and then minimizing a certain criterion function over a chosen class of conditional density approximations for $\theta\mid S$ (more precise details are given in Definition \ref{def:snpe}, Section \ref{sec:snpe}). The goal of NPE is not to directly approximate  $\Pi(\cdot\mid S_n)$, but to approximate the entire distribution of $\theta\mid S$, i.e., $\Pi(\cdot\mid S)$ at any $S\in\mathcal{S}$. This last point is crucial to note and is a consequence of the training data used in NPE, which does not include the observed summaries $S_n$.  Algorithm \ref{alg:npe} provides an algorithmic formulation of NPE.

	In contrast to NPE, NLE learns an approximation of the intractable likelihood $g_n(\cdot\mid\theta)$ by simulating pairs $(\theta^\top,S_n(z)^\top)^\top\stackrel{iid}{\sim}p(\theta)g_n(S\mid\theta)$, and then minimizing a certain criterion function over a chosen class of conditional density approximations for $S\mid \theta$ (see Definition \ref{def:snl} in Section \ref{sec:snl} for further details). This approximation replaces the intractable likelihood $g_n(S_n\mid\theta)$ in $\pi(\theta\mid S_n)$ and MCMC is used to produce draws from an approximate posterior. Algorithm \ref{alg:snl} gives the single round version of the neural likelihood posterior estimator (NLE).	

    \begingroup
    \singlespacing
    \begin{algorithm}[t]
        \SetKwInOut{Input}{Input}
        \SetKwInOut{Output}{Output}
        \Input{prior $p(\theta)$, number of simulations $N$}
        \Output{approximate posterior $\wh q_N(\theta\mid S)$}
        \BlankLine
        Set $\mathcal{D} = \emptyset$ \\
        \For{$j=1:N$}{
            sample $\theta^j \sim p(\theta)$\\
            simulate $S^j \sim g_n(S\mid\theta^j)$ \\
            add $(S^j, \theta^j)$ into $\mathcal{D}$
        }
            {train} $q(\theta\mid S)$ on $\mathcal{D}$ to obtain $\wh q_N(\theta\mid S)$

        \Return $\wh q_N(\theta\mid S)$
        \caption{Neural Posterior Estimation (NPE)}
        \label{alg:npe}
    \end{algorithm}
    \endgroup

    \begingroup
    \singlespacing
    \begin{algorithm}[t]
        \SetKwInOut{Input}{Input}
        \SetKwInOut{Output}{Output}
        \Input{$S_n$, prior $p(\theta)$, number of simulations $N$}
        \Output{approximate posterior $\wh \pi_N(\theta \mid S_n)$}
        \BlankLine
        Set $\mathcal{D} = \emptyset$ \\
        \For{$j=1:N$}{
            sample $\theta^j \sim \wh{q}_{r-1}(\theta\mid S_n)$ with MCMC \\
            simulate $S^j \sim g_n(S\mid\theta^j)$ \\
            add $(S^j, \theta^j)$ into $\mathcal{D}$
        }
            {train} $\wh q(S\mid\theta)$ on $\mathcal{D}$ and set $\wh{\pi}_r(\theta\mid S_n) \propto \wh q(S_n\mid\theta)p(\theta)$\\
        \Return $\wh \pi_N(\theta\mid S_n)$
        \caption{Neural Likelihood Posterior Estimation (NLE)}
        \label{alg:snl}
    \end{algorithm}
    \endgroup

Both NPE and NLE can sequentially update their approximations over multiple rounds in an attempt to deliver more accurate approximations. Indeed, there exists a diverse set of sequential extensions to NPE and NLE with each having specific features that have been designed to deliver more accurate approximations. However, if our goal is to present results that articulate the underlying behavior of NPE and NLE methods, this diversity detracts from our intended purpose, as our results would have to cater for the various differences in the sequential algorithms. 

Therefore, to ensure that we can present the most general results, we do not directly analyze the sequential versions of NPE and NLE herein, and instead focus only on one-shot implementations. This will allow us the necessary flexibility to state general results without getting bogged down in specific implementation details. That being said, so long as the sequential versions of these approaches deliver more accurate approximations than their one-shot counterparts,\footnote{Formally stating such a condition would require extending a version of Assumptions \ref{ass:sieve} and \ref{ass:sieve3} to these sequential versions of NPE and NLE.} the results presented herein can be viewed as upper bounds on those achievable by sequential versions of these algorithms, was subsequent rounds of the algorithms will likely deliver sharper results than those obtained from the one-shot versions.

	While NPE and NLE have been shown to deliver accurate posterior approximations in certain empirical problems, the statistical behavior of these resulting approximate posteriors have not been formally studied. In Section \ref{sec:snpe} we develop general theoretical results for NPE; while Section \ref{sec:snl} provides similar results for NLE.

\section{Neural Posterior Estimation}\label{sec:snpe}

\subsection{Definitions and preliminaries}
We first introduce several definitions that we use to study to NPE. 
\begin{defn}Let $P$ and $Q$ be two probability measures with support $\Theta\times \mathcal{S}$, and which admit densities - wrt the Lebesgue measure $\lambda$ - $p$ and $q$. Define the total variation (TV), and  Hellinger distance as
\begin{flalign*}
	\dt_{\mathrm{H}}\left(P,Q\right)=&\frac{1}{\sqrt{2}}\left\{\int\left(\sqrt{p(\theta, S)}-\sqrt{q(\theta ,  S)}\right)^2 \dt\lambda\right\}^{\frac{1}{2}},\;\dt_{\mathrm{TV}}\left(P,Q\right)=\int\left|p(\theta , S)-q(\theta ,  S)\right|  \dt\lambda.
	\end{flalign*}The Kullback-Leibler divergence (KLD) is defined as
$$
\mathrm{KLD}(P\|Q)=\int_{\Theta}\int_{\mathcal{S}}p(\theta,S)\log \frac{p(\theta,S)}{q(\theta,S)}\dt\lambda.
$$
\end{defn}	
We will make often use of the following well-known relationships between these distances.  
\begin{lemma}[Pinsker]\label{lem:pinsk}For probability measures $P,Q$,
\begin{flalign*}
\dt_{\mathrm{H}}\left(P,Q\right)^2\le \dt_{\mathrm{TV}}\left(P,Q\right)\le \sqrt{2}\dt_{\mathrm{H}}(P,Q) \text{ and } \dt_{\mathrm{TV}}\left(P,Q\right) \le \sqrt{\frac{1}{2}{\KL(P||Q)}}.
\end{flalign*}
\end{lemma}

Let $\mathcal{Q}$ be a class of conditional probability measures for $\theta\mid S$. NPE attempts to find the closest $Q\in\mathcal{Q}$ to $\Pi(\cdot\mid S)$, where closeness is defined in terms of $\KL(\Pi\|Q)$. However, since $\Pi(\cdot\mid S)$ is intractable, so is $\KL(\Pi\|Q)$. NPE circumvents this intractability by replacing $\KL(\Pi\|Q)$ with a Monte Carlo approximation based on $N$ iid simulated datasets where $(\theta^i,S^i)\stackrel{iid}{\sim} p(\theta)g_n(S\mid\theta)$. 
\begin{defn}[NPE]\label{def:snpe}Let $\mathcal{Q}$ be a family of conditional distributions for $\theta\mid S$.  At termination of the algorithm, $\widehat{Q}_N$ denotes the $Q\in\mathcal{Q}$ with density $q=\dt Q/\dt\lambda$ that satisfies 
\begin{equation}\label{eq:snpe}
-\frac{1}{N}\sum_{i=1}^{N}\log \wh q_N(\theta^i\mid S^i)K(\theta^i)\le \inf_{Q\in\mathcal{Q}} -\frac{1}{N}\sum_{i=1}^{N}\log q(\theta^i\mid S^i)K(\theta^i),
\end{equation}for a known importance function $K(\theta)$. Given observed summaries $S_n$, the NPE is $\wh Q_N(\cdot\mid S_n)$.
\end{defn}
\begin{remark}
$\widehat{Q}_N$ is the solution to a variational optimization problem that seeks to minimize a Monte Carlo estimate of the (possibly weighted) forward KLD between $\Pi(\cdot\mid S)$ and $Q(\cdot\mid S)$ over $(\theta,S)$. In contrast to standard variational Bayesian inference, NPE minimizes the forward KLD rather than the reverse KLD, i.e., $\KL(Q\|\Pi)$, using simulations from the assumed model. Definition \ref{def:snpe} clarifies that as $N\rightarrow\infty$ it is the difference between $\Pi(\cdot\mid S)$ and $Q(\cdot\mid S)\in\mathcal{Q}$  in forward KLD that NPE attempts to control.  
\end{remark}

Definition \ref{def:snpe} clarifies that NPE learns a flexible approximation to the conditional distribution 
$$
\pi(\theta\mid S)={g_n(S\mid \theta)p(\theta)}/{p(S)},\quad p(S)=\int_\Theta g_n(S\mid\theta)p(\theta)\dt\theta,
$$
and evaluates this approximation at $S=S_n$ to produce $\widehat{Q}_n(\cdot\mid S_n)$. Hence, NPE  may not recover the exact posterior $\Pi(\cdot\mid S_n)$ for at least two reasons. First, if $\mathcal{Q}$ is not rich enough, it may be that $\dt_{\mathrm{TV}}(\widehat{Q}_N, \Pi)>0$ even as $N\rightarrow\infty$, which some have referred to as the expressiveness of $\mathcal{Q}$. However, so long as $\mathcal{Q}$ is a class of universal conditional density approximators and if $\Pi(\cdot\mid S)$ is well-behaved, then NPE will satisfy
$
\dt_{\mathrm{TV}}(\Pi,\wh Q_N)\rightarrow0$ as ${N\rightarrow\infty}$.

More critically, $S\mapsto\wh Q_N(\cdot\mid S)$ is a global approximation to the map $S\mapsto \Pi(\cdot\mid S)$, and so it may be a poor approximation to $\Pi(\cdot\mid S)$ at the point $S=S_n$. To see why this can be the case, note that if $S_n$ is in the tail of $g_n(S\mid\theta)$, then $S\mapsto\wh Q_N(\cdot\mid S)$ will be trained on many values of $S$ that are unlike $S_n$, and NPE will encounter the classical issue of ``extrapolation bias''. The way in which $\wh Q_N(\cdot\mid S)$ is trained represents a fundamental difference with other SBI methods like ABC: $\wh Q_N(\cdot\mid S_n)$ is trained on samples from the entire prior predictive space $\mathcal{S}$; ABC is trained only on values of $S$ in the restricted prior predictive space $\mathcal{S}_\delta=\{S\in\mathcal{S}:\dt(S_n,S)\le\delta\}$, for some known tolerance $\delta\ge0$. Hence, if $\mathcal{S}$ has little mass near $S_n$ under the assumed model density $g_n(\cdot\mid\theta)$ for any $\theta\in\Theta$, then $\wh Q_N(\cdot\mid S_n)$ will be less accurate than ABC methods for approximating $\Pi(\cdot\mid S_n)$. 

To minimize the potential for extrapolation bias, $S\mapsto\wh Q_N(\cdot\mid S)$ must be trained on samples from $g_n(S\mid\theta)p(\theta)$ where $S$ is sufficiently close to $S_n$. One way to minimize possible extrapolation bias is to modify NPE training to include a large number of data points for which $(\theta,S)\sim p(\theta) g_n(S\mid\theta) 1(S\in\mathcal{S}_\delta)$ for some $\delta>0$; the latter idea has been termed preconditioned NPE by \cite{wang2024preconditioned}. %Such a training regime can be incorporated in Definition \ref{def:snpe} by taking $K(\theta)=p(\theta)1(\theta\in\Theta: S\in\mathcal{S}_\delta)$ for some pre-specified $\delta>0$.

The above discussion clarifies that to ensure NPE is not corrupted by extrapolation bias the observed summary $S_n$ cannot be far in the tail of the assumed model density $g_n(S\mid\theta)$ for some $\theta\in\Theta$. Heuristically, this will require that the summaries can be matched by the assumed model, for some value of $\theta$, for all $n$ large enough. This condition has been called model compatibility by \cite{marin2014relevant}, \cite{frazier2020model}, and \cite{frazier2024synthetic}. When model compatibility is satisfied, and under further regularity conditions, we can derive general results on the statistical accuracy of $\wh Q_N(\cdot\mid S_n)$.

\subsection{Statistical Guarantees for NPE}\label{sec:general}	
Before presenting the theoretical behavior of $\wh Q_n(\cdot\mid S_n)$, we present several notations used throughout the remainder of the paper. The letter $C$ denotes a generic positive constant (independent
of $n$), whose value may change from one occurrence to the next, but whose precise value is unimportant.  For two sequences $\{a_n\},\{b_n\}$ of real numbers, 
$a_n \lesssim b_n$ (resp. $\gtrsim$) means $a_n \leq C b_n$ (resp. $a_n \geq Cb_n)$ for all $n$ large enough. Similarly,
$a_n \sim b_n$ means that 
$$
{1/C} \leq \liminf_{n \rightarrow \infty} |a_n/b_n|\leq \limsup_{n \rightarrow \infty} |a_n/b_n| \leq C \;.
$$  For two sequences $a_n,b_n\rightarrow0$, the notation $a_n\ll b_n$ means that $b_n/a_n\rightarrow0$. 
Recall that $P_0^{(n)}$ denotes the true distribution of the data and let $F^{(n)}_0$ denote the implied distribution of the summaries $S_n(y)$, with density $f_0^{(n)}=\dt F_0^{(n)}/\dt\lambda$. The symbol $\Rightarrow$ denotes convergence in distribution  under $P_0^{(n)}$.  For sequences $a_n$, the notation $o_p(a_n)$ and $O_p(a_n)$ have their usual connotations. We let $N(x;\mu,\Sigma)$ denote the Gaussian density at $x$ with mean $\mu$ and variance $\Sigma$. Expectations under $P^{(n)}_0$ are denoted by $\E_{0}^{(n)}$. 

The following regularity conditions on the observed and simulated summaries are used to deduce posterior concentration. These conditions are similar to those maintained in \cite{marin2014relevant} and \cite{FMRR2016} to derive the behavior of ABC posteriors. 

\begin{assumption}\label{ass:obs_sum}
	There exist a positive sequence $\nu_n\rightarrow\infty$ as $n\rightarrow\infty$, a vector $\mu_0 \in \mathbb{R}^{d_s}$, and a covariance matrix $V$ such that
	$
	\nu_n  (S_n-b_0) \Rightarrow \mathcal{N}(0,V)
	$. For some $\alpha\ge d_s+1$, $x_0>0$, and all $x$ such that $0<x< \nu_nx_0$, 	
	\begin{equation}\label{tests}
		F_0^{(n)} \left[ \nu_n \|S_n - b_0\| > x \right] \leq C x^{-\alpha}\,.
	\end{equation}
	
\end{assumption}
\begin{assumption}\label{ass:sim_sum}
	The mean of the summaries under $g_n(\cdot\mid \theta)$, $b_{n}(\theta)=\E_{S\sim g_n(S\mid\theta)}(S)$, exists. For $\alpha$ as in Assumption \ref{ass:obs_sum}, 	and some $x_0>0$, 
	\begin{equation}\label{tests}
		G_n\left\{\nu_n\|S_n(z) - b_{n}(\theta)\| > x \mid\theta\right\} \leq C(\theta)x^{-\alpha}\,\text{ for all }0<x\le C \nu_nx_0,
	\end{equation}where $C(\theta)$ satisfies $\int_\Theta C(\theta)p(\theta)\dt\theta<\infty$. 
	
\end{assumption}
\begin{remark}\label{rem:regulars}
Assumptions \ref{ass:obs_sum}-\ref{ass:sim_sum} are general regularity conditions that are satisfied for many different choices of summaries and are discussed in detail in \cite{marin2014relevant} and \cite{FMRR2016}. Essentially, Assumption \ref{ass:obs_sum} requires that the observed summaries satisfy a central limit theorem and have at least a polynomial tail; Assumption \ref{ass:sim_sum} requires that, uniformly over $\Theta$, the mean of the simulated summaries exists, and that the summaries on which the NPE are trained have a polynomial tail. The tail concentration in Assumptions \ref{ass:obs_sum}-\ref{ass:sim_sum} is what will ultimately drive the concentration of the NPE. These assumptions are common in the literature on SBI, and satisfied for many different choices of summaries. Given this, we defer to \cite{marin2014relevant} and \cite{FMRR2016} for further details. 
\end{remark}
Assumptions \ref{ass:obs_sum}-\ref{ass:sim_sum} do not directly link the assumed model, $g_n(\cdot\mid\theta)$, and the observed summaries in Assumption \ref{ass:obs_sum}. To do so, we rely on the following condition that ensures the assumed model is `compatible' with the value of the observed summary statistic at some $\theta\in\Theta$. 

\begin{assumption}\label{ass:compat}
There exist $\theta_0\in\Theta$ such that $\lim_{n \rightarrow \infty}\inf_{\theta\in\Theta}\|b_{n}(\theta)-b_0\|=0$.	For any $\epsilon > 0$ there exist $\delta >0$  and a set 
	$ \mathcal{E}_n$ such that for all $\theta \in \Theta$ with $\dt(\theta,\theta_0)\le \epsilon$, 
	\begin{equation*}
		\mathcal{E}_n \subseteq \{ S\in\mathcal{S}: g_n(S\mid\theta)  \geq \delta  f_0^{(n)}(S) \},\quad P_0^{(n)}\left( \mathcal{E}_n^c \right) < \epsilon.
	\end{equation*} 
\end{assumption}
\begin{remark}
Assumption \ref{ass:compat} essentially requires that the assumed model and the true DGP are similar, and that the observed summaries can be replicated under the assumed model. In particular, the second part of the assumption states that, for all values of $\theta$ that are close to $\theta_0$, the assumed and true model densities must also be close. Note, that Assumption \ref{ass:compat} does not require that the true DGP $P^{(n)}_0$ is equal to $P^{(n)}_{\theta_0}$, which is a sufficient condition, but stronger than necessary, only that the density of the summaries under the true model is close enough to that of the assumed model for the summaries.  A version of Assumption \ref{ass:compat} has been used in all theoretical studies of SBI methods of which we are aware. 
 \end{remark}
\begin{remark}\label{rem:compat1}
Assumption \ref{ass:compat} makes rigorous our earlier discussion that for $\wh Q_N(\cdot\mid S_n)$ to be an accurate approximation of $\Pi(\cdot\mid S_n)$ it must be trained on samples from the prior predictive that are sufficiently close to $S_n$. Without this requirement there is no hope that $\wh Q_N(\cdot\mid S_n)$ will be an accurate approximation to $\Pi(\cdot\mid S_n)$. When $g_n(\cdot\mid \theta)$ has little support near $S_n$, we suggest to apply the preconditioned NPE approach of \cite{wang2024preconditioned}.	
\end{remark}

We also require that the prior places sufficient mass on the value $\theta_0\in\Theta$ under which the assumed model can match the observed summaries, and that the prior does not put proportionally more mass on values away from $\theta_0$. Let $\epsilon_n=o(1)$ be such that $\epsilon_n\gtrsim \nu_n^{-1}$, take $M>0$, and define $\Theta_n=\{\theta\in\Theta:\dt(\theta,\theta_0)\le M\epsilon_n\}$. 
\begin{assumption}\label{ass:prior_concentration}
	There are constants $\tau$ and $d$, with $\tau\ge d+1$, such that
	\begin{flalign*}
		\int_{\Theta_{n}^c}p(\theta)\dt\theta\lesssim M^{\tau}\epsilon_n^{\tau},\quad	
		\int_{\Theta_{n}}p(\theta)\dt\theta\gtrsim M^{d}\epsilon_n^{d}.
	\end{flalign*}
\end{assumption}

Assumptions \ref{ass:obs_sum}-\ref{ass:prior_concentration} allow us to control the behavior of the (infeasible) partial posterior $\Pi(\cdot\mid S_n)$, but do not control the accuracy with which $\widehat{Q}_N(\cdot\mid S_n)$ approximates $\Pi(\cdot\mid S_n)$. To understand how we can control the accuracy of this approximation, first recall that NPE is trained using $N$ randomly simulated datasets each generated independent from $g_n(\cdot\mid\theta)p(\theta)$. Therefore, $\widehat{Q}_N(\cdot\mid S_n)$ contains two separate sources of randomness: one source from $S_n$, with law $P^{(n)}_0$, and one source from the $N$ iid datasets $(S^i,\theta^i)$ generated from $g_n(\cdot\mid\theta)p(\theta)$. To accommodate the randomness of this synthetic training data, we must integrate out its influence within $\widehat{Q}_N(\cdot\mid S_n)$. To this end, we let $\E^{(N,n)}_{0}$ denote expectations calculated under the training data $\{(\theta^i,S^i)\stackrel{iid}{\sim}g_n(\cdot\mid\theta)p(\theta):i\le N\}$ and $P^{(n)}_0$. 

While it is possible to entertain several different assumptions regarding the accuracy with which $\wh Q_N(\cdot\mid S_n)$ approximates $\Pi(\cdot\mid S_n)$, we note that the concentration in Assumption \ref{ass:obs_sum} implies that as $n$ becomes large we do not require an accurate approximation to $\Pi(\cdot\mid S_n)$ over the entirety of $\mathcal{S}$, but only in a neighborhood of $b_0=\mathbb{E}_0^{(n)}(S_n)$. This observation, and the representation of NPE as a variational optimizer based on forward KLD  motivates the following assumption.\footnote{A version of Assumption \ref{ass:sieve} is also possible to entertain if we replace $b_0$ with $S_n$. However, since $\|S_n-b_0\|=o_p(1)$, we believe the stated version of Assumption \ref{ass:sieve} to be more interpretable.} 

\begin{assumption}\label{ass:sieve}For some $\gamma_{N}=o(1)$, 
	$
	\E_0^{(N,n)}\KL\{\Pi(\cdot\mid b_0)||\wh Q_N(\cdot\mid b_0)\}\le \gamma_N^2 .
	$
\end{assumption}

\begin{remark}\label{rem:compat2}
	Assumption \ref{ass:sieve} places a condition on the rate at which $\wh Q_N(\cdot\mid S)$ learns about $\Pi(\cdot\mid S)$, at the point $S=b_0$. Note that if Assumption \ref{ass:compat} is not satisfied, there is no hope that Assumption \ref{ass:sieve} can be satisfied: if Assumption \ref{ass:compat} is violated, $b_0\not\in\mathcal{S}$, then $b_0$ is not in the support of $g_n(S\mid\theta)$ for any $\theta\in\Theta$, and NPE cannot learn about $b_0\mapsto \Pi(\cdot\mid b_0)$. This clarifies that without Assumption \ref{ass:compat}  there is no reason to suspect $\wh Q_N(\cdot\mid S_n)$ will be a useful approximation to $\Pi(\cdot\mid S_n)$.  
\end{remark} 

Up to a constant, and for $n$ large, $\KL\{\Pi(\cdot\mid b_0)||Q(\cdot\mid b_0)\}$ is the term we minimize in NPE when $N\rightarrow\infty$.\footnote{This constant term depends purely on the assumed model, and is independent of the class $\mathcal{Q}$.} That is, $\KL\{\Pi(\cdot\mid b_0)||Q(\cdot\mid b_0)\}$ is precisely the criterion that NPE would minimize if $b_0$ were known. Hence, Assumption \ref{ass:sieve} is both intuitive and natural to maintain.
\begin{remark}\label{rem:other_ass}
	{While Assumption \ref{ass:sieve} may be natural to maintain given the definition of NPE, it can be cumbersome to check in practice. Equivalent theoretical results to those presented later can be obtained if Assumption \ref{ass:sieve} is replaced with a similar condition in Hellinger distance; e.g., $\E_0^{(N,n)}\dt_{\mathrm{H}}\{\Pi(\cdot\mid b_0)||\wh Q_N(\cdot\mid b_0)\}\le \gamma_N$. See Appendix \ref{app:results} and Corollary \ref{corr:main} for further details.} 	
\end{remark}

The last assumption we maintain guarantees that the mappings $S\mapsto\Pi(\cdot\mid S)$ and $S\mapsto Q(\cdot\mid S)$ have sufficient regularity in terms of their conditioning argument $S$, and requires smoothness of the conditional density, wrt the conditioning argument, in the total variation distance. This condition has been referred to as total-variation smoothness by \cite{neykov2021minimax} and \cite{li2022minimax}, and was shown by the latter authors to be valid for smooth densities of the form $p(\theta,s)/\int p(\theta,s)\dt\theta$.

\begin{assumption}\label{ass:lipz} For some $\delta>0$, all $\theta\in\Theta$ and $s, s^{\prime} \in\mathcal{S}_\delta(b_0)$: for $p(\theta\mid s)$ denoting $\pi(\theta\mid s)$ or $q(\theta\mid s)$,  
	$
	\int |p(\theta\mid s)-p(\theta\mid s')|\dt \theta \leq C\left\|s-s^{\prime}\right\|,
	$
	for some sufficiently large $C>0$ such that $C\perp(\theta,s,s')$ and $\E_{0}^{(N,n)}(C)<\infty$. 
\end{assumption}
Under Assumptions \ref{ass:obs_sum}-\ref{ass:lipz} we prove that,  for a generic approximating class $\mathcal{Q}$, the NPE $\wh Q_N(\cdot\mid S_n)$ has theoretical behavior that makes it useful as a posterior approximation. 

\begin{theorem}\label{thm:main}Let $\epsilon_n=o(1)$ be a positive sequence such that $\nu_n\epsilon_n\rightarrow\infty$.  Under Assumptions \ref{ass:obs_sum}-\ref{ass:lipz}, for any positive sequence $M_n\rightarrow\infty$,  and any loss function $:L:\mathcal{S}\times\mathcal{S}\rightarrow\mathbb{R}_+$, 
	$$
	\E_0^{(N,n)}\wh Q_N\left[ L\{b_{n}(\theta),b_0\}>M_n (\epsilon_n+\gamma_{N})\mid S_n\right]=o(1)\quad (\text{as }n,N\rightarrow\infty).
	$$
\end{theorem}
\begin{remark}
	Theorem \ref{thm:main} implies that for $\wh Q_N(\cdot\mid S_n)$ to concentrate at rate $\epsilon_n$ we must asymptotically learn the map $S\mapsto \Pi(\cdot\mid S)$ fast enough so that $\gamma_N$ in Assumption \ref{ass:sieve} satisfies $\gamma_N\lesssim \epsilon_n$. For instance, if $\gamma_N\lesssim 1/\nu_n$, then Theorem \ref{thm:main} implies that NPE will concentrate at the usual rate. That is, if the practitioner chooses $N$ large enough so that $\nu_n\gamma_N=o(1)$ as $n,N\rightarrow\infty$, then $\wh Q_N(\cdot\mid S_n)$ attains the standard posterior concentration rate. To obtain precise conclusions on how large $N$ must be to achieve this rate, we require more structure on $\mathcal{Q}$, which we analyze in Section \ref{sec:specific}. 
\end{remark}

If $\Pi(\cdot\mid S_n)$ correctly quantifies uncertainty, then the regularity conditions in Assumptions \ref{ass:obs_sum}-\ref{ass:lipz} are sufficient to prove that NPE correctly quantifies uncertainty. Let $\theta_n=\argmin_{\theta\in\Theta}-\log g_n(S_n\mid \theta)$, $t=\nu_n(\theta-\theta_n)$, $\mathcal{B}=\lim_n\partial b_n(\theta_0)/\partial\theta^\top$, $\Sigma=\mathcal{B}^\top V^{-1}\mathcal{B}$, and let $\pi(t\mid S_n)$ denote the posterior for $t$. 
 
\begin{theorem}\label{thm:bvm}Under Assumptions \ref{ass:obs_sum}-\ref{ass:lipz}, if $\mathcal{B}$ is non-singular, $\nu_n\gamma_N=o(1)$, and $\int |\pi(t\mid S_n)-N\{t;0,\Sigma^{-1}\}|\dt\theta=o_p(1)$, then
	$\int |\wh q_N(t\mid S_n)-N\{t;0,\Sigma^{-1}\}|\dt t=o_p(1).
$
\end{theorem}

Theorem \ref{thm:bvm} demonstrates that if the intractable posterior $\Pi(\cdot\mid S_n)$ is asymptotically Gaussian and calibrated, then so long as $\nu_n^{}\gamma_{N}=o(1)$, the NPE is also asymptotically Gaussian and calibrated. Theorem \ref{thm:bvm} implies that NPE behaves similarly to other SBI methods like ABC and BSL under model compatibility, and gives a rigorous basis for their continued use as inference engines.

\subsection{The Impact of Compatibility on NPE}

As discussed in Remarks \ref{rem:compat1}-\ref{rem:compat2}, the statistical behavior of NPE depends on the validity of the compatibility condition in Assumption \ref{ass:compat}: if NPE does not see training data like $S_n$, which occurs when Assumption \ref{ass:compat} is violated, then Assumption \ref{ass:sieve} will not be satisfied in practice and NPE will not be an accurate approximation to $\Pi(\cdot\mid S_n)$ in general. In this section, we illustrate that even when $N$ is very large, if $S_n$ differs from the values under which $\wh{Q}_N$ was trained, then NPE will not accurately approximate $\Pi(\cdot\mid S_n)$. 

We illustrate this phenomenon using a simple moving average model of order two: the observed data $y_t$ is generated according to the model
$$
y_t=\theta_1\epsilon_{t-1}+\theta_2\epsilon_{t-2}+\epsilon_t,\quad \epsilon_t\stackrel{iid}{\sim}N(0,1),\quad(t=1,\dots,T),
$$where the unknown parameters $\theta_1$ and $\theta_2$ control the autocorrelation in the observed data, and the prior on $\theta=(\theta_1,\theta_2)^\top$ is uniform over
$$
-1<\theta_1<1,\quad \theta_1+\theta_2>-1,\quad\theta_1-\theta_2<1.
$$The uniform prior over the above triangular region is enforced to ensure the model generates a process that is invertible, so that the true unknown value of $\theta$ can be recovered from the data. Consider as our observed summaries for inference on $\theta$ the sample variance, the first and second order auto-covariances: $S_n(y)=(\delta_0(y),\delta_1(y),\delta_2(y))^\top$ where $\delta_0(y)=T^{-1}\sum_{t=1}^{T}y_t^2$, $\delta_1(y)=T^{-1}\sum_{t=2}^{T}y_ty_{t-1}$ and $\delta_2(y)=T^{-1}\sum_{t=3}^{T}y_ty_{t-2}$. Under the assumed model, the mean of these summaries is given by 
$$b_n(\theta)=b(\theta)=(b_0(\theta),b_1(\theta),b_2(\theta))^\top=((1+\theta_1^2+\theta_2^2),\theta_1(1+\theta_2),\theta_2)^\top.$$ 

The term $b_0(\theta)$ is lower bounded by unity. Therefore, in large samples, if we encounter a value of $\delta_0(y)$ that is below unity, we will enter a situation of incompatibility (i.e., Assumption \ref{ass:compat} is invalid), and Assumptions \ref{ass:sieve} is likely violated. To illustrate the impact of incompatibility on NPE, and its relevance for Assumption \ref{ass:sieve}, we fix the value of $\delta_0(y)$ and compare $\KL\{\Pi(\cdot\mid S_n)\|\wh Q_N(\cdot\mid S_n)\}$ for various choices of $N$. We set $\delta_0(y) = 0.01$ to mimic a case of extreme incompatibility and $\delta_0(y) = 0.99$ to represent minor incompatibility.  %Taking values of $\delta_0(y)\in\{0.25,0.50,0.75,1\}$ then considers a case where the model is incompatible, tending towards compatibility. 

In Figure \ref{fig:ma2_kld_compare}, we plot the KLD between the exact posterior $\Pi(\cdot\mid S_n)$ and the NPE $\wh{Q}_n(\cdot\mid S_n)$ for both of our choices of $\delta_0(y)$, and for our four difference choices of $N\in\{n,n\log(n),n^{3/2},n^2\}$.
Details on the computation of the KLD and the exact partial posterior sampling are provided in Appendix~\ref{sec:ma2_appendix}.
The results clearly show that under major incompatibility, larger choices of $N$ do not deliver a more accurate approximation of the posterior, which violates Assumption \ref{ass:sieve}. However, in the case of minor incompatibility we see that the value of the KLD between $\Pi(\cdot\mid S_n)$  and $\wh Q_N(\cdot\mid S_n)$ decreases as $N$ increases, suggesting that Assumption \ref{ass:sieve} is satisfied, or is close to being satisfied. These results empirically demonstrate that for Assumption \ref{ass:sieve} to be satisfied NPE must be trained on data that is close to the observed summaries $S_n$.

% \begin{figure}
%     \centering
%     \includegraphics[width=0.5\linewidth]{untitled folder/kl_vs_b0.pdf}
%     \caption{Comparison of $\KL\{\Pi(\cdot\mid S_n)\|\widehat{Q}(\cdot\mid S_n)\}$  across different choices of $N$. }
%     \label{fig:mat2_kld_compare}
% \end{figure}

\begin{figure}[htbp]
    \centering
    % Incompatible Case
    \begin{subfigure}[b]{0.48\textwidth}
        \centering
        \includegraphics[width=\textwidth]{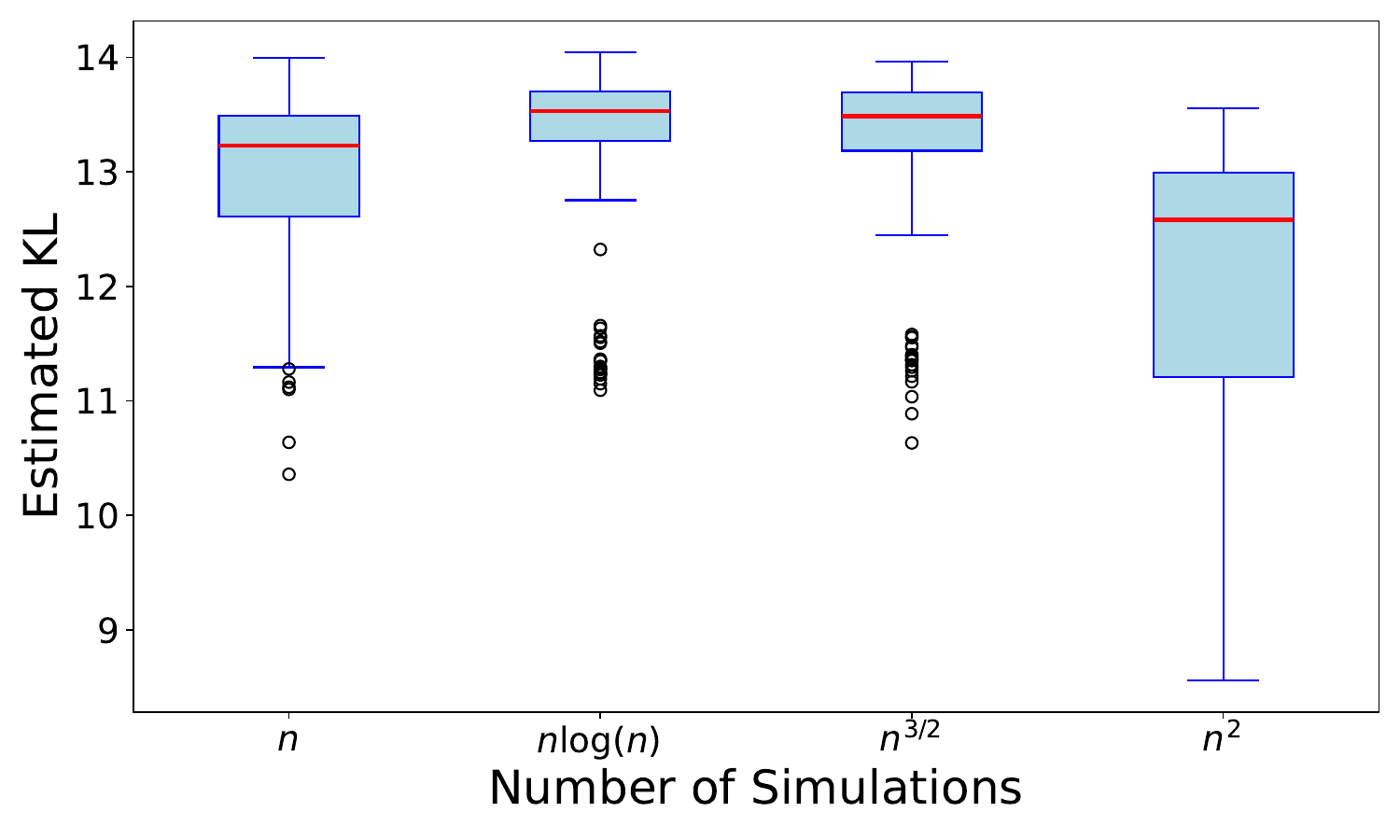}
        \caption{$\delta_0(y) = 0.01$}
        \label{fig:ma2_kld_incompatible}
    \end{subfigure}
    \hfill
    % Compatible Case
    \begin{subfigure}[b]{0.48\textwidth}
        \centering
        \includegraphics[width=\textwidth]{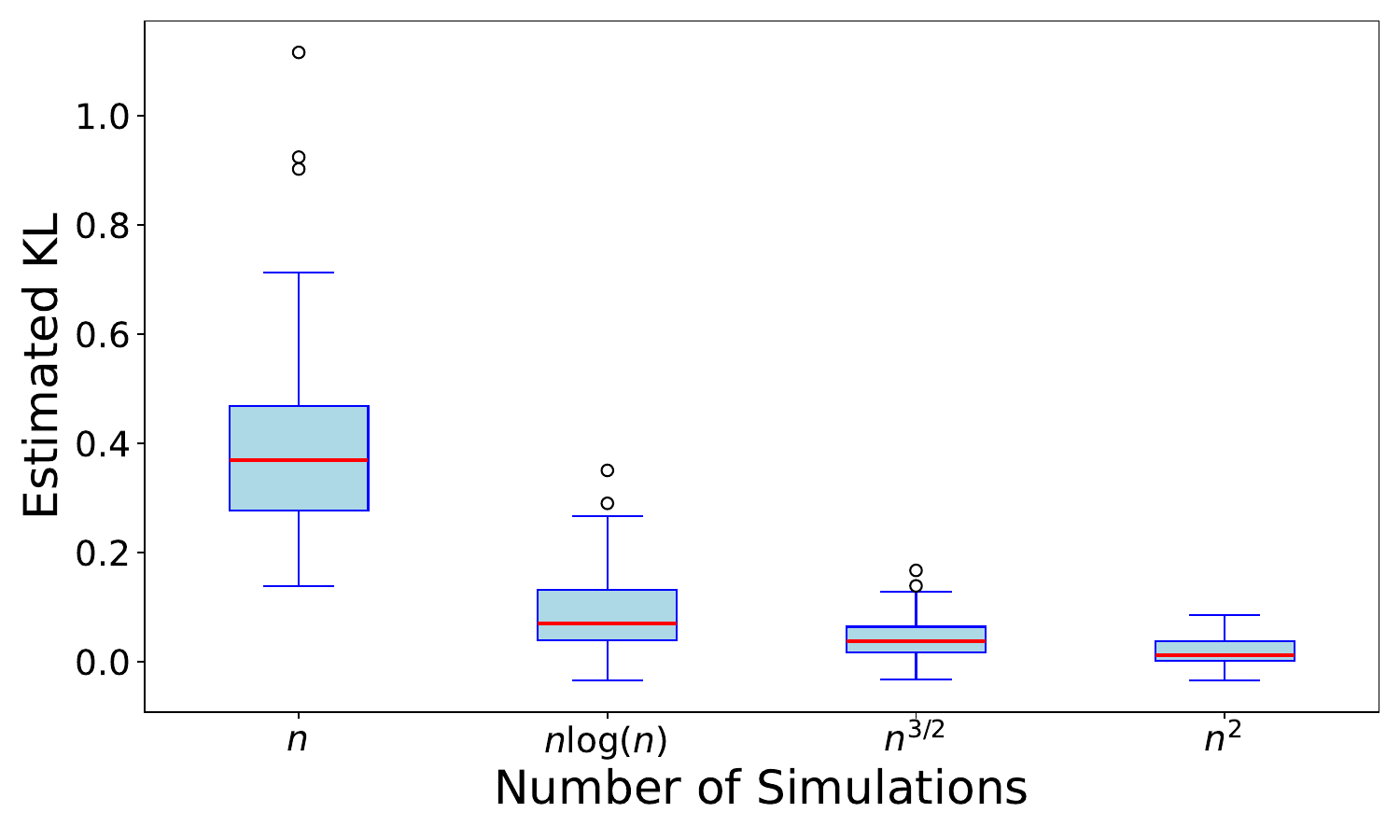}
        \caption{$\delta_0(y) = 0.99$}
        \label{fig:ma2_kld_compatible}
    \end{subfigure}
    \caption{Comparison of $\KL\{\Pi(\cdot\mid S_n)\|\widehat{Q}_n(\cdot\mid S_n)\}$ across different choices of $N$ for extreme (a) and minor (b) cases of incompatibility. Please note that the scales in Panels (a) and (b) differ markedly.}
    \label{fig:ma2_kld_compare}
\end{figure}

\section{The Choice of $\mathcal{Q}$ and its Implications}\label{sec:specific}

Theorem \ref{thm:main} implies that for NPE to deliver reliable inference on $\theta$, it must be that: 1) the assumed model is compatible with $S_n$ (Assumption \ref{ass:compat}); 2) elements in $\mathcal{Q}$ must be a good approximation to $\Pi(\cdot\mid S_n)$ (Assumption \ref{ass:sieve}). 
In our framework the accuracy of $\widehat{Q}_N(\cdot\mid S_n)$, as an approximation to $\Pi(\cdot\mid S_n)$, is encapsulated by $\gamma_N$,  and  depends on the smoothness of $ \Pi(\cdot\mid S)$ and the complexity of $\Theta\times\mathcal{S}$. To obtain specific requirements on $N$ to ensure that $\wh Q_N(\cdot\mid S_n)$ attains the usual posterior concentration rate, we must choose $\mathcal{Q}$ and make assumptions about the smoothness of $\Pi(\cdot\mid S_n)$.

While neural methods and normalizing flows are common choices for the approximating class $\mathcal{Q}$, the diversity of such methods, along with their complicated tuning and training regimes, makes establishing theoretical results on the rate of convergence, $\gamma_N$ in Assumption \ref{ass:sieve}, for such methods difficult. Without such a result, obtaining a specific requirement on $N$ is infeasible. We circumvent this issue and obtain concrete requirements on the value of $N$ for NPE by making a high-level assumption that the chosen conditional density estimation method used to approximate $\Pi(\cdot\mid S_n)$ attain the minimax rate of convergence for conditional density estimation in the class of locally $\beta$-H\"{o}lder functions: for any $\beta>0,\tau_0\ge0$ and a non-negative function $L:\mathbb{R}^d\rightarrow\mathbb{R}_+$, let $\mathcal{C}^{\beta,L,\tau_0}(\mathbb{R}^d)$ be the set 
$$
\mathcal{C}^{\beta,L,\tau_0}(\mathbb{R}^d)=\{f:  |(\partial^kf)(x+y)-(\partial^kf)(x)|\le L(x)\exp\{\tau_0\|y\|^2\}\|y\|^{\beta-\lfloor \beta\rfloor}\;(y,x\in\mathbb{R}^d)\},
$$where $\partial^k f$ denotes mixed partial derivatives up to order $k$, and $\lfloor x\rfloor$ the largest integer strictly smaller than $x$. When $\Pi(\theta\mid S)\in\mathcal{C}^{\beta,L,\tau_0}(\mathbb{R}^{d_\theta+d_s})$, there are different possible choices of $\mathcal{Q}$ that one could obtain to deduce a specific rate for $\gamma_N$. In what follows, we assume that $\mathcal{Q}$ is chosen so that $\wh Q_N(\cdot\mid S_n)$ in Definition \ref{def:snpe} achieves the minimax rate for density estimation over $\mathcal{C}^{\beta,L,\tau_0}(\mathbb{R}^{d_\theta+d_s})$.    
\begin{assumption}\label{ass:sieve2}
	For all $n$ large enough, $\pi(\theta\mid S_n)\in \mathcal{C}^{\beta,L,0}(\mathbb{R}^{d_\theta+d_s})$. For $\gamma_N$ as in Assumption \ref{ass:sieve}, the class $\mathcal{Q}$ is such that, for some $\beta>0$, $\gamma_N\asymp N^{-\beta/(2\beta+d_\theta+d_s)}\log (n)^\kappa$ for some $\kappa>0$. 
\end{assumption}
%\begin{remark}
The rate in Assumption \ref{ass:sieve2}, without the $\log(n)^\kappa$ term, is minimax for estimation when $\Pi(\cdot\mid S_n)\in\mathcal{C}^{\beta,L,0}(\mathbb{R}^{d_\theta+d_s})$ (\citealp{barron1999risk}); see, also Theorem 1 of \cite{efromovich2010dimension} and Theorem 1 of \cite{shen2013adaptive} for further examples. That is, Assumption \ref{ass:sieve2} is a high-level assumption on the class $\mathcal{Q}$ and the estimator $\wh Q_N\in\mathcal{Q}$ that assumes they deliver the best possible estimator when all that is known is that $\Pi(\cdot\mid S_n)\in\mathcal{C}^{\beta,L,0}(\mathbb{R}^{d_\theta+d_s})$. While we are unaware of any theoretical results which state that the classes of neural-nets and normalizing flows commonly used for NPE satisfy Assumption \ref{ass:sieve2}, in practice the accuracy of these methods suggests that such an assumption is reasonable to maintain.

The following result, which is a direct consequence of Theorem \ref{thm:main}, Assumption \ref{ass:sieve2} and some algebra, provides an explicit requirement on the number of simulations, $N$, required for NPE to attain the usual posterior concentration rate. 
\begin{corollary}\label{cor:one}
	Under Assumptions \ref{ass:obs_sum}-\ref{ass:sieve2}, $\wh Q_N(\cdot\mid S_n)$ concentrates at rate $\epsilon_n$ so long as  $N\gtrsim \nu_n^\alpha\log(n)^\kappa$, for $\alpha> (2\beta+d_\theta+d_s)/\beta$. 	
\end{corollary}

\subsection{The impact of dimension and computational efficiency}
Corollary \ref{cor:one} makes clear, for the first time, the impact of dimension on the behavior of NPE. In the absence of any modeling guidance on the choice of $\mathcal{Q}$ for NPE to concentrate at the standard rate we must train $\wh Q_N(\cdot\mid S_n)$ using at least $N\gtrsim \nu_n^\alpha\log(n)^\kappa$ model simulations.

To see how dimension impacts this requirement, consider that we have $d_\theta=5$ parameters and $d_s=10$ summaries, and that the summaries converge at the $\sqrt{n}$ rate (i.e., $\nu_n=\sqrt{n}$). When $\Pi(\cdot\mid S_n)\in\mathcal{C}^{\beta,L,0}(\mathbb{R}^{d_\theta+d_s})$, for $\widehat Q_N(\cdot\mid S_n)$ to concentrate at rate $\epsilon_n=\log(n)/\sqrt{n}$ our theoretical results suggest that we require at least $N\gtrsim \log(n) n^{\frac{2\beta+15}{2\beta}}$ simulated datasets to train the NPE. For example, if $n=1000$, and $\beta=10$,  this requires simulating at least $N\gtrsim 178,000$ simulated datasets.

The above requirement on $N$ is much larger than the number of simulations routinely employing in NPE, e.g., \cite{lueckmann2021benchmarking} report good results using $N\approx 10K$ datasets in NPE, albeit in experiments with fewer than $1000$ observations. That being said, the requirement that NPE must simulate $N\asymp 10^4$ datasets is orders of magnitude smaller than the requirement in ABC where it is not uncommon to simulate $10^6$ or $10^7$ datasets to obtain reliable posterior approximations. 

If one is willing to impose more onerous smoothness conditions, lower requirements on $N$ can be achieved. If we assume $\beta\ge d_s+d_\theta$, then NPE must only simulate $N\gtrsim \log(n)n^{3/2}$ datasets to achieve the parametric concentration rate. Considering this bound, it is then clear that our choices of $N$ used throughout the numerical experiments, $N\in\{n,\log(n)n,n^{3/2},n^2\}$, directly reflect the assumed smoothness of the approximation class, with larger choices of $N$ correlating to a lower degree of assumed smoothness.

Corollary \ref{cor:one} also allows us to compare the computational requirements of NPE to those of ABC. For the ABC posterior to concentrate at the $\sqrt{n}$-rate, as $n$ diverges, we require at least $N_{\text{ABC}}\gtrsim \log(n) n^{d_\theta/2}$ simulated draws (Corollary 1, \citealp{FMRR2016}). If the posterior we are approximating is smooth, so that we may take $\beta\ge(d_\theta+d_s)$, then for $n$ large NPE will be more efficient than ABC so long as $d_\theta\ge3$.\footnote{Plugging $\beta=(d_\theta+d_s)$ into the requirement in Corollary \ref{cor:one} yields the bound $N\gtrsim \log(n)n^{3/2}$.} Critically, this latter result is irrespective of the dimension of the summaries $d_s$, and we note that in practice obtaining samples from the ABC posterior when $d_s$ is large is computationally cumbersome.

Lastly, we would like to clarify that Corollary \ref{cor:one} can deliver a larger value than necessary for $N$ in certain cases. In particular, as we discuss in Section \ref{app:accQ} in the supplementary material, if we are willing to entertain that $\Pi(\cdot\mid S_n)\in\mathcal{Q}$, then to obtain the parametric rate of posterior concentration we only need to train NPE on $N\gtrsim \log(n)n$ datasets. We refer the interested reader to Section \ref{app:accQ} in the supplementary material for further details.

\subsection{Example: Illustrating the impact of dimensional on NPE}
In the following example, we use a relatively simple model to illustrate the impact of dimension (for parameters and summaries) on the accuracy of NPE (as predicted by Corollary \ref{cor:one}). The g-and-k distribution is a flexible univariate distribution that is popular as an illustrative example in SBI.
% The g-and-k model is governed by four parameters A (location), B (scale), g (skewness) and k (kurtosis).
The g-and-k distribution is defined via the quantile function,
\begin{equation}
     Q_{\text{gk}}(z(p); A, B, g, k) = A + B \left( 1 + c \tanh \left (\frac{g z(p)}{2} \right) \right) \times (1 + z(p)^2)^k z(p),
\end{equation}
where $z(p)$ is the $p$th quantile of the standard normal distribution.
Typically, we fix $c=0.8$ and conduct inference over the model parameters $\theta = (A, B, g, k)$.
It is trivial to simulate pseudo-data for a given $\theta$ by simulating $z(p) \sim \mathcal{N}(0, 1)$ and computing $Q_{\text{gk}}$.

Herein we take $n=1000$ and $n=5000$ and set the true parameter values as $A=3, B=1, g=2$ and $k=0.5$.
We set the summary statistics to be the octiles of the sample data as in \citet{drovandi+p11}.\footnote{Further details and results for the g-and-k model are presented in Appendix~\ref{sec:gnk_appendix}.}

We focus attention on the $g$ parameter, which corresponds to the skewness of the distribution, and is known to be the most challenging parameter to infer.
In Figure~\ref{fig:gnk_boxplot_all}, we illustrate the bias between the posterior mean and the true parameter value across 100 replicated datasets.
Similar to the motivating example, we again see more accurate results as we increase $n$, or for fixed $n$ as $N$ is increased. However, unlike the motivating example, we see that improvements are still obtainable, in terms of decreased bias and variance of the posterior mean, when choosing $N > n^{3/2}$.
\begin{figure}[!htb]
    \centering
    \begin{subfigure}[b]{0.45\textwidth}
        \centering
        \includegraphics[width=\textwidth]{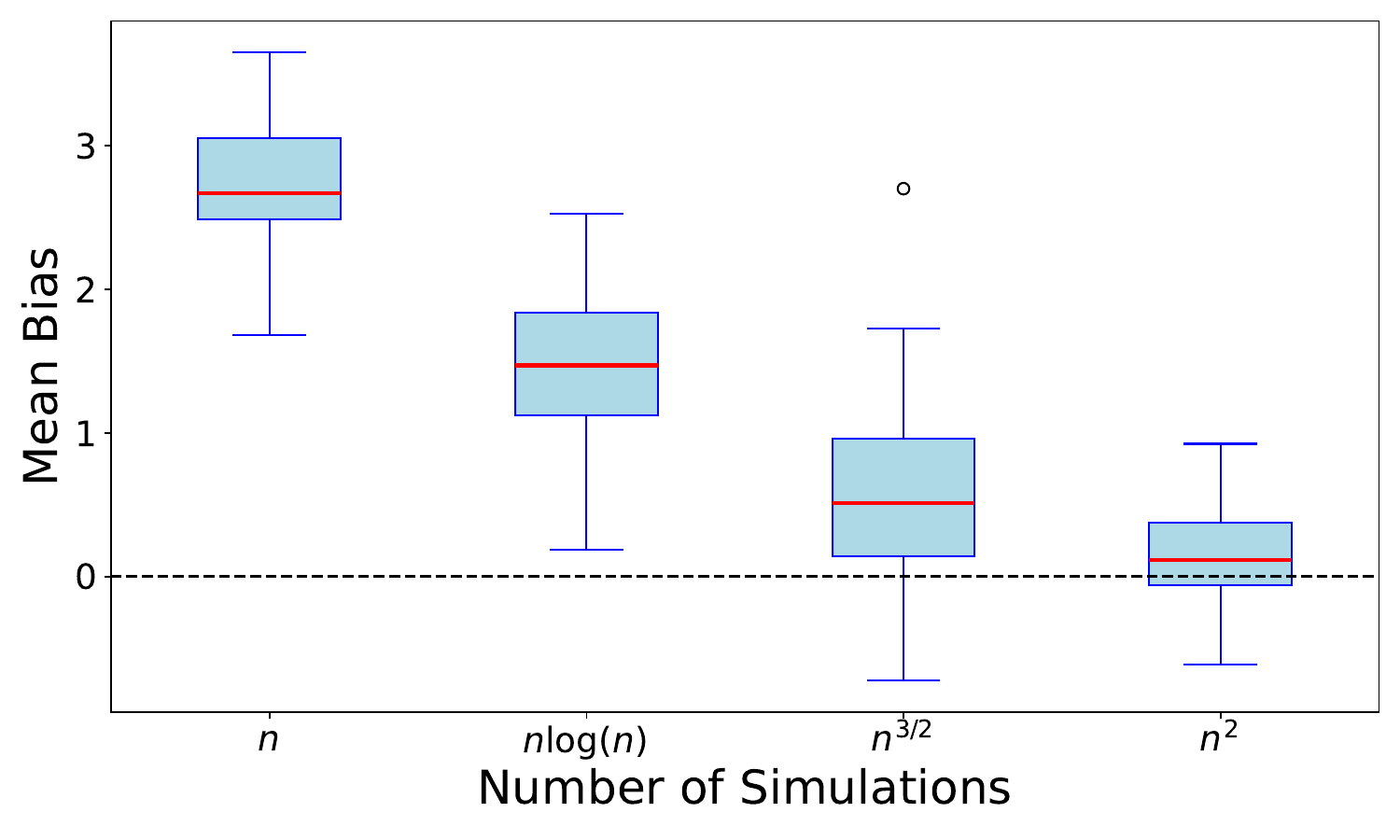}
        \caption{Boxplot of Mean Biases for $n=1000$.}
        \label{fig:boxplot_n_obs_100}
    \end{subfigure}
    \hfill
    \begin{subfigure}[b]{0.45\textwidth}
        \centering
        \includegraphics[width=\textwidth]{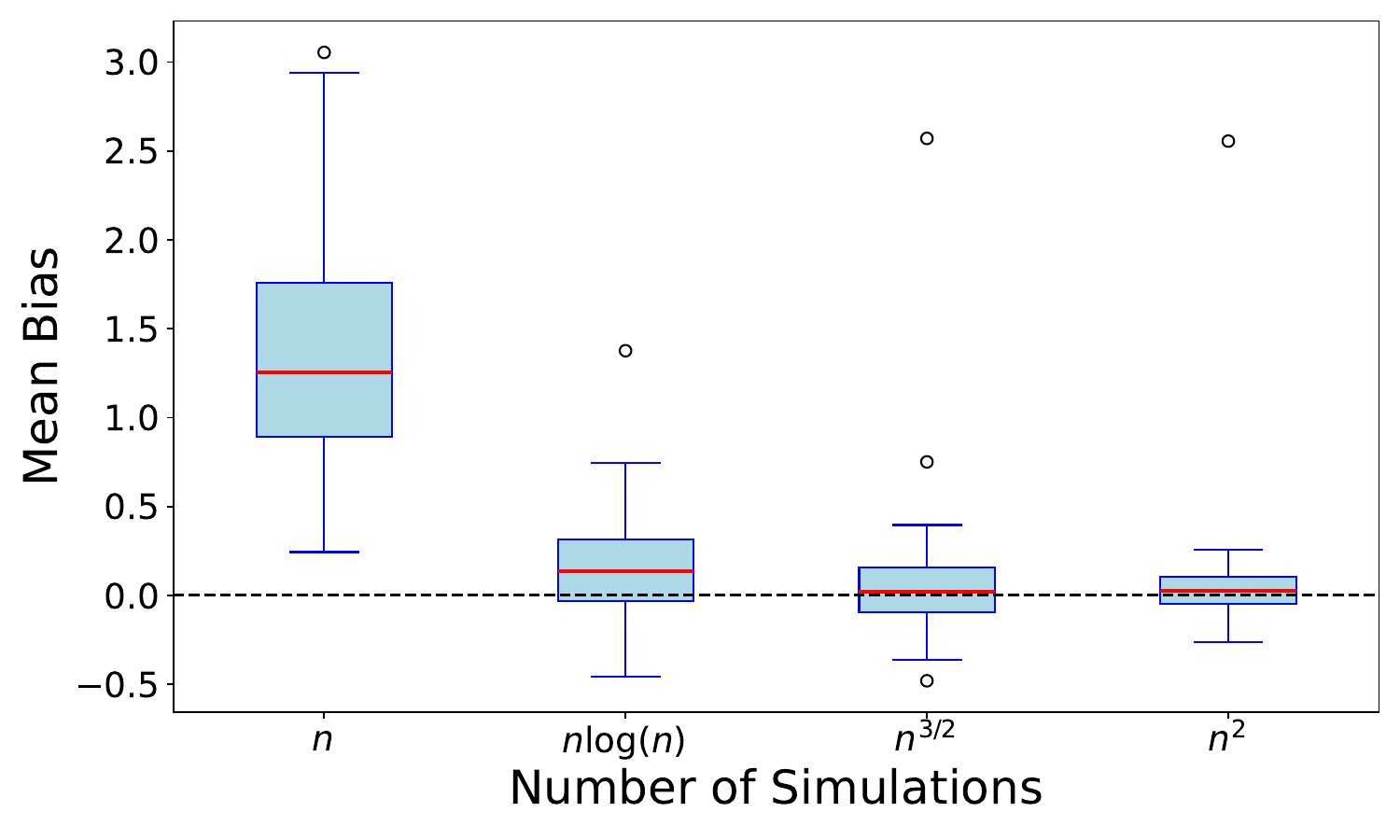}
        \caption{Boxplot of Mean Biases for $n=5000$.}
        \label{fig:boxplot_n_obs_500}
    \end{subfigure}
    
    \caption{Bias of the posterior mean for $g$ visualized through boxplots for $n=1000$ and $n=5000$ across varying $N$.}
    \label{fig:gnk_boxplot_all}
\end{figure}

A similar story is given in Table~\ref{tab:gnk_g_coverage}, which presents the marginal Monte Carlo coverage for the 80\%, 90\% and 95\% NPE credible intervals based on 100 replicated datasets. Unlike the  stereological example, which for $N\ge n^{3/2}$ produced reasonably accurate coverage results, the results under the g-and-k are more mixed: reliable coverage seems only to be delivered when we choose $N=n^2$. 

The reason why a higher value of $N$ is required to deliver accurate posteriors in this example has already been elucidated in Corollary \ref{cor:one}: all else equal, the larger the dimension of the summaries, $d_s$, and/or the parameters, $d_\theta$, the higher $N$ must be to obtain an accurate approximation. The motivating example in Section \ref{sec:motiv} featured three unknown parameters and four summaries, so that $d_s+d_\theta=7$, while in the g-and-k example, $d_s+d_\theta=11$. All else equal, this increase in dimension necessitates a higher number of simulations in order to accurately approximate the posterior, and so appreciable gains in accuracy are still achievable when we use $N=n^{3/2}$ simulations with NPE. 

\begin{table}[!htb]
\centering
\begin{tabular}{@{}l l l l l @{} }\toprule
 & $N=n$ &  $N=n\log(n)$ & $N=n^{3/2}$ & $N=n^2$ \\ 
\hline
$n=100$ &0.70/0.96/1.00 &0.35/1.00/1.00 &0.26/0.96/1.00 &0.58/0.98/1.00 \\ 
$n=500$ &0.58/0.96/1.00 &0.93/0.99/1.00 &0.98/0.99/1.00 &0.92/0.97/0.99 \\ 
$n=1000$ &0.89/1.00/1.00 &0.99/1.00/1.00 &0.96/0.99/1.00 &0.91/0.97/0.99 \\ 
$n=5000$ &1.00/1.00/1.00 &0.98/1.00/1.00 &0.91/0.98/0.99 &0.89/0.96/0.98 \\ 
\bottomrule
\end{tabular}
\caption{Monte Carlo coverage of 80\%, 90\%, and 95\% credible intervals for $g$ across different choices of $N$ and sample sizes $n$, based on 100 repeated runs.}

\label{tab:gnk_g_coverage}
\end{table}

Before concluding this section, we further illustrate the impact of dimension on the accuracy of NPE by comparing $\KL\{\Pi(\cdot\mid S_n)\|\widehat{Q}(\cdot\mid S_n)\}$ for two difference choices of summaries and across the four different choices of $N$ in the g-and-k model.
Specifics on the computation of the KLD and the exact partial posterior sampling are detailed in Appendix~\ref{sec:gnk_appendix}.
The first set of summaries is the octiles used previously, for which $d_s=7$, while we also consider inferences based on the hexadeciles, for which $d_s=15$ summaries.  Figure \ref{fig:gnk_kld_compare} shows that for our choices of $N$ and summaries, the KLD, relative to the respective partial posterior, based on the smaller number of summaries is always smaller than that based on the larger number of summaries. This clearly illustrates the behavior described in Corollary \ref{cor:one}: to obtain an accurate NPE approximation, all else equal, as the dimension of the summaries increases so must the number of simulations used to train NPE. 

\begin{figure}[htbp]
    \centering
    % Plot for n_obs = 100
    \begin{subfigure}[b]{0.48\textwidth}
        \centering
        \includegraphics[width=\textwidth]{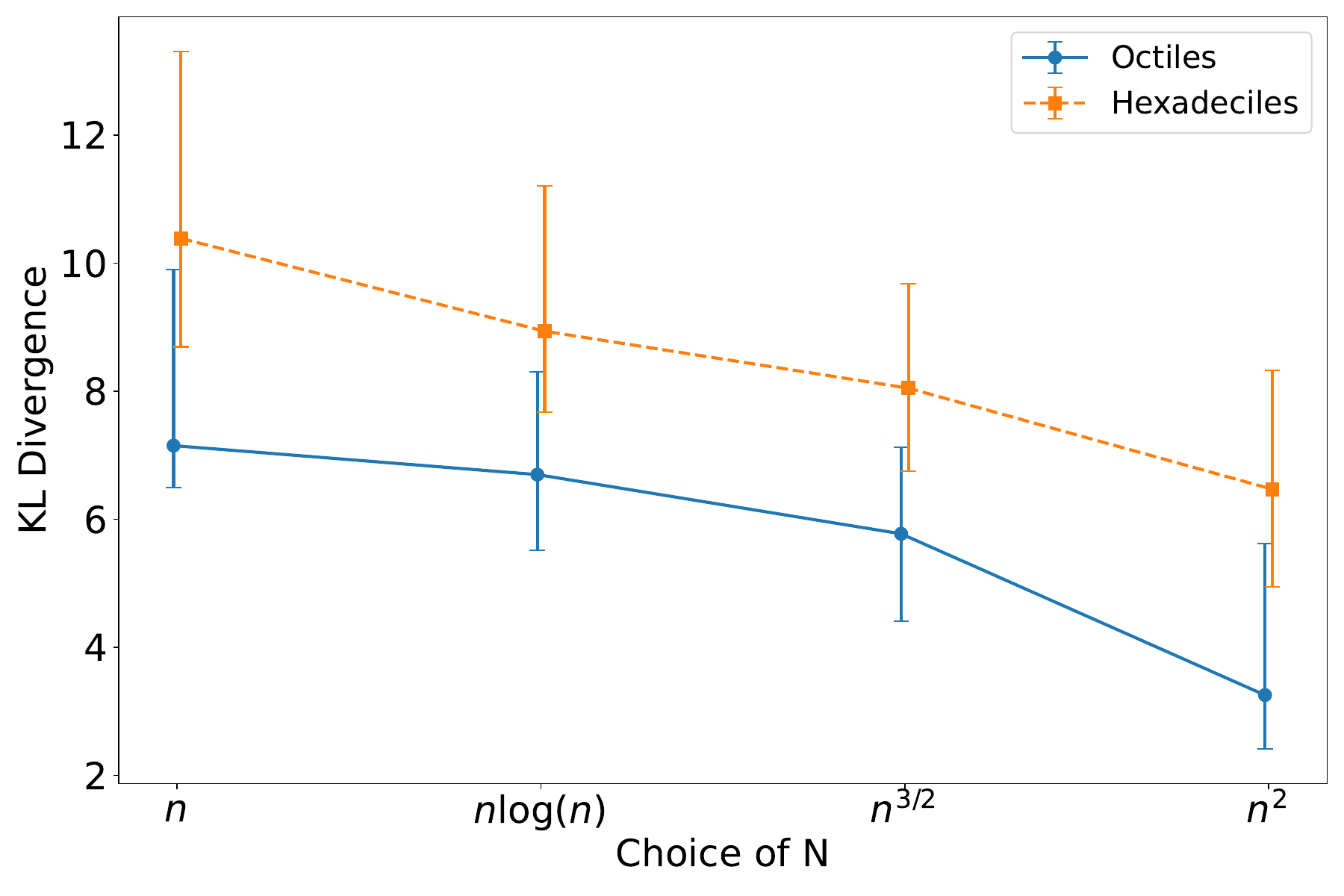}
        \caption{Number of observations $n = 100$}
        \label{fig:gnk_summaries_100}
    \end{subfigure}
    \hfill
    % Plot for n_obs = 1000
    \begin{subfigure}[b]{0.48\textwidth}
        \centering
        \includegraphics[width=\textwidth]{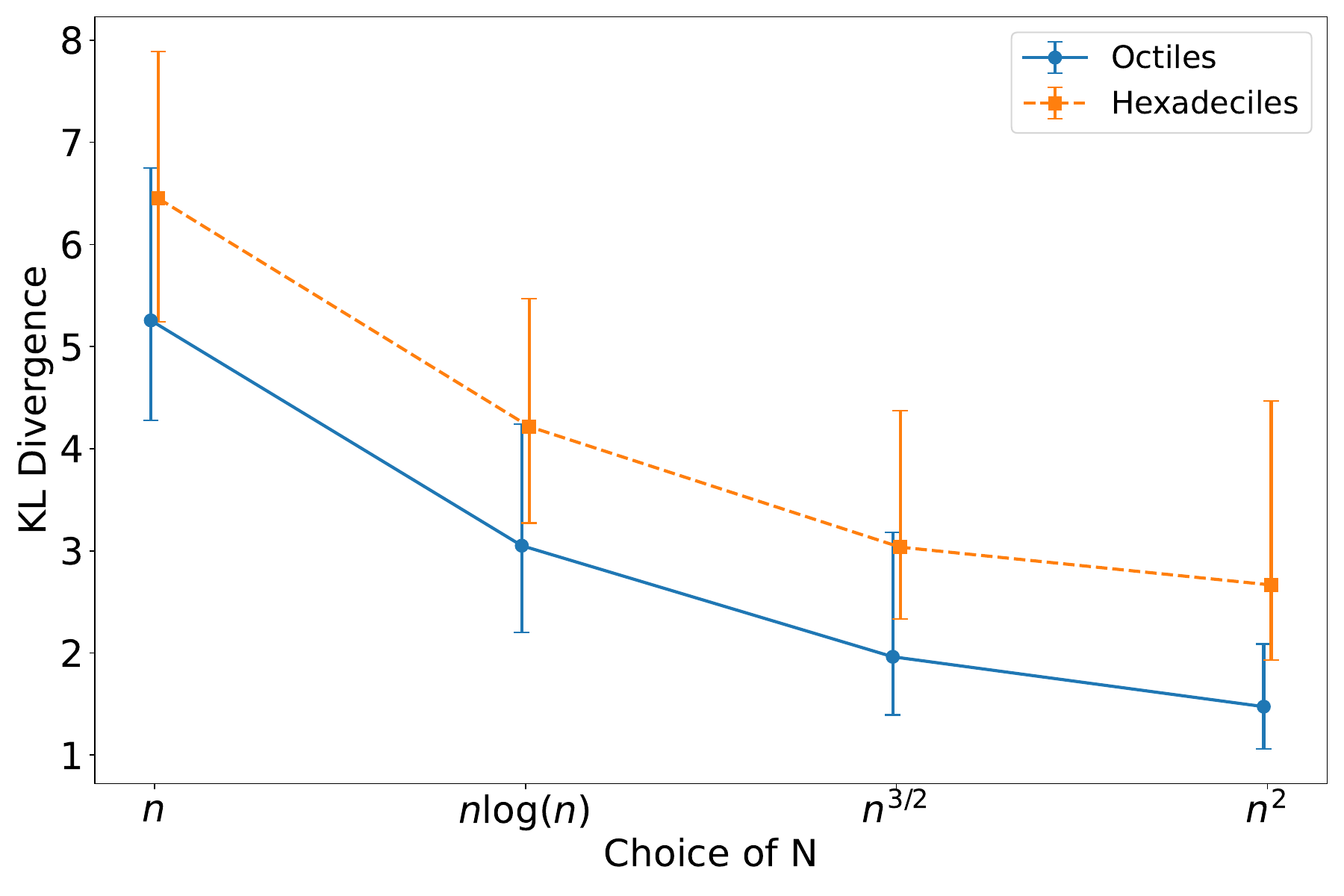}
        \caption{Number of observations $n = 1000$}
        \label{fig:gnk_summaries_1000}
    \end{subfigure}
    \caption{Comparison of $\KL\{\Pi(\cdot\mid S_n)\|\widehat{Q}_n(\cdot\mid S_n)\}$ across different choices of $N$, for summaries using octiles and hexadeciles.}
    \label{fig:gnk_kld_compare}
\end{figure}

\section{Neural Likelihood}\label{sec:snl}
In contrast to NPE, neural likelihood estimation (NLE) uses an estimated density within MCMC to produce posterior samples. As such, the theoretical results obtained for NPE do not directly transfer to NLE. We formally define the neural likelihood (NL) and NLE below.
\begin{defn}[NLE]\label{def:snl}
Let $\mathcal{Q}$ be a family of conditional density estimators for $S\mid\theta$. At termination of the algorithm, the estimator $ \wh q_N(S\mid\theta)$ satisfies 
\begin{equation}\label{eq:snle}
-\frac{1}{N}\sum_{i=1}^{n}K(\theta^i)\log \wh q_N(S^i\mid\theta^i)	\le \inf_{Q\in\mathcal{Q}} -\frac{1}{N}\sum_{i=1}^{n}K(\theta^i)\log q(S^i\mid\theta^i),
\end{equation}for some known function $K(\theta)$. We then call $\wh q_N(S_n\mid\theta)$ the neural likelihood (NL) and define the NLE posterior as $\wh \pi_N(\theta\mid S_n)\propto{\wh q_N(S_n\mid\theta)p(\theta)}$.
\end{defn}

As $N\rightarrow\infty$, NLE obtains an estimator of the intractable $g_n(S\mid\theta)$ as the minimizer of the conditional KL divergence (up to a constant that does not depend on $q$) 
$$
\KL_{\theta\sim p(\theta)}\{g_n(\cdot\mid\theta)\|q(\cdot\mid\theta)\}=\int_\Theta \int_{\mathcal{S}}p(\theta)g_n(s\mid\theta)\log\frac{p(\theta)g_n(s\mid\theta)}{q(s\mid\theta)p(\theta)}\dt\theta\dt s.
$$Consequently, to formalize the accuracy of NLE, we must maintain an assumption on the accuracy with which $\wh q_N(S\mid\theta)$ approximates $g_n(S\mid\theta)$ at $S=S_n$.
\begin{assumption}\label{ass:sieve3}
	For some $\epsilon_n,\gamma_N>0$,  $\gamma_N=o(1)$, $\epsilon_n=o(1)$, and $\epsilon_n\gtrsim \gamma_N$, 	with probability at least $1-\epsilon_n$,
$
\KL_{\theta\sim p(\theta)}\{g_n(S_n\mid\cdot)||\wh q_{N}(S_n\mid\cdot)\}\le \gamma_{N}^2.
$
\end{assumption}
\begin{remark}
	Assumption \ref{ass:sieve3}  bounds the accuracy of the corresponding NLE and is the NLE equivalent to Assumption \ref{ass:sieve} for NPE. The requirement in Assumption \ref{ass:sieve3} is precisely the quantity that NLE seeks to minimize when $N\rightarrow\infty$, and as such is a natural requirement to maintain.\footnote{Similar to the discussion after Assumption \ref{ass:sieve}, the requirement in KLD maintained in Assumption \ref{ass:sieve3} can be replaced with a version based on Hellinger distance without altering the conclusions of the stated results.} Furthermore, depending on the class of approximations used, the rate of convergence for $\wh q_N(S_n\mid\theta)$ can be the same as that in Assumption \ref{ass:sieve2}, or as fast as $\gamma_N\asymp \log(N)/N^{1/2}$ if we are willing to assume that $\Pi(\cdot\mid S_n)\in\mathcal{Q}$ (see, e.g., Theorem 2 in \citealp{forbes2024bayesian} for an example of this latter assumption). We refer the interested reader to  Section \ref{sec:specific} for further details.\footnote{ In particular, the discussion in Section \ref{sec:specific} is also applicable when phrased in terms of the conditional density estimator of $q(S\mid\theta)$, and so the rates described in that section remain applicable to the case of NLE.}
\end{remark}

Under Assumption \ref{ass:sieve3}, we obtain the following result on the concentration of the NLE.

\begin{theorem}\label{thm:two}
Assumptions \ref{ass:obs_sum}-\ref{ass:prior_concentration} are satisfied. If Assumption \ref{ass:sieve3} is also satisfied, then with probability converging to one the NLE, $\wh\pi_N(\theta\mid S_n)\propto p(\theta) \wh q_{N}(S_n\mid\theta)$  concentrates onto $\theta_0$ at rate $\epsilon_n=\log(n)/\nu_n$. If, in addition, Assumption \ref{ass:sieve2} is also satisfied for $g_n(S_n\mid\theta)$ and $q(S_n\mid\theta)$, then the result is valid so long as we train NLE on at least $N\gtrsim \nu_n^{(2\beta+d_\theta+d_s)/\beta}\log(n)^\kappa$ datasets.
\end{theorem}

\begin{remark} 
Both Bayesian synthetic likelihood (BSL) and NLE use an approximate likelihood estimator within MCMC to produce samples from a posterior approximation: BSL uses a Gaussian parametric approximation that is updated at each MCMC iteration; NLE learns a flexible approximation that is fixed across MCMC iterations. To see how the two methods compare in terms of computational cost, note that BSL simulates $m$ independent sets of data at each MCMC evaluation, and runs $N_{\text{MCMC}}$ evaluations, which requires $N_{\text{BSL}}=m \times N_{\text{MCMC}}$ total model simulations. If we take $m=\sqrt{n}$ datasets, when running BSL we then generate $\sqrt{n}\times N_{\text{BSL}}$ datasets (of size $n$) across the entire MCMC chain. The computational cost of BSL, in terms of generated datasets,  will then be more costly than NLE whenever
$$
N_{\text{BSL}}=\sqrt{n}N_{\text{MCMC}}>n^{\frac{2\beta+d_s+d_\theta}{2\beta}}\quad\implies N_{\text{MCMC}}>n^{\frac{\beta+d_s+d_\theta}{2\beta}}.
$$
If $n=1000$, $\beta=10$, $d_\theta=5$ and $d_s=10$, BSL will be more costly than NPE, in terms of model simulations, when we must run more than $6,000$ MCMC evaluations, which will almost always be the case in practice. This result implies that in many settings, NLE will be much more computationally efficient than BSL. 
\end{remark}

Similar to NPE, if $\pi(\theta\mid S_n)$ is asymptotically Gaussian, and if $N$ is chosen large enough, then $\wh \pi(\theta\mid S_n)$ will also be asymptotically Gaussian. Recall that $\theta_n=\arg\min_{\theta\in\Theta}-\log g_n(S_n\mid\theta)$, $t=\nu_n(\theta-\theta_n)$, and consider the posterior
\begin{flalign*}
\pi(t\mid S_n)=\frac{1}{\nu_n^{d_\theta}}\pi(\theta_n+t/\nu_n\mid S_n)&=\frac{g_n(S_n\mid\theta_n+t/\nu_n)p(\theta_n+t/\nu_n)}{\int g_n(S_n\mid\theta_n+t/\nu_n)p(\theta_n+t/\nu_n)\dt t}.%\\&=\frac{g_n(S_n\mid\theta_n+t/\nu_n)\pi(\theta_n+t/\nu_n)}{Z_{n,t}}.
\end{flalign*}

\begin{theorem}\label{thm:bvm-snl}If Assumptions \ref{ass:obs_sum}-\ref{ass:prior_concentration}, and \ref{ass:sieve3} are satisfied, and if 
$$
\int \left|g_n(S_n\mid\theta_n+t/\sqrt{n})\frac{p(\theta_n+t/\sqrt{n})}{p(\theta_n)}
-\exp\left(-t^\top\Sigma t/2\right)\right|\dt t=o_p(1),
$$ with $p(\theta_n)>0$ for all $n$,	
then 
$
\int |\wh \pi_N(t\mid S_n)-N\{t;0,\Sigma^{-1}\}|\dt t=o_p(1).
$
\end{theorem}

	Theorems \ref{thm:bvm} and \ref{thm:bvm-snl} provide the first rigorous basis for the literature's existing preference for NPE over NLE. Both approaches require similar regularity conditions and requirements on $N$, and both ultimately deliver posterior approximations with the same level of theoretical accuracy. However, NLE requires additional MCMC steps to produce a posterior approximation, whereas NPE produces a posterior approximation directly and does not require any additional sampling.  Therefore, based on these considerations of computation and theoretical accuracy, and when the smoothness assumptions that underpin Assumptions \ref{ass:sieve} and \ref{ass:sieve3} are satisfied, NPE should be preferable to NLE in most cases where practitioners employ SBI methods.

\section{Discussion}\label{sec:discuss}

We have shown that there is a clear link between the number of simulations used in NPE and the resulting accuracy of the posterior. While there are clear choices for too small a number of simulations one should use, we have shown that for the posterior to deliver reliable inferences the number of simulations must depend on the dimension of the problem - the number of parameters and summaries used in the analysis - as well as the flexibility of the chosen posterior approximation class. We have also shown that after choosing an appropriate number of simulations for the problem at hand, there are little to no benefits from using a larger number of simulations. 

An essential insight from our study is the critical importance of the compatibility assumption in ensuring the reliability of the NPE posterior. When this assumption is violated, it does not appear feasible to deliver theoretical guarantees on the accuracy or trustworthiness of the NPE approximation, even as the number of observations ($n$) and simulations ($N$) increase. This observation aligns with existing empirical findings that demonstrate the poor performance of NPE under model misspecification \citep{cannon2022investigating, schmitt2023detecting}. While some robust methods have been proposed to mitigate these issues \citep{huang2024learning, kelly2024misspecification, ward2022robust}, there remains a lack of rigorous theoretical treatment of the impact of incompatibility and model misspecification on NPE. In future work, we plan to explore the theoretical behavior of the NPE when the compatibility assumption is violated, and the ability of the proposed robust methods to mitigating the poor behavior of the NPE in these settings.

While our current analysis focuses on NPE methods, we conjecture that these results may serve as an upper bound for SNPE methods. Since SNPE iteratively refines the posterior estimate through multiple rounds of simulation, it is plausible that the required number of simulations $N$ can be reduced in subsequent rounds without sacrificing accuracy. This iterative refinement could potentially lead to more efficient inference compared to NPE, especially in complex models. However, formalizing this relationship and determining optimal strategies for adjusting $N$ across rounds necessitates further investigation, which we leave for future work.

{\footnotesize
\spacingset{1.0}
    \bibliographystyle{chicago}
	\bibliography{refs_mispecCDF}
}

\appendix

\section{Accurate $\mathcal{Q}$ Implies Faster Convergence Rates}\label{app:accQ} 
%\end{remark}

If one is willing to assume that $\pi(\theta\mid S_n)\in\mathcal{Q}$ for all $n$ large enough, then even faster rates for $\gamma_N$ than those obtained in Corollary \ref{cor:one} are feasible, which will entail that NPE \textit{is even more computational efficient than ABC}. However, to work out such cases, we will need an explicit guarantee on the rate at which consistent density estimation is feasible in a chosen class $\mathcal{Q}$. Since these rates are not yet known for many classes of neural-based conditional density approximations, we restrict our analysis in this section to simpler choices but note that similar results would follow once such rates become available. 

Consider that $\pi(\theta\mid S_n)\in\mathcal{M}_k$, where $\mathcal{M}_k$ is the class of $k$-component multivariate Gaussian location-scale mixtures:  for $\vartheta=(\theta,S)$,
$$
\mathcal{M}_k=\left\{\Gamma_{k}=\sum_{j=1}^{k}\pi_j\delta_{\phi_j},\;\phi_j\in\Phi_{j},\;\Phi=\cup_{j=1}^{k}\Phi_j:q(\vartheta)=\sum_{j=1}^{k}\pi_j\mathcal{N}(\vartheta;\mu_j,\Sigma_j)\right\},
$$	 where $\Sigma_1,\dots,\Sigma_d$ satisfy 
$$
\sigma_{\text{min}}^2\le \min_{j=1,\dots,d}\lambda_{\text{min}}(\Sigma_j)\le  \max_{j=1,\dots,d}\lambda_{\text{max}}(\Sigma_j)\le \sigma^2_{\text{max}}, 
$$and where $\lambda_{\text{min}}(\Sigma)$ is the smallest eigenvalue of $\Sigma$. The results of Theorem 2.5 in \cite{saha2020nonparametric}, and Theorem 1.2 in \cite{doss2023optimal}, imply that $\gamma_{N}^2\asymp (k/N)\log(N)^{d+1}C_d$, where $C_d$ is a constant that depends on $d$. While these results require that the value of $k$ is known, if instead only an upper bound on $k$, say $\bar{k}$, is known, the results of \cite{ghosal2001entropies} imply that the rate of convergence remains similar, namely $\gamma_N^2\asymp \log(N)^{2\kappa}/N^{}$, for some $\kappa>0$.\footnote{Technically, \cite{ghosal2001entropies} obtain this rate for $\gamma_N$ under the Hellinger metric, however, as we have argued in Remark \ref{rem:other_ass}, the results of Theorem \ref{thm:main} remain applicable when we have available rates for the accuracy of $\wh Q_N(\cdot\mid S_n)$ in either the total variation metric or the  Hellinger divergence. See Corollary \ref{corr:main} in Appendix \ref{app:results} for further details.} Therefore,  if $\pi(\theta\mid S_n)\in \mathcal{M}_k$ for some $k$ large, and we consider over-fitted mixtures, then the parametric rate of concentration - for the approximation itself - will be achieved so long as we take $N_{{}}\gtrsim \log(n) n$ in NPE. Moreover,  this rate requirement is true regardless the value of $d_\theta$ and $d_s$. 

A secondary class of useful posterior approximations that come with known theoretical guarantees on the possible rate of convergence of the approximation, i.e., $\gamma_N$, is the class of Gaussian mixtures of experts (\citealp{jacobs1991adaptive}). This class can be viewed as conditional Gaussian location-scale mixtures, $\mathcal{M}_k$, but where the component means and variances are covariate dependent: 
$$
\mathcal{GM}_k=\left\{\Gamma_{k}=\sum_{j=1}^{k}\pi_j\delta_{\phi_j},\;\phi_j\in\Phi_{j},\;\Phi=\cup_{j=1}^{k}\Phi_j:q(\theta\mid  S)=\sum_{j=1}^{k}\pi_j\mathcal{N}[\theta;\mu(S,\phi_{1j}),\Sigma(S,\phi_{2j})]\right\}
$$where 
$\mu(S,\phi_{1j})$ and $\Sigma(S,\phi_{2j})$ are known functions of the unknown parameters $\phi_{1j}$, $\phi_{2j}$, which are usually taken as regression functions. 

As shown in Proposition 3 of \cite{ho2022convergence}, under weak smoothness conditions on $\mu(\cdot)$ and $\Sigma(\cdot)$, when $\wh q_N(\theta \mid S)\in\mathcal{GM}_{k}$ and $\pi(\theta\mid S)\in\mathcal{GM}_{\bar{k}}$ with $k\ge \bar{k}$, with probability converging to one 
$\dt_{\mathrm{H}}\{\wh q_N(\theta\mid S),\pi(\theta\mid S)\}\le \gamma_N=\log(N)^{1/2}/N^{1/2}$. Hence, if we are willing to assume that $\pi(\theta\mid S)\in\mathcal{GM}_{\bar{k}}$, the NPE will concentrate at the parametric rate so long as we take $N\gtrsim \log(n)n$, regardless of $d_s$ or $d_\theta$.

The above discussion clarifies that the rate $\gamma_{N}$ can decay very fast to zero if we assume that $\pi(\theta\mid S_n)\in\mathcal{Q}$ for a particular class $\mathcal{Q}$. In such cases, NPE will require orders of magnitude fewer model simulations than ABC, and will deliver the same theoretical guarantees. For example, if $\pi(\theta\mid S_n)\in\mathcal{Q}$  and $d_\theta=3$, ABC must generate a reference table of size $N_{\text{ABC}}\gtrsim n^{3/2}$, while NPE must only generate a table of size $N_{}\gtrsim \log(n)n$. 

We do remark, however, that if one is only willing to assume a general form of regularity for the posterior, namely that $\pi(\theta\mid S_n)\in \mathcal{C}^{\beta,L,0}(\mathbb{R}^d)$, then the slower rate of $\gamma_{N}\asymp N^{-\beta/(2\beta+(d_\theta+d_s))}\log(n)$ should be assumed. The latter would be important in situations where the dependence of the posterior cannot be easily speculated on, such as, e.g., situations where the parameters have differing supports, non-elliptical dependence, or when some of the summary statistics have discrete support. 

\section{Further Details and Results for the Stereological Example}\label{stereo}
We provide additional details here on the stereological example described in Section~\ref{sec:motiv}.\footnote{Code to reproduce the results for all examples is available at \url{https://github.com/RyanJafefKelly/npe_convergence}.} Direct observation of the three-dimensional inclusions within steel is not possible; instead, only their two-dimensional cross-sections can be observed in practice. Consequently, a mathematical model is required to relate the observed cross-sectional diameter, denoted by $S$, to the true latent diameter of the inclusion, denoted by $V$.

We follow the mathematical model described in \citet{bortot2007inference}. Since the largest inclusions are primarily responsible for fatigue in the material, modeling focuses on inclusions larger than a threshold $\nu_0$ (which we set  $\nu_0=5$). These large inclusions are modeled using a generalized Pareto distribution:

\begin{equation*} P(V \leq \nu \mid V > \nu_0) = 1 - \left\{ 1 + \frac{\xi(\nu - \nu_0)}{\sigma} \right\}_{+}^{-1/\xi}, \end{equation*}

where $\sigma > 0$ is the scale parameter, $\xi$ is the shape parameter, and $\{\cdot\}_{+}$ denotes the positive part function, ensuring the distribution is defined for $\nu \geq \nu_0$. We assume that inclusions occur independently.

Early work on this problem assumed that inclusions are spherical in shape. However, following the approach of \citet{bortot2007inference}, we model the inclusions as ellipsoids to more accurately reflect their geometric properties. In this model, each inclusion is characterized by its three principal diameters $(V_1, V_2, V_3)$, where we can assume without loss of generality that $V_3$ is the largest. The smaller diameters are proportional to $V_3$, defined as $V_1 = U_1 V_3$ and $V_2 = U_2 V_3$, where $U_1$ and $U_2$ are independent random variables uniformly distributed on $(0, 1)$. When an ellipsoid is intersected by a random two-dimensional plane (the cross-section), the resulting shape is an ellipse. The observed quantity $S$ is then the largest principal diameter of this elliptical cross-section.

We use the same four-dimensional vector of summaries as in \citet{an2020robust} to conduct inference on the unknown parameters. These summaries are the number of inclusions, and the logarithms of the mean, minimum, and maximum of the two-dimensional cross-sectional diameters of the inclusions.
We specify uniform prior distributions for the model parameters, $\lambda \sim \mathcal{U}(30, 200)$, $\sigma \sim \mathcal{U}(0, 15)$ and $\xi \sim \mathcal{U}(-3, 3)$.

Additional results for the parameters $\sigma$ and $\xi$ are provided in Tables \ref{tab:stereo_coverage_sigma} and \ref{tab:stereo_coverage_xi}, which present the Monte Carlo coverage of the 80\%, 90\%, and 95\% credible intervals across different 
choices of $N$ and $n$, based on 100 repeated runs.
Figures \ref{fig:stereo_sigma_boxplot} and \ref{fig:stereo_xi_boxplot} display the bias of the posterior mean for $\sigma$ and $\xi$, 
respectively, visualized through boxplots across varying $n$ and $N$.
Univariate posterior approximations of $\sigma$ and $\xi$ for a single dataset with $n = 1000$ observations are shown in Figures \ref{fig:stereo_sigma_posterior} and \ref{fig:stereo_xi_posterior}, comparing NPE approximations using different numbers of simulations.

\begin{table}[!htbp]
\centering
\begin{tabular}{@{}l l l l l @{} }\toprule
 & $N=n$ &  $N=n\log(n)$ & $N=n^{3/2}$ & $N=n^2$ \\ 
\hline
$n=100$ &1.00/1.00/1.00 &1.00/1.00/1.00 &0.98/1.00/1.00 &0.88/0.96/0.99 \\ 
$n=500$ &1.00/1.00/1.00 &0.96/0.98/0.99 &0.90/0.97/0.99 &0.81/0.91/0.96 \\ 
$n=1000$ &0.99/1.00/1.00 &0.95/0.99/1.00 &0.85/0.94/0.98 &0.82/0.92/0.96 \\ 
$n=5000$ &0.97/0.99/1.00 &0.90/0.97/0.99 &0.81/0.91/0.96 &0.84/0.94/0.99 \\ 
\bottomrule
\end{tabular}
\caption{Monte Carlo coverage of 80\%, 90\%, and 95\% credible intervals for $\sigma$ across different choices of $N$ and $n$, based on 100 repeated runs.}
\label{tab:stereo_coverage_sigma}
\end{table}

\begin{table}[!htb]
\centering
\begin{tabular}{@{}l l l l l @{} }\toprule
 & $N=n$ &  $N=n\log(n)$ & $N=n^{3/2}$ & $N=n^2$ \\ 
\hline
$n=100$ & 0.26/0.49/0.79 &0.34/0.67/0.95 &0.83/0.88/0.98 &0.93/0.98/1.00 \\ 
$n=500$ & 0.30/0.68/0.96 &1.00/0.99/1.00 &0.94/0.98/1.00 &0.83/0.93/0.97 \\ 
$n=1000$ & 0.69/0.84/0.99 &0.99/0.99/1.00 &0.90/0.97/0.99 &0.84/0.93/0.97 \\ 
$n=5000$ & 0.98/0.99/1.00 &0.91/0.97/0.99 &0.89/0.97/0.99 &0.92/0.97/0.99 \\ 
\bottomrule
\end{tabular}
\caption{Monte Carlo coverage of 80\%, 90\%, and 95\% credible intervals for $\xi$ across different choices of $N$ and $n$, based on 100 repeated runs.}
\label{tab:stereo_coverage_xi}
\end{table}

\begin{figure}[!htb]
    \centering
    \begin{subfigure}[b]{0.49\textwidth}
        \centering
        \includegraphics[width=\textwidth]{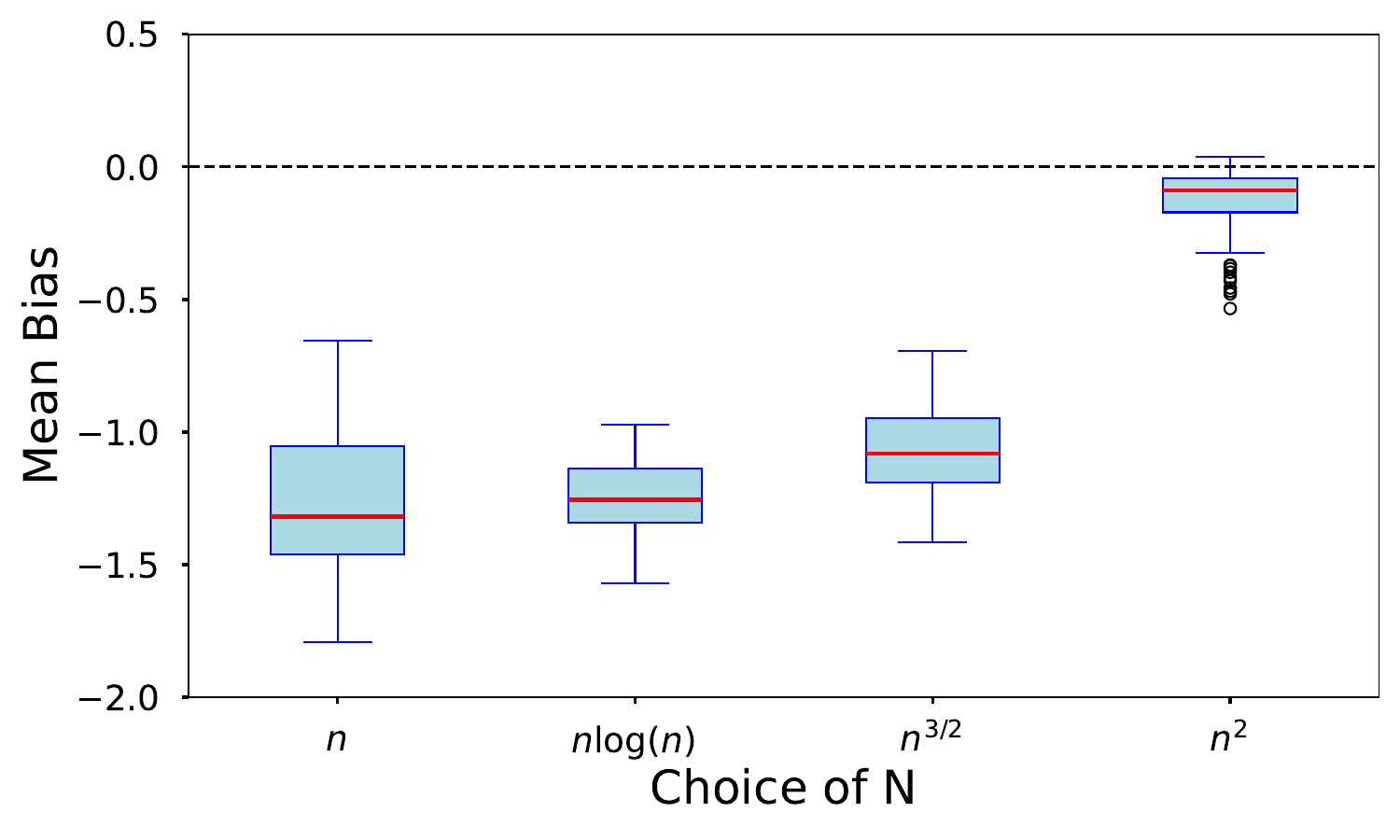}
        \caption{$n=100$}
    \end{subfigure}
    \hfill
    \begin{subfigure}[b]{0.49\textwidth}
        \centering
        \includegraphics[width=\textwidth]{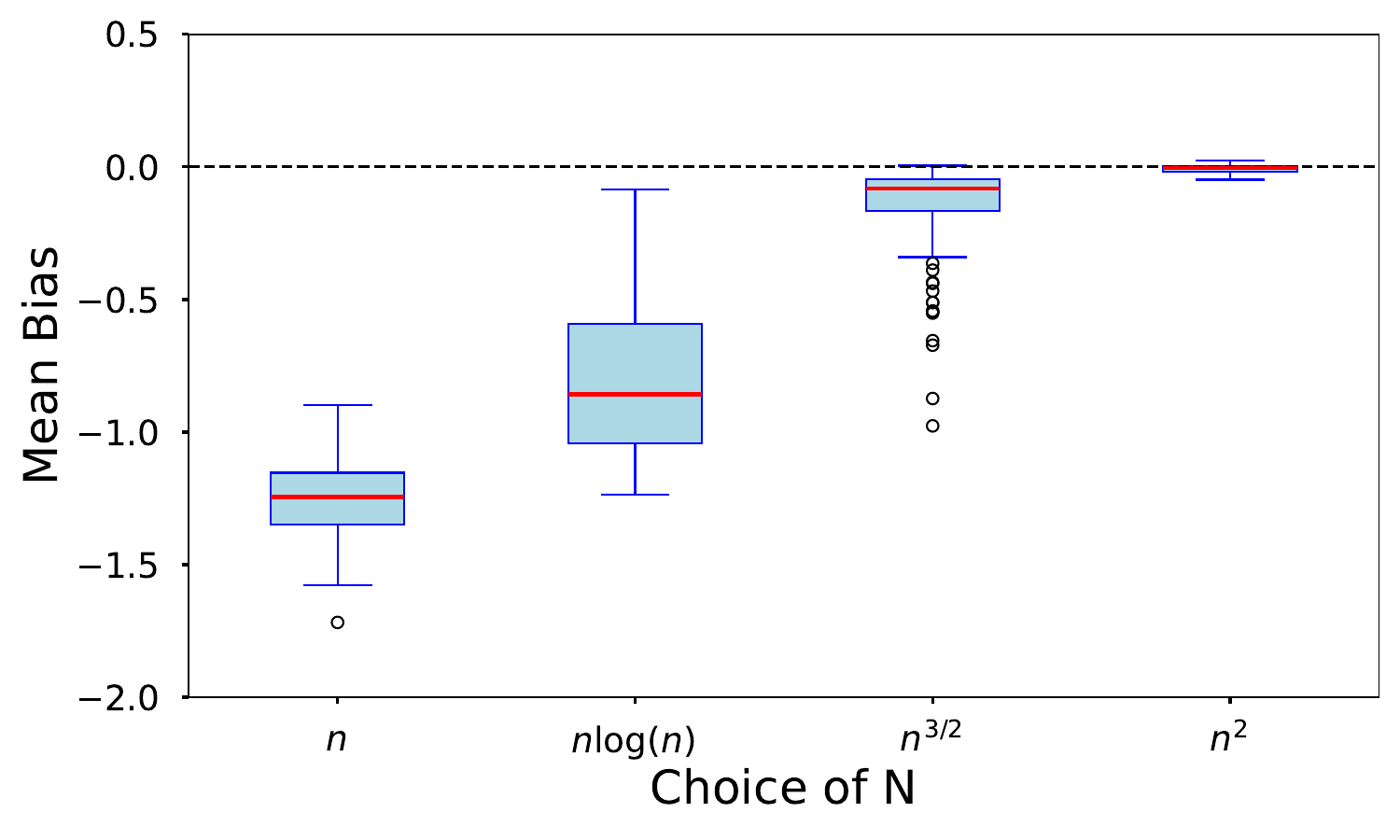}
        \caption{$n=500$}
    \end{subfigure}
    
    \vspace{5pt}

    \begin{subfigure}[b]{0.49\textwidth}
        \centering
        \includegraphics[width=\textwidth]{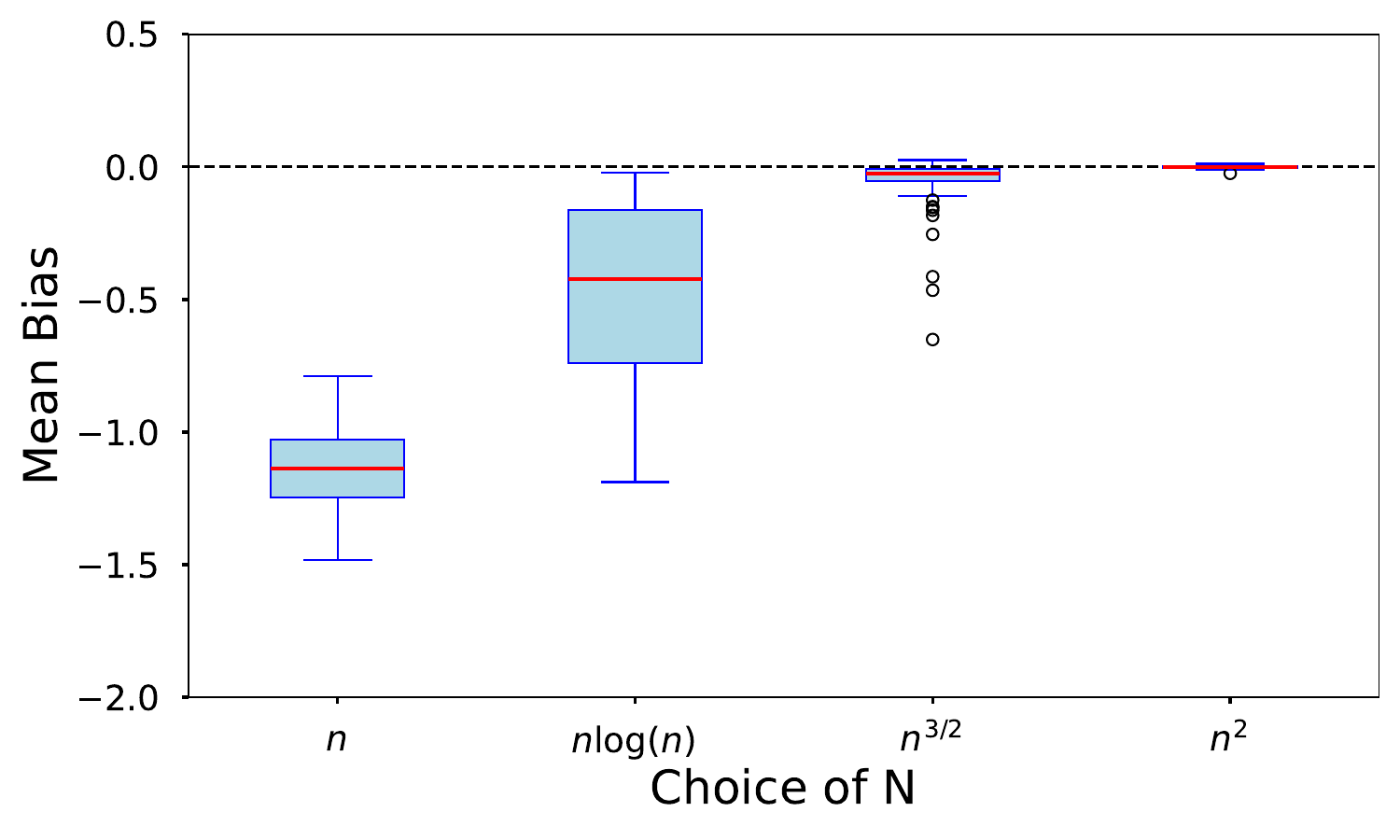}
        \caption{$n=1000$}
    \end{subfigure}
    \hfill
    \begin{subfigure}[b]{0.49\textwidth}
        \centering
        \includegraphics[width=\textwidth]{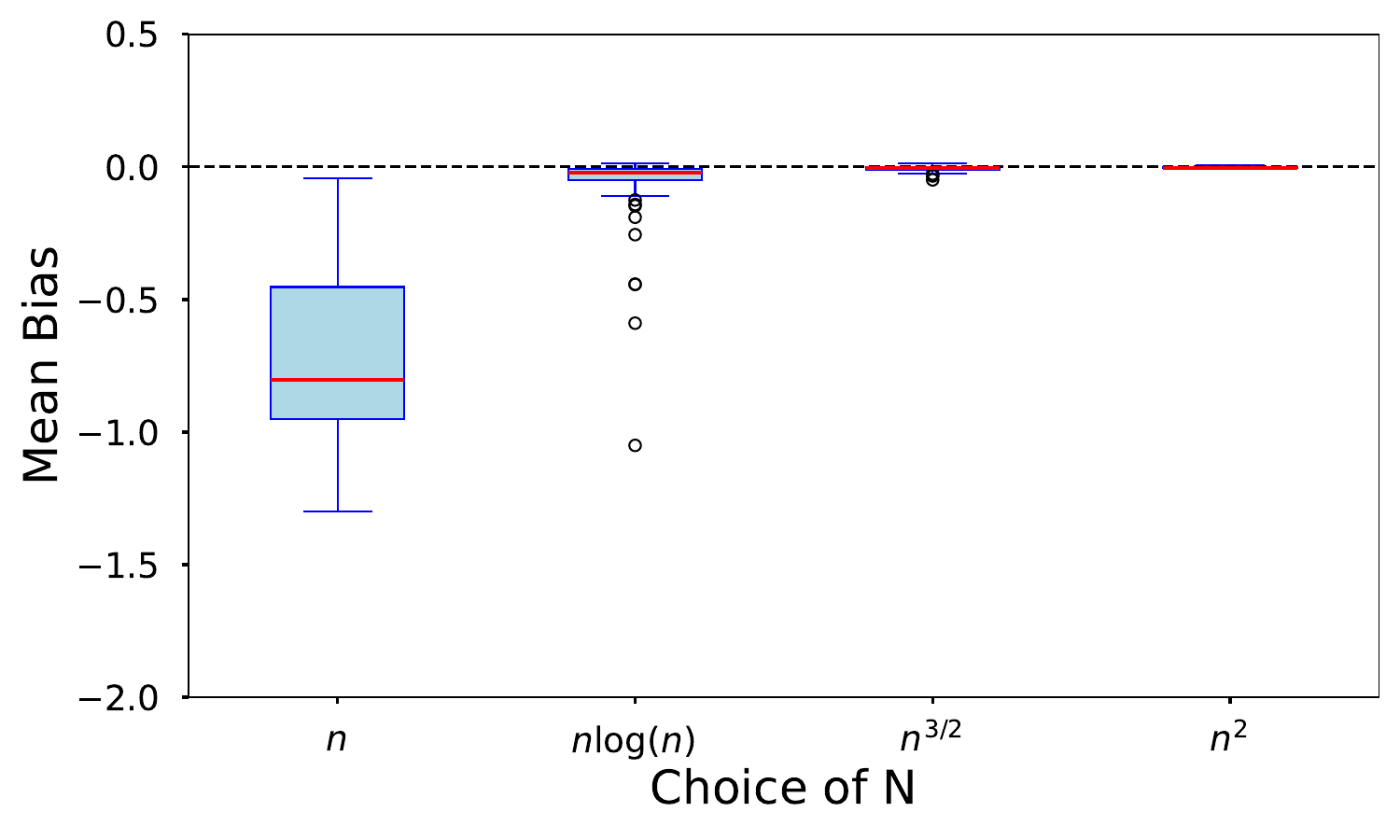}
        \caption{$n=5000$}
    \end{subfigure}
    
    \caption{Bias of the posterior mean for $\sigma$ visualized through boxplots across varying $n$ and $N$.}
    \label{fig:stereo_sigma_boxplot}
\end{figure}

\begin{figure}[!htbp]
    \centering
    \begin{subfigure}[b]{0.49\textwidth}
        \centering
        \includegraphics[width=\textwidth]{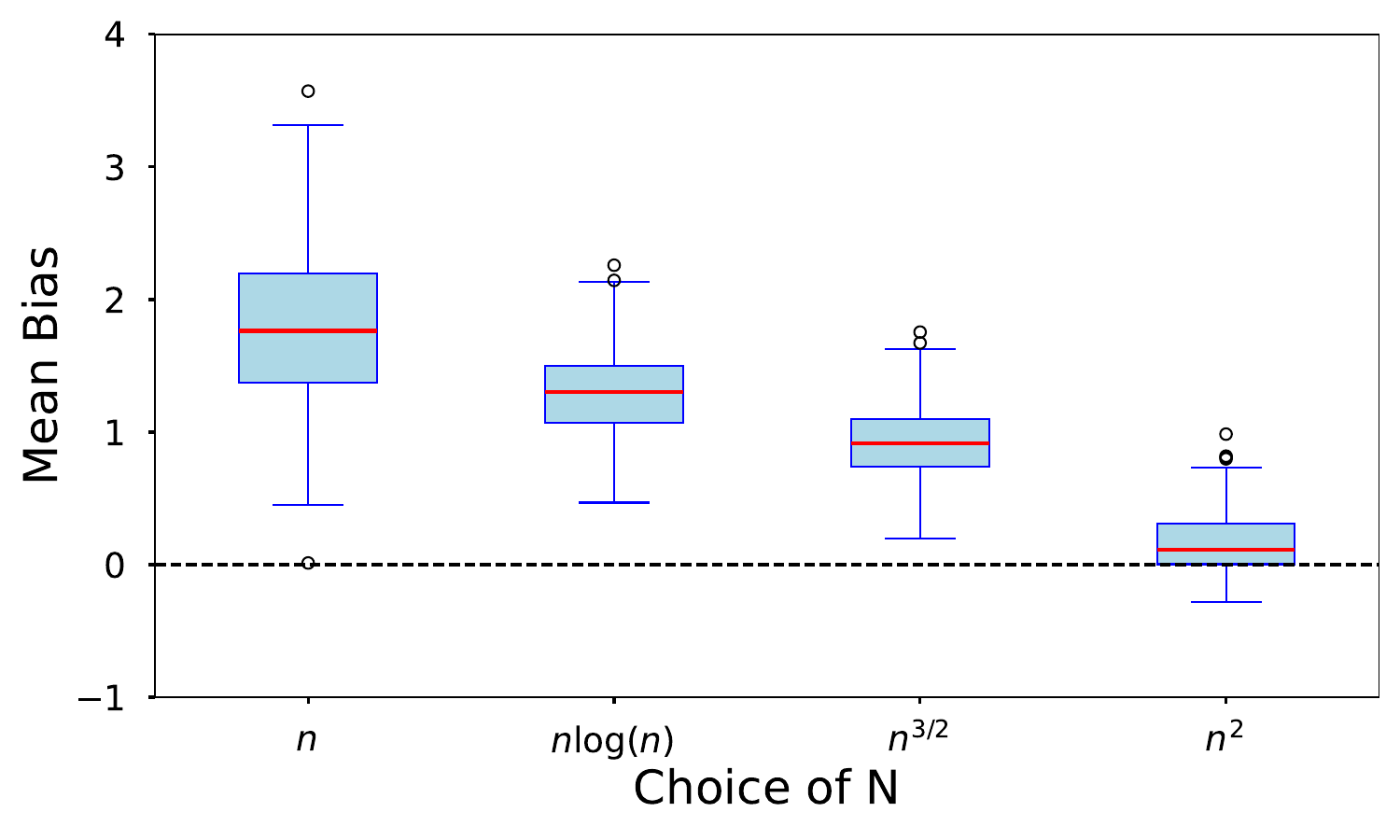}
        \caption{$n=100$}
    \end{subfigure}
    \hfill
    \begin{subfigure}[b]{0.49\textwidth}
        \centering
        \includegraphics[width=\textwidth]{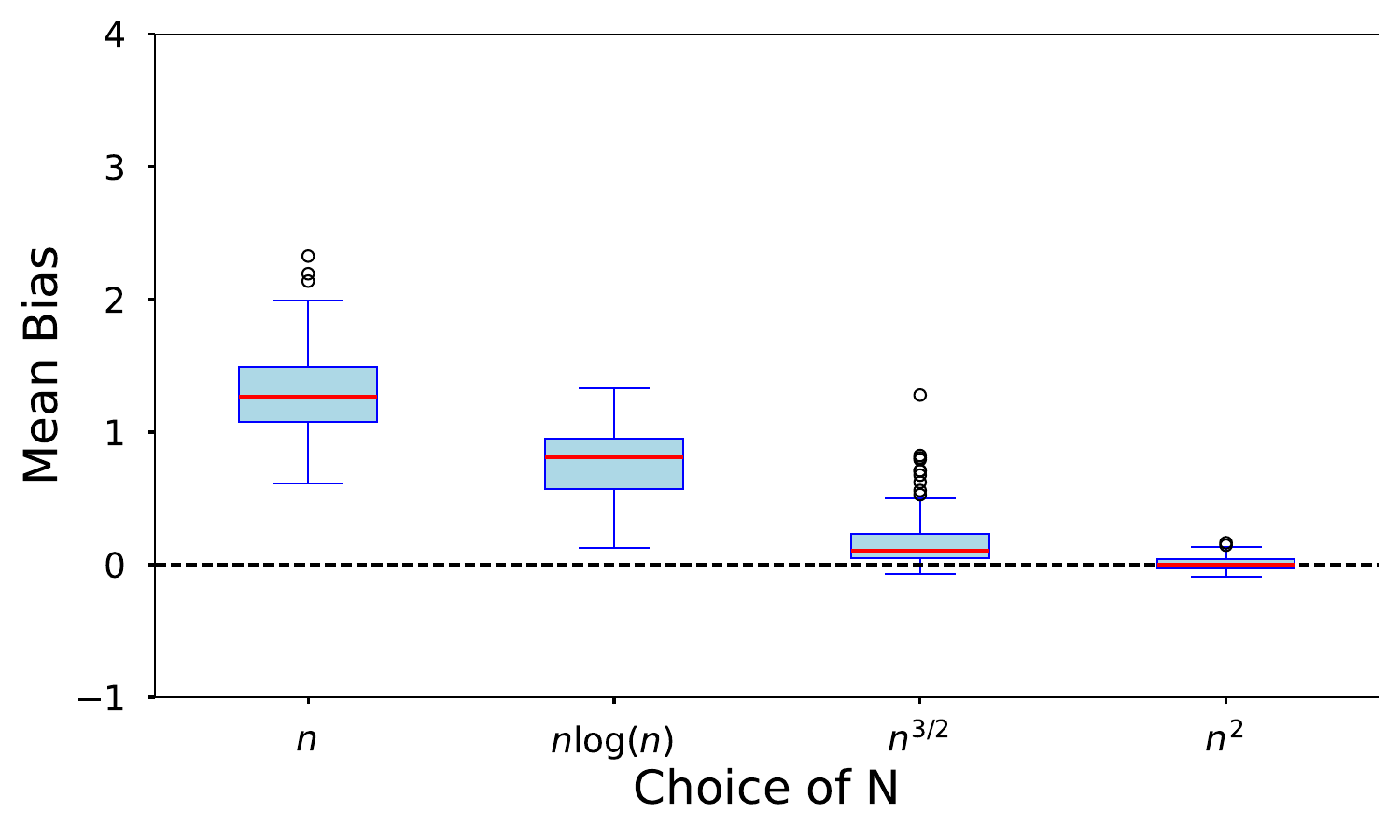}
        \caption{$n=500$}
    \end{subfigure}
    
    \vspace{5pt}
    
    \begin{subfigure}[b]{0.49\textwidth}
        \centering
        \includegraphics[width=\textwidth]{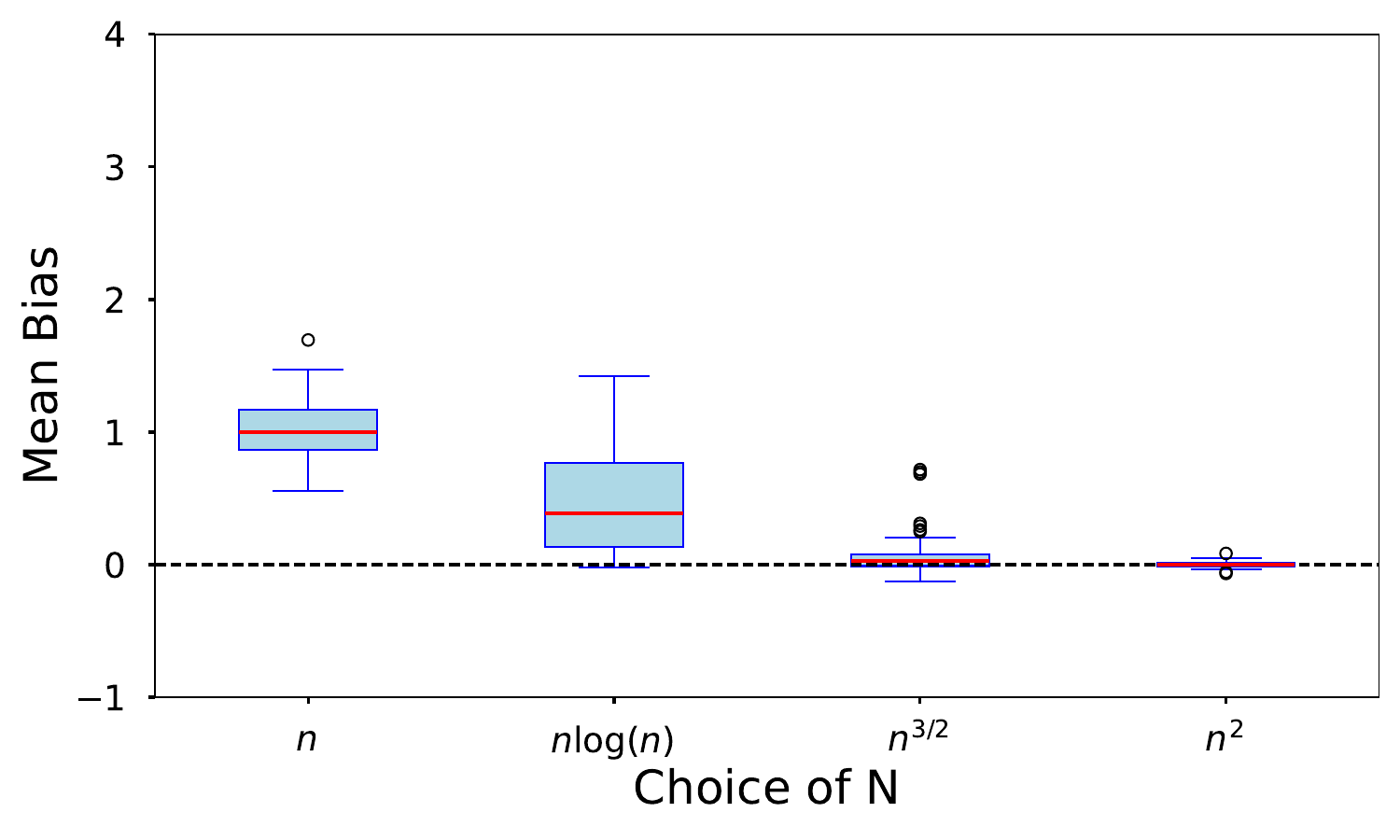}
        \caption{$n=1000$}
    \end{subfigure}
    \hfill
    \begin{subfigure}[b]{0.49\textwidth}
        \centering
        \includegraphics[width=\textwidth]{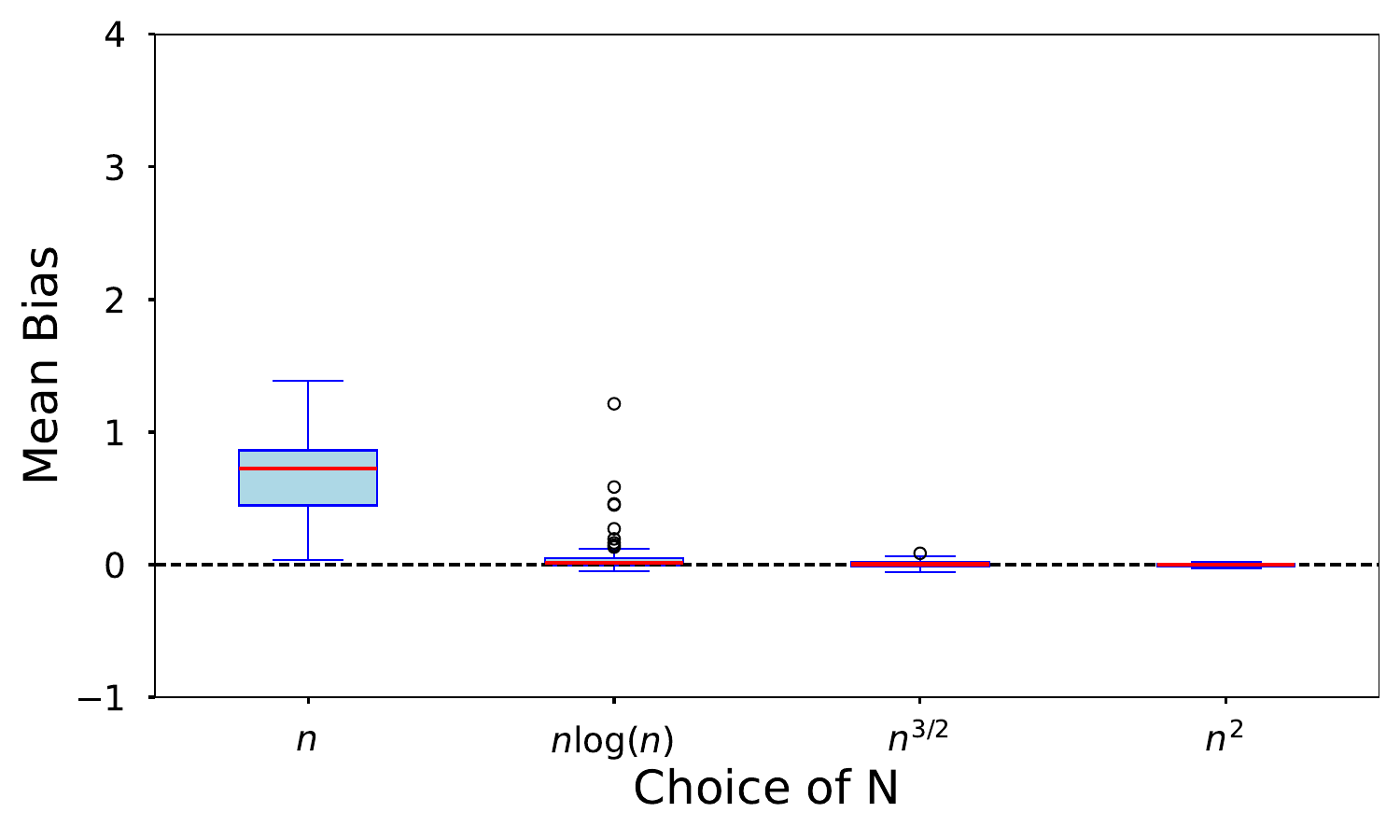}
        \caption{$n=5000$}
    \end{subfigure}
    
    \caption{Bias of the posterior mean for $\xi$ visualized through boxplots across varying $n$ and $N$.}
    \label{fig:stereo_xi_boxplot}
\end{figure}
\begin{figure}[!htbp]
    \centering
    \includegraphics[width=0.6\textwidth]{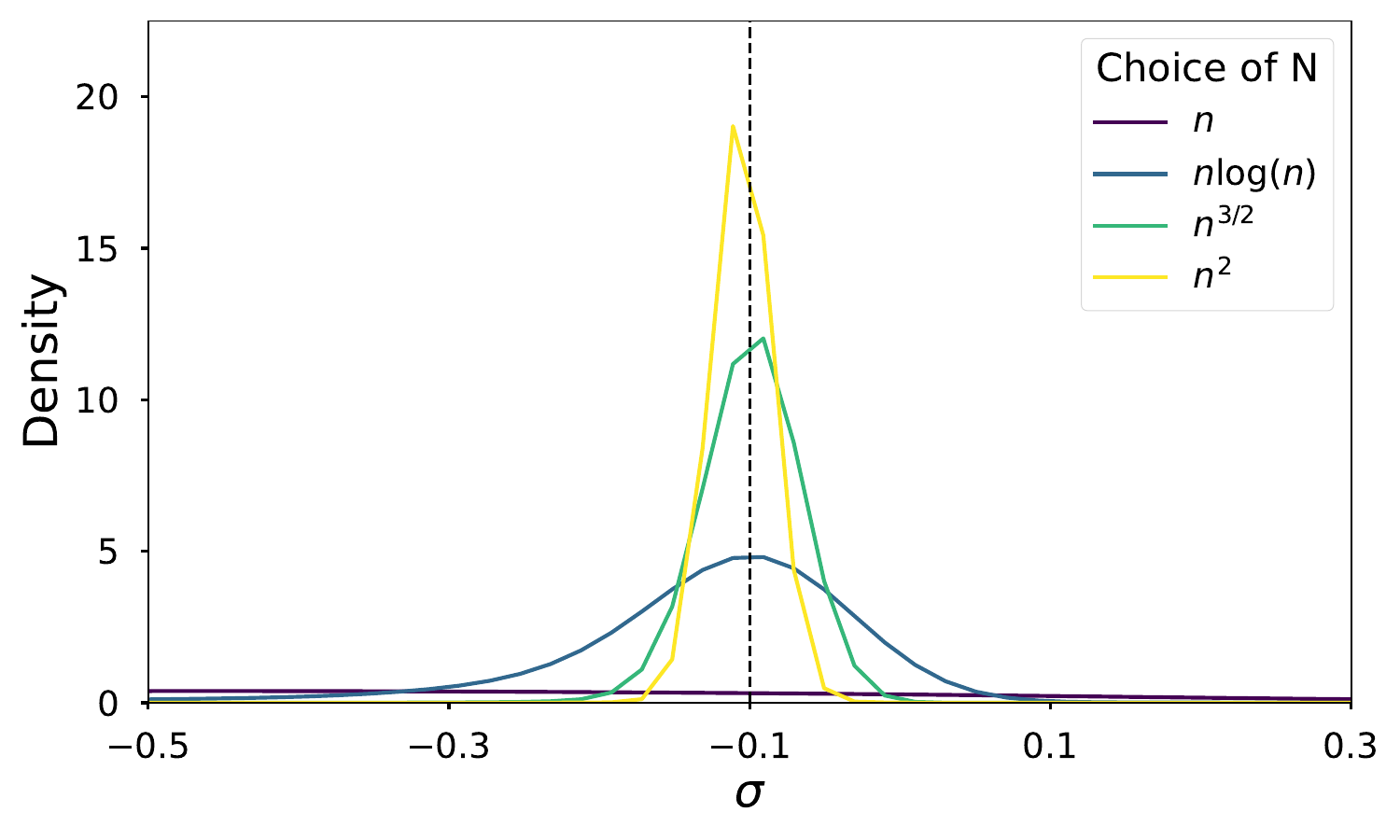}
    \caption{Univariate posterior approximations of $\sigma$ for a single dataset generated from the stereological model with $n = 1000$ observations, comparing NPE approximations using different numbers of simulations.}
    \label{fig:stereo_sigma_posterior}
\end{figure}

\begin{figure}[!htbp]
    \centering
    \includegraphics[width=0.6\textwidth]{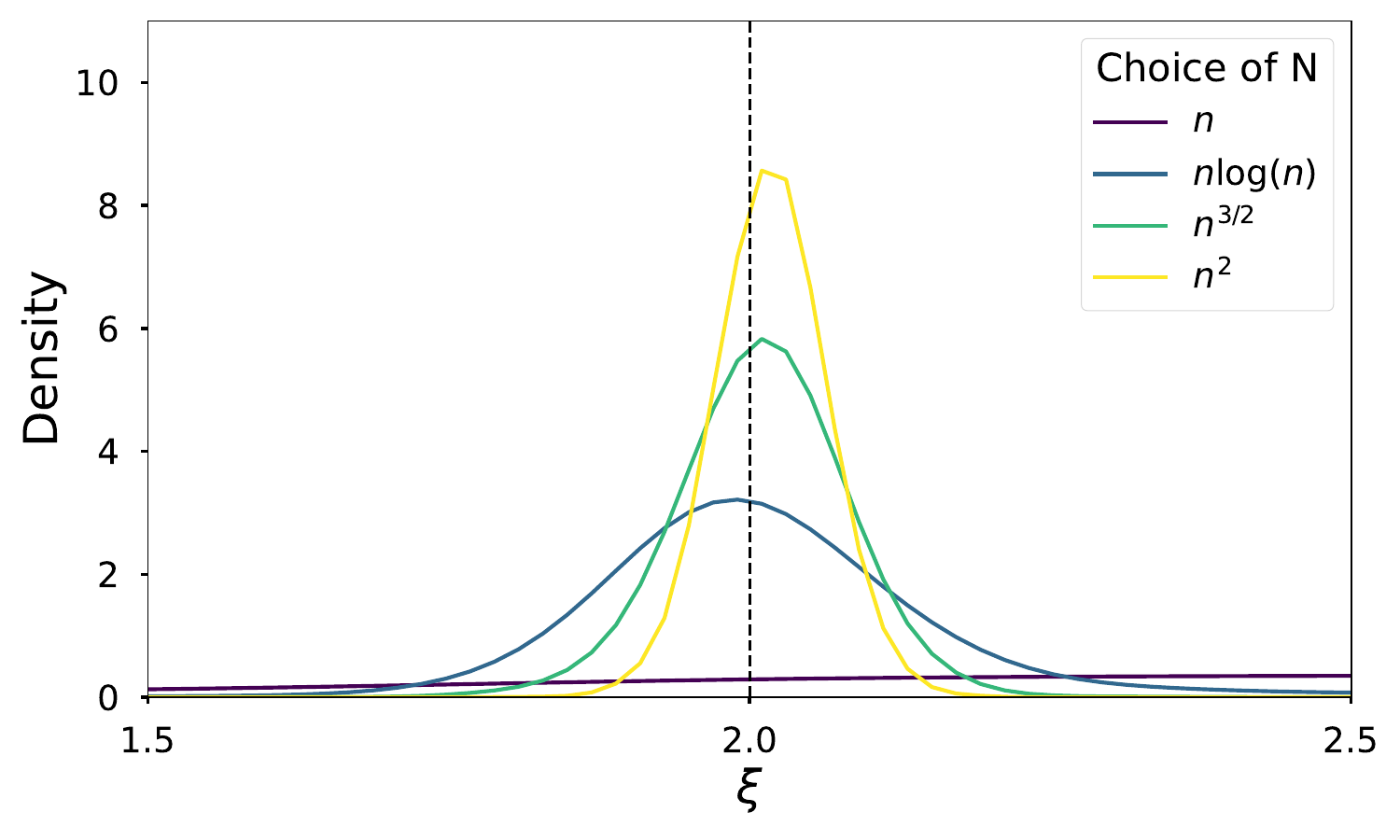}
    \caption{Univariate posterior approximations of $\xi$ for a single dataset generated from the stereological model with $n = 1000$ observations, comparing NPE approximations using different numbers of simulations.}
    \label{fig:stereo_xi_posterior}
\end{figure}

\FloatBarrier

\section{Further Details and Results for the MA(2) example}\label{sec:ma2_appendix}

In Figure~\ref{fig:ma2_kld_compare}, we estimated the KLD using the estimator proposed by \citet{perez-cruz_kullback-leibler_2008}, which computes the KLD from samples of two distributions via a $k$-nearest neighbor density estimation method. Samples from the approximate posterior are straightforwardly obtained from our NPE posterior. For the exact posterior, we sampled from the partial posterior using tempered sequential Monte Carlo \citep[see][chap. 17]{chopin2020introduction}, which is able to handle the bi-modality present in this example when the summaries have high incompatibility. The exact partial posterior is asymptotically given as follows.

The observed summaries---the sample autocovariances of lags 1 and 2 and the sample variance---are asymptotically normally distributed. Specifically, the means of these summaries, for a white noise variance of one, are: $S_n(y) = [\delta_0(y), \delta_1(y), \delta_2(y)] = [1 + \theta_1^2 + \theta_2^2, \theta_1 + \theta_1\theta_2, \theta_2$]. The variance of the summaries are computed using asymptotic results for the variance of sample autocovariances \citep{shumway2000time}. The large sample covariance between the sample autocovariances at lags $k_1$ and $k_2$:
\[
\operatorname{Cov}(\wh{\delta}_{k_1}, \wh{\delta}_{k_2}) = \frac{1}{n} \left[ \sum_{h=-q}^{q} \delta_h \delta_{h + k_1 - k_2} + \sum_{i=-q}^{q} \delta_{k_1 + i} \delta_{k_2 - i} \right],
\]
where we set $q=2$ is the order of the MA(2) model, and $\delta_h=0$ for $h \notin \{0, 1, 2\}$. 

\section{Further Details and Results for the g-and-k Example}\label{sec:gnk_appendix}
Tables \ref{tab:gnk_coverage_A}, \ref{tab:gnk_coverage_B}, and \ref{tab:gnk_coverage_k} present the Monte Carlo coverage of the 80\%, 90\%, and 95\% credible intervals for the parameters $A$, $B$, and $k$, respectively, across different choices of $N$ and sample sizes $n$, based on 100 repeated runs.

To evaluate the performance of the posterior approximation, we estimate the KLD between samples drawn from the exact partial posterior and the NPE posterior approximation.
Since the summary statistics used are order statistics, which are asymptotically normal \citep[see][chap. 21]{van2000asymptotic}, we assume normality and conduct sampling using MCMC. 
The mean vector of the multivariate normal distribution is obtained by evaluating the g-and-k quantile function at the chosen quantile levels.  The covariance matrix entries are given by the asymptotic variances and covariances of the order statistics, calculated using:
\[
\operatorname{Var}(X_{(i)}) = \frac{p_i (1 - p_i)}{n [f(Q(p_i))]^2},
\]

and for \( i \neq j \):

\[
\operatorname{Cov}(X_{(i)}, X_{(j)}) = \frac{\min(p_i, p_j) - p_i p_j}{n\, f(Q(p_i)) f(Q(p_j))},
\]

where \( f(Q(p_i)) \) is the density of the g-and-k distribution evaluated at the quantile \( Q(p_i) \), and \( n \) is the sample size.

To compute the KLD, we again employ the \citet{perez-cruz_kullback-leibler_2008} sample-based estimator between samples from the NPE approximation \(\hat{\mathcal{Q}}(\cdot \mid S_n)\) and the exact posterior, \(\Pi(\cdot \mid S_n)\). Table \ref{tab:gnk_kl} reports the estimated KLD values, illustrating that the divergence tends to decrease as both $n$ and $N$ increase, indicating improved posterior approximation accuracy with larger sample sizes and more simulations.

\begin{table}[!htbp]
\centering
\begin{tabular}{@{}l l l l l @{} }\toprule
 & $N=n$ &  $N=n\log(n)$ & $N=n^{3/2}$ & $N=n^2$ \\ 
\hline
$n=100$ &0.99/1.00/1.00 &1.00/1.00/1.00 &0.99/1.00/1.00 &0.85/0.94/0.98 \\ 
$n=500$ &1.00/1.00/1.00 &0.89/0.97/1.00 &0.78/0.90/0.95 &0.76/0.87/0.93 \\ 
$n=1000$ &0.99/1.00/1.00 &0.83/0.92/0.97 &0.78/0.88/0.94 &0.80/0.90/0.94 \\ 
$n=5000$ &0.94/0.99/1.00 &0.83/0.93/0.97 &0.78/0.89/0.94 &0.80/0.90/0.95 \\ 
\bottomrule\end{tabular}
\caption{Monte Carlo coverage of 80\%, 90\%, and 95\% credible intervals for $A$ across different choices of $N$ and $n$, based on 100 repeated runs.}
\label{tab:gnk_coverage_A}
\end{table}

\begin{table}[!htbp]
\centering
\begin{tabular}{@{}l l l l l @{} }\toprule
 & $N=n$ &  $N=n\log(n)$ & $N=n^{3/2}$ & $N=n^2$ \\ 
\hline
$n=100$ &0.96/1.00/1.00 &1.00/1.00/1.00 &1.00/1.00/1.00 &0.97/1.00/1.00 \\ 
$n=500$ &1.00/1.00/1.00 &0.97/1.00/1.00 &0.91/0.99/1.00 &0.84/0.92/0.97 \\ 
$n=1000$ &0.95/0.99/1.00 &0.91/0.98/1.00 &0.89/0.97/0.99 &0.84/0.94/0.97 \\ 
$n=5000$ &0.94/0.99/1.00 &0.90/0.97/0.99 &0.87/0.95/0.98 &0.84/0.93/0.97 \\ 
\bottomrule\end{tabular}
\caption{Monte Carlo coverage of 80\%, 90\%, and 95\% credible intervals for $B$ across different choices of $N$ and $n$, based on 100 repeated runs.}
\label{tab:gnk_coverage_B}

\end{table}

\begin{table}[!htbp]
\centering
\begin{tabular}{@{}l l l l l @{} }\toprule
 & $N=n$ &  $N=n\log(n)$ & $N=n^{3/2}$ & $N=n^2$ \\ 
\hline
$n=100$ &0.80/0.89/0.94 &0.98/1.00/1.00 &1.00/1.00/1.00 &1.00/1.00/1.00 \\ 
$n=500$ &0.92/0.99/1.00 &0.99/1.00/1.00 &0.97/1.00/1.00 &0.92/0.98/1.00 \\ 
$n=1000$ &0.97/1.00/1.00 &0.97/1.00/1.00 &0.97/1.00/1.00 &0.87/0.96/0.99 \\ 
$n=5000$ &0.96/1.00/1.00 &0.98/1.00/1.00 &0.94/0.99/1.00 &0.87/0.94/0.97 \\ 
\bottomrule\end{tabular}
\caption{Monte Carlo coverage of 80\%, 90\%, and 95\% credible intervals for $k$ across different choices of $N$ and $n$, based on 100 repeated runs.}
\label{tab:gnk_coverage_k}
\end{table}

\begin{table}[!htbp]
\centering
\begin{tabular}{@{}lcccc@{}}
\toprule
 & $N=n$ & $N=n\log(n)$ & $N=n^{3/2}$ & $N=n^2$ \\ 
\midrule
$n=100$  & 7.71 (1.21) & 6.71 (0.87) & 5.71 (0.95) & 3.62 (1.13) \\ 
$n=500$  & 5.87 (0.79) & 4.00 (0.48) & 2.52 (0.43) & 1.55 (0.30) \\ 
$n=1000$  & 5.40 (0.80) & 3.18 (0.77) & 2.12 (0.69) & 1.53 (0.51) \\ 
$n=5000$  & 4.84 (0.91) & 2.40 (0.61) & 1.79 (0.62) & 1.77 (0.58) \\ 
\bottomrule
\end{tabular}
\caption{Estimated KL divergence between $\Pi(\cdot \mid S_n)$ and $\hat{\mathcal{Q}}(\cdot \mid S_n)$ for the g-and-k example. The table shows the mean of the estimated KL across repeated runs and its standard deviation in parentheses.}
\label{tab:gnk_kl}
\end{table}

\FloatBarrier

\section{Proofs of Technical Results}

\subsection{Lemmas}\label{app:lemmas}
In this section, we give several lemmas that are used to help prove the main results. The following result bounds the expectation of the NPE.  
\begin{lemma}\label{lem:general1_TV}
	For $\wh Q_N(\cdot\mid S_n)$ defined in \eqref{eq:snpe}, 
	\begin{flalign*}
		&\E_{0}^{(N,n)}\wh Q_N\left [L\{b_n(\theta), b_0\}>M_n(\epsilon_n+\gamma_N)\mid S_n\right] \\&\le \left[\E_{0}^{(N,n)}\KL\{\Pi(\cdot\mid b_0),\widehat{Q}_N(\cdot\mid b_0)\}\right]^{1/2}+\sqrt{\tr \E_{0}^{(n)}\{(S_n-b_0)(S_n-b_0)^\top\}}\\&+  \E_{0}^{(n)}\Pi[L\{b_n(\theta), b_0\}>M_n(\epsilon_n+\gamma_N)|S_n].
	\end{flalign*}
\end{lemma}
\begin{proof}
	Write
	\begin{flalign*}
		&\wh Q_N\left [L\{b_n(\theta), b_0\}>M_n(\epsilon_n+\gamma_N)\mid S_n\right]\\&=\Pi[L\{b_n(\theta), b_0\}>M_n(\epsilon_n+\gamma_N)|S_n]\\&+\wh Q_N\left [L\{b_n(\theta), b_0\}>M_n(\epsilon_n+\gamma_N)\mid S_n\right]-\Pi[L\{b_n(\theta), b_0\}>\epsilon_n|S_n],		
	\end{flalign*}so that
	\begin{flalign}
		&\wh Q_N\left [L\{b_n(\theta), b_0\}>M_n(\epsilon_n+\gamma_N)\mid S_n\right]\nonumber\\\le& \dt_{\mathrm{TV}}\{\Pi(\cdot\mid S_n),\widehat{Q}_N(\cdot\mid S_n)\}+\Pi[L\{b_n(\theta), b_0\}>M_n(\epsilon_n+\gamma_N)|S_n].\label{eq:newbound}		
	\end{flalign}
	Bound the first term using, in turn, the triangle inequality, Assumption \ref{ass:lipz} and Lemma \ref{lem:pinsk}:
	\begin{flalign*}
		\dt_{\mathrm{TV}}\{\Pi(\cdot\mid S_n),\widehat{Q}_N(\cdot\mid S_n)\}\le& 		\dt_{\mathrm{TV}}\{\Pi(\cdot\mid S_n),\Pi(\cdot\mid b_0)\}+\dt_{\mathrm{TV}}\{\widehat{Q}_N(\cdot\mid S_n),\wh Q_N(\cdot\mid b_0)\}\\+&	\dt_{\mathrm{TV}}\{\Pi(\cdot\mid b_0),\wh Q_N(\cdot\mid b_0)\}\\\le& 	2C\|S_n-b_0\|+\dt_{\mathrm{TV}}\{\Pi(\cdot\mid b_0),\wh Q_N(\cdot\mid b_0)\}\\\le& 2C\|S_n-b_0\|+\sqrt{\KL\{\Pi(\cdot\mid b_0)||\wh Q_N(\cdot\mid b_0)\}}.
	\end{flalign*}
	Hence, we have the following bound for \eqref{eq:newbound}:
	\begin{flalign*}
		&\wh Q_N\left [L\{b_n(\theta), b_0\}>M_n(\epsilon_n+\gamma_N)\mid S_               n\right]\\&\le \sqrt{\KL\{\Pi(\cdot\mid b_0),\widehat{Q}_N(\cdot\mid b_0)\}}+2C\|S_n-b_0\|\\&+\Pi[L\{b_n(\theta), b_0\}>M_n(\epsilon_n+\gamma_N)|S_n].%\\
		%		\le& \KL\{\Pi(\cdot\mid S_n),\widehat{Q}_N\}^{1/2}+\Pi[L\{b_n(\theta), b_0\}>\epsilon_n|S_n]\\&+2C\|S_n-b_0\
	\end{flalign*}%where the second inequality follows since $\KL\{\Pi(\cdot\mid S_n),\widehat{Q}_N\}\ge0$. 
	Write $\|x\|=\sqrt{\|x\|^2}=\sqrt{\tr(xx^\top)}$, and apply Jensen's inequality to both terms to obtain 
	\begin{flalign*}
		&\E_{0}^{(N,n)}\wh Q_N\left [L\{b_n(\theta), b_0\}>M_n(\epsilon_n+\gamma_N)\mid S_n\right]\\\lesssim& \E_{0}^{(N,n)}\KL\{\Pi(\cdot\mid b_0),\widehat{Q}_N(\cdot\mid b_0)\}^{1/2}+\E_{0}^{(n)}\|S_n-b_0\|\\&+\E_{0}^{(n)}\Pi[L\{b_n(\theta), b_0\}>M_n(\epsilon_n+\gamma_N)|S_n]\\\lesssim&  \left[\E_{0}^{(N,n)}\KL\{\Pi(\cdot\mid b_0),\widehat{Q}_N(\cdot\mid b_0)\}\right]^{1/2}+\sqrt{\tr \E_{0}^{(n)}(S_n-b_0)(S_n-b_0)^\top}\\&+\E_{0}^{(n)}\Pi[L\{b_n(\theta), b_0\}>M_n(\epsilon_n+\gamma_N)|S_n].
	\end{flalign*}
\end{proof}

For a given $\epsilon_n\ge0$ and $\epsilon_n=o(1)$, recall the definition $\Theta_n:=\{\theta\in\Theta:L\{b_n(\theta),b_0\}\le M\epsilon_n\}$. For $A\subseteq\Theta$, let $P(A)=\int_{A}p(\theta)\dt\theta$. The next result gives a posterior concentration result for $\Pi(\Theta_n^c\mid S_n)$. 
\begin{lemma}\label{lem:contmap}
	Let $\epsilon_n=o(1)$ be a positive sequence such that, for $\nu_n$ as in Assumption \ref{ass:obs_sum},   $\nu_n\epsilon_n\rightarrow\infty$. Under Assumptions \ref{ass:obs_sum}-\ref{ass:compat}, for any $M_n\rightarrow\infty$,
	$
	\E_0^{(n)}\Pi\left[L\{b_n(\theta),b_0\}>M_n\epsilon_n\mid S_n\right]\lesssim M_n\epsilon_n .
	$	
\end{lemma}
\begin{proof}[Proof of Lemma \ref{lem:contmap}]
Consider that 	
	\begin{flalign*}
		\Pi(\Theta_{n}^c\mid S_n)
		&\le \frac{\int_{\Theta_{n}^c}g_n(S_n\mid\theta)p(\theta)\dt\theta}{\int_{\Theta_{n}}g_n(S_n\mid\theta)p(\theta)\dt\theta}\\
		&\le\frac{\int_{\Theta_{n}^c}g_n(S_n\mid\theta)p(\theta)\dt\theta}{\int_{\Theta_{n}}g_n(S_n\mid\theta)p(\theta)\dt\theta}\times 1\left\{\int_{\Theta_{n}}g_n(S_n\mid\theta)p(\theta)\dt\theta>P(\Theta_{n})c_{\epsilon_n} f_0^{(n)}(S_n)\right\}\\&+1\left\{\int_{\Theta_{n}}g_n(S_n\mid\theta)p(\theta)\dt\theta\le c_{\epsilon_n}  P(\Theta_{n})f_0^{(n)}(S_n)\right\}
		\\&\le \frac{\int_{\Theta_{n}^c}g_n(S_n\mid\theta)p(\theta)\dt\theta}{P(\Theta_{n})c_{\epsilon_n} f_0^{(n)}(S_n)}+1\left\{\int_{\Theta_{n}}g_n(S_n\mid\theta)p(\theta)\dt\theta\le  P(\Theta_{n})c_{\epsilon_n} f_0^{(n)}(S_n)\right\}.
	\end{flalign*}Recall that $f_0^{(n)}(S)$ is the true density of $S_n$, and take expectations on both sides wrt $S_n$ to obtain
	\begin{flalign*}
		\E_0^{(n)}\Pi(\Theta_{n}^c\mid S_n)\le &	\int_{\mathcal{S}} \frac{\int_{\Theta_{n}^c}g_n(s\mid\theta)p(\theta)\dt\theta}{ P(\Theta_{n})c_{\epsilon_n} f_0^{(n)}(s)}f_0^{(n)}(s)\dt s\\&+P_0^{(n)}\left\{\int_{\Theta_{n}}g_n(S_n\mid\theta)p(\theta)\dt\theta\le  P(\Theta_{n})c_{{\epsilon_n}} f_0^{(n)}(S_n)\right\}\\&
		\le\frac{1}{c_{\epsilon_n} P(\Theta_{n})} \int_{\mathcal{S}} {\int_{\Theta_{n}^c}g_n(s\mid\theta)p(\theta)\dt\theta\dt s}+M_n\epsilon_n,
	\end{flalign*}where the second inequality follows by Lemma \ref{lem:denom}. 
	Consider the first term on the RHS. For $M_n\rightarrow\infty$, and $M_n\lesssim \nu_n$,  let $\mathcal{B}_n(\theta):=\{s\in \mathcal{S}:\|s-b(\theta)\|\le M_n\}$, 
	\begin{flalign}
		\int_{\mathcal{S}} \frac{\int_{\Theta_{n}^c}g_n(s\mid\theta)p(\theta)\dt\theta}{ P(\Theta_{n})c_{\epsilon_n} f_0^{(n)}(s)}f_0^{(n)}(s)\dt s &\le \frac{1}{c_{\epsilon_n} P(\Theta_{n})} \int_{\Theta_{n}^c}\int_{\mathcal{B}_n(\theta)^c}g_n(s\mid\theta)p(\theta)\dt\theta\dt s\nonumber\\&+\frac{1}{c_{\epsilon_n} P(\Theta_{n})} \int_{\Theta_{n}^c}\int_{\mathcal{B}_n(\theta)}g_n(s\mid\theta)p(\theta)\dt\theta\dt s
		\nonumber\\&\le \frac{1}{c_{\epsilon_n} P(\Theta_{n})}\int_{\Theta_{n}^c}G_n\{\mathcal{B}_n(\theta)^c\mid\theta\}p(\theta)\dt\theta\nonumber\\&+\frac{1}{c_{\epsilon_n} P(\Theta_{n})} \int_{\Theta_{n}^c}\int_{\mathcal{B}_n(\theta)}g_n(s\mid\theta)p(\theta)\dt\theta\dt s
		\nonumber\\&\le \frac{1}{c_{\epsilon_n} P(\Theta_{n})}\int_{\Theta_{n}^c}G_n\{\mathcal{B}_n(\theta)^c\mid\theta\}p(\theta)\dt\theta \nonumber\\&+\frac{1}{c_{\epsilon_n} P(\Theta_{n})} \int_{\Theta_{n}^c}\int_{\mathcal{S}}g_n(s\mid\theta)p(\theta)\dt\theta\dt s.\label{eq:last}
	\end{flalign}
	The last term in equation \eqref{eq:last} can be rewritten as 
	\begin{equation}\label{eq:prior_term}
		\frac{1}{c_{\epsilon_n} P(\Theta_{n})} \int_{\Theta_n^c}\int_{\mathcal{S}}g_n(S\mid\theta)p(\theta)\dt\theta\dt S=\frac{1}{c_{\epsilon_n}}\frac{ P(\Theta_{n}^c)}{ P(\Theta_{n})}.
	\end{equation}
	Under Assumption \ref{ass:sim_sum}, the first term in equation \eqref{eq:last} can be upper bounded as
	\begin{flalign}
		\frac{1}{c_{\epsilon_n} P(\Theta_{n})}\int_{\Theta_n^c}G_n\{\mathcal{B}_n(\theta)^c\mid\theta\}p(\theta)\dt\theta\le&\frac{1}{c_{\epsilon_n} P(\Theta_{n})}\int_{\Theta}G_n\{\mathcal{B}_n(\theta)^c\mid\theta\}p(\theta)\dt\theta\nonumber\\\le& \frac{1}{c_{\epsilon_n} P(\Theta_{n})}\frac{1}{(\nu_nM_n)^\alpha}\int_{\Theta}C(\theta)p(\theta)\dt\theta .	\label{eq:post_term}
	\end{flalign}
	Applying equations \eqref{eq:prior_term} and \eqref{eq:post_term} into \eqref{eq:last} yields the upper bound 
	\begin{flalign*}
		\E_0^{(n)} \Pi({\Theta_{n}^c}\mid S_n)\dt\theta&\lesssim \frac{1}{c_{\epsilon_n} P(\Theta_{n})}\frac{1}{(\nu_nM_n)^\alpha}+\frac{1}{c_{\epsilon_n}}\frac{ P(\Theta_{n}^c)}{ P(\Theta_{n})}.
	\end{flalign*}Apply Assumption \ref{ass:prior_concentration}, and the fact that $\epsilon_n\gtrsim 1/\nu_n^{}$ to obtain 
	\begin{flalign*}
		\E_0^{(n)} \Pi({\Theta_{n}^c}\mid S_n)\dt\theta&\lesssim \frac{1}{c_{\epsilon_n} M_n^d\epsilon_n^d}\frac{1}{(\nu_nM_n)^\alpha}+\frac{M_n^{\tau-d}\epsilon_n^{\tau-d}}{c_{\epsilon_n}}\\&\lesssim\frac{1}{c_{\epsilon_n}}\frac{\nu_n^d}{M^d_n}\frac{1}{M_n^\alpha\nu_n^\alpha}+\frac{M_n^{\tau-d}\epsilon_n^{\tau-d}}{c_{\epsilon_n}}\\&\lesssim\frac{1}{c_{\epsilon_n}}(M_n\epsilon_n)^{(\tau-d)\wedge (\alpha-d)},
	\end{flalign*}	where the second line uses the fact that $\epsilon_n\gtrsim 1/\nu_n$ and the last line uses the fact that, by Assumptions \ref{ass:obs_sum}, $\alpha\ge d+1$ and by Assumption \ref{ass:prior_concentration},  $\tau\ge d+1$. It follows that 
	$$
	\E_0^{(n)} \Pi({\Theta_{n}^c}\mid S_n)\lesssim (\epsilon_n/M_n)^{(\tau-d)\wedge (\alpha-d)}\le \epsilon_nM_n.
	$$
\end{proof}
This following result gives an in-probability bound for the denominator of  $\Pi(\cdot\mid S_n)$. 
\begin{lemma}\label{lem:denom}
	Let $M_n>0$, $M_n\rightarrow\infty$ be such that $\epsilon_n=M_n/\nu_n=o(1)$, and recall $\Theta_n=\{\theta\in\Theta: L\{b(\theta),b_0\}\le M_n\epsilon_n\}$. Under Assumption \ref{ass:compat}, 	
	$$
	P^{(n)}_0 \left\{ \int_{\Theta_n}g_n(S_n\mid\theta)p(\theta)\dt\theta\le c_{\epsilon_n} f^{(n)}_0(S_n) P(\Theta_{n})\right\}\le \epsilon_n.
	$$
\end{lemma}
\begin{proof}[Proof of Lemma \ref{lem:denom}]
	For any $\epsilon>0$ let $c_\epsilon>0$, be as defined in Assumption \ref{ass:compat}. For all $S_n\in\mathcal{E}_n$, and for all $\theta\in\Theta_n$, by Assumption \ref{ass:compat},
	$$
	P^{(n)}_0\left\{S_n\in\mathcal{E}_n:g_n(S_n\mid\theta) \ge c_\epsilon f^{(n)}_0(S_n)\right\}\ge 1-\epsilon.
	$$
Hence, for any $\epsilon>0$, 
	\begin{equation}\label{eq:bound1}
		P^{(n)}_0\left\{S_n\in\mathcal{E}_n:\int_{\Theta_n}g_n(S_n\mid\theta)p(\theta)\dt\theta \ge c_\epsilon f^{(n)}_0(S_n)\int_{\Theta_n}p(\theta)\dt\theta\right\}\ge 1-\epsilon.	
	\end{equation}
	Since $\epsilon_n$ is arbitrary, and $c_\epsilon>0$ exists for all $\epsilon>0$, take $\epsilon=\epsilon_n=M_n/\nu_n$, and take the complement of the event in equation \eqref{eq:bound1} to see that 
	$$
	P^{(n)}_0\left\{S_n\in\mathcal{E}_n:\int_{\Theta_n}g_n(S_n\mid\theta)p(\theta)\dt\theta \le c_\epsilon f^{(n)}_0(S_n)\int_{\Theta_n}p(\theta)\dt\theta\right\}\le \epsilon_n.
	$$
\end{proof}
The following result is used in deriving the rate of posterior concentration for the NLE, and is similar to Lemma \ref{lem:denom}.
\begin{lemma}\label{lem:denom_ext}
	Under the conditions in Theorem \ref{thm:two}, 	
	for any $M_n\rightarrow\infty$ such that $\epsilon_nM_n\rightarrow0$, 
	$$
	P^{(n)}_0 \left\{\int_{\Theta_n}\wh q_N(S_n\mid\theta)p(\theta)\dt\theta\le c_{\epsilon_n} f^{(n)}_0(S_n) P(\Theta_{n})(1-\epsilon_n)\right\}\le \epsilon_n.
	$$
\end{lemma}
\begin{proof}[Proof of Lemma \ref{lem:denom_ext}]
	Fix an event $S_n\in\mathcal{E}_n$, and bound $\int_{\Theta_n}\wh q_N(S_n\mid\theta)p(\theta)\dt\theta$ as
	\begin{flalign*}
		\int_{\Theta_n} \wh q_N(S_n\mid\theta)p(\theta)\dt\theta &\ge 	\int_{\Theta_n} g_n(S_n\mid\theta)p(\theta)\dt\theta-\int_\Theta |g_n(S_n\mid\theta)-\wh q_N(S_n\mid\theta)|p(\theta)\dt\theta\\&\ge\int_{\Theta_n} g_n(S_n\mid\theta)p(\theta)\dt\theta -\sqrt{\KL\{g_n(S_n\mid\theta)p(\theta)\|\wh q_N(S_n\mid\theta)p(\theta)\}} .
	\end{flalign*}From the choice of $N$ in Theorem \ref{thm:two}, $\gamma_N\lesssim\epsilon_n$. With probability at least $1-\epsilon_n$, by Assumption \ref{ass:sieve3},
	\begin{flalign*}
		\int_{\Theta_n} \wh q_N(S_n\mid\theta)p(\theta)\dt\theta &\ge\int_{\Theta_n} g_n(S_n\mid\theta)p(\theta)\dt\theta-\epsilon_n .
	\end{flalign*}Following the arguments in equation \eqref{eq:bound1} of Lemma \ref{lem:denom}, it follows that with probability at least $1-\epsilon_n$, for $n$ large enough,
	$$
	\int_{\Theta_n} \wh q_N(S_n\mid\theta)p(\theta)\dt\theta \ge \int_{\Theta_n} g_n(S_n\mid\theta)p(\theta)\dt\theta-\epsilon_n\ge P(\Theta_n)f_0^{(n)}(S_n)(1-\epsilon_n).
	$$Taking the probability of the compliment of the above event yields the stated result.

\end{proof}

\subsection{Proofs of Main results}\label{app:results}

\begin{proof}[Proof of Theorem \ref{thm:main}]
	For $\epsilon_n=o(1)$ positive and such that $\nu_n\epsilon_n\rightarrow\infty$. By Lemma \ref{lem:general1_TV}, for some $M_n>0$, 
	\begin{flalign}
		\E_{0}^{(N,n)} \wh Q_N[L\{b(\theta),b_0\}>M_n(\epsilon_n+\gamma_N)\mid S_n]&\le \sqrt{\E_{0}^{(N,n)}\KL\{\Pi(\cdot\mid b_0),\wh Q_N(\cdot\mid b_0)\}}\nonumber\\&+\sqrt{\tr \E_{0}^{(n)}\{(S_n-b_0)(S_n-b_0)^\top\}}\nonumber\\&+  \E_{0}^{(n)}\Pi[L\{b_n(\theta),b_0\}>M_n(\epsilon_n+\gamma_N)|S_n].\label{eq:new0}
	\end{flalign}
	From  Lemma \ref{lem:contmap}, 
	\begin{equation}\label{eq:new1}
		\E_{0}^{(n)}\Pi[L\{b(\theta),b_0\}>M_n(\epsilon_n+\gamma_N)|S_n]\le M_n(\epsilon_n+\gamma_N),
	\end{equation}while by Assumption \ref{ass:sieve}
	\begin{equation}\label{eq:new2}
		\E_{0}^{(N,n)}\KL\{\Pi(\cdot\mid b_0),\widehat{Q}_N(\cdot\mid b_0)\}%\le \E_{0}^{(n)}\inf_{Q\in\mathcal{Q}}\KL\{\Pi(\cdot\mid b_0),Q(\cdot\mid b_0)\}
		\le \gamma_N^2.	
	\end{equation}
	By Assumption \ref{ass:obs_sum},
	\begin{equation}\label{eq:new3}
		\tr [\E_{0}^{(n)}\{(S_n-b_0)(S_n-b_0)^\top\}]\le d_s\cdot\lambda_{\max}(V)/\nu^2_n.
	\end{equation}Applying equations \eqref{eq:new1}-\eqref{eq:new3} into equation \eqref{eq:new0}, for $M_n>0$ large enough, 
	\begin{flalign*}
		\E_{0}^{(N,n)} \wh Q_N[L\{b(\theta),b_0\}>M_n(\epsilon_n+\gamma_N)\mid S_n]&\le M_n\epsilon_n +M_n\gamma_N+d\cdot\lambda_{\max}(V)/\nu_n\\&\lesssim M_n(\epsilon_n+\gamma_N).
	\end{flalign*} 
\end{proof}

The proofs of Theorem \ref{thm:main} and Lemma \ref{lem:contmap} suggest that Theorem \ref{thm:main} will remain true if we replace Assumption \ref{ass:sieve} with the following alternative assumption. 
\begin{assumption}\label{ass:sieve_alt}For some $\gamma_{N}=o(1)$,
	$
\E_{0}^{(N,n)}\dt_{\mathrm{H}}\{\Pi(\cdot\mid b_0),\wh Q_N(\cdot\mid b_0)\}\le \gamma_N .
	$
\end{assumption}
\begin{corollary}\label{corr:main}
	Let $\epsilon_n=o(1)$ be a positive sequence such that $\nu_n\epsilon_n\rightarrow\infty$.  Under Assumptions \ref{ass:obs_sum}-\ref{ass:prior_concentration}, \ref{ass:lipz}, and Assumption \ref{ass:sieve_alt}, for any positive sequence $M_n\rightarrow\infty$, 
	$$
	\E^{(N,n)}_0\wh Q_N\left[ L\{b_{n}(\theta),b_0\}>M_n (\epsilon_n+\gamma_{N})\mid S_n\right]=o(1).
	$$for any loss function such that $:L:\mathcal{S}\times\mathcal{S}\rightarrow\mathbb{R}_+$.	
\end{corollary}
\begin{proof}[Proof of Corollary \ref{corr:main}]
	From the proof of Lemma \ref{lem:contmap}, we have the upper bound
	\begin{flalign*}
		\wh Q_N\left [L\{b_n(\theta), b_0\}>M_n(\epsilon_n+\gamma_N)\mid S_n\right]&\le \dt_{\mathrm{TV}}\{\Pi(\cdot\mid S_n),\widehat{Q}_N(\cdot\mid S_n)\}\\&+\Pi[L\{b_n(\theta), b_0\}>M_n(\epsilon_n+\gamma_N)|S_n].
	\end{flalign*}
	Bound the first term as
	\begin{flalign*}
		\dt_{\mathrm{TV}}\{\Pi(\cdot\mid S_n),\widehat{Q}_N\}\le& 		\dt_{\mathrm{TV}}\{\Pi(\cdot\mid S_n),\Pi(\cdot\mid b_0)\}+\dt_{\mathrm{TV}}\{\widehat{Q}_N(\cdot\mid S_n),\wh Q_N(\cdot\mid b_0)\}\\+&	\dt_{\mathrm{TV}}\{\Pi(\cdot\mid b_0),\wh Q_N(\cdot\mid b_0)\}\\\le& 	2C\|S_n-b_0\|+\dt_{\mathrm{TV}}\{\Pi(\cdot\mid b_0),\wh Q_N(\cdot\mid b_0)\},
	\end{flalign*}where the second inequality uses Assumption \ref{ass:lipz}. Hence,
	\begin{flalign*}
		\wh Q_N\left [L\{b_n(\theta), b_0\}>M_n(\epsilon_n+\gamma_N)\mid S_n\right]\le &2C\|S_n-b_0\|+\dt_{\mathrm{TV}}\{\Pi(\cdot\mid b_0),\wh Q_N(\cdot\mid b_0)\}\\&+\Pi[L\{b_n(\theta), b_0\}>M_n(\epsilon_n+\gamma_N)|S_n].
	\end{flalign*}Applying Lemma \ref{lem:pinsk} to $\dt_{\mathrm{TV}}\{\Pi(\cdot\mid b_0),\wh Q_N(\cdot\mid b_0)\}$ yields
	\begin{flalign*}
		\wh Q_N\left [L\{b_n(\theta), b_0\}>M_n(\epsilon_n+\gamma_N)\mid S_n\right]\le &2C\|S_n-b_0\|+2\dt_{\mathrm{H}}\{\Pi(\cdot\mid b_0),\wh Q_N(\cdot\mid b_0)\}\\&+\Pi[L\{b_n(\theta), b_0\}>M_n(\epsilon_n+\gamma_N)|S_n].
	\end{flalign*}
	The remainder of the proof follows that of Theorem \ref{thm:main}.
\end{proof}

\begin{proof}[Proof of Theorem \ref{thm:bvm}]By the triangle inequality, 
	\begin{flalign*}
		\int |\wh q_N(t\mid S_n)-N\{t;0,\Sigma^{-1}\}|\dt t	&\le\int |\pi(t\mid S_n)-N\{t;0,\Sigma^{-1}\}|\dt t+\int |\wh q_N(t\mid S_n)-\pi(t\mid S_n)|\dt t,
	\end{flalign*}The first term is $o_p(1)$ under the stated assumptions of the result. Recall that $\dt_{\mathrm{TV}}(\cdot,\cdot)$ is invariant to scale and location transformations. Using this fact and  the triangle inequality we obtain
	\begin{flalign*}
		\int |\wh q_N(t\mid S_n)-\pi(t\mid S_n)|\dt t&=\dt_{\mathrm{TV}}\{\Pi(\cdot\mid S_n),\widehat{Q}_N(\cdot\mid S_n)\}\\&\le 		\dt_{\mathrm{TV}}\{\Pi(\cdot\mid S_n),\Pi(\cdot\mid b_0)\}+\dt_{\mathrm{TV}}\{\widehat{Q}_N(\cdot\mid S_n),\wh Q_N(\cdot\mid b_0)\}\\&+	\dt_{\mathrm{TV}}\{\Pi(\cdot\mid b_0),\wh Q_N(\cdot\mid b_0)\}.
	\end{flalign*}Apply Assumption \ref{ass:lipz} twice, and Lemma \ref{lem:pinsk} to obtain 
	\begin{flalign*}
		\dt_{\mathrm{TV}}\{\Pi(\cdot\mid S_n),\widehat{Q}_N\}\le& 	2C\|S_n-b_0\|+\dt_{\mathrm{TV}}\{\Pi(\cdot\mid b_0),\wh Q_N(\cdot\mid b_0)\}\\\le& 2C\|S_n-b_0\|+\sqrt{\KL\{\Pi(\cdot\mid b_0)\|\wh Q_N(\cdot\mid b_0)\}}. 
	\end{flalign*}Fix $\epsilon>0$. By Assumption \ref{ass:obs_sum}, Assumption \ref{ass:sieve}, and Markov's inequality, we obtain
	\begin{flalign*}
		P^{(n)}_0\left\{	\dt_{\mathrm{TV}}\{\Pi(\cdot\mid S_n),\widehat{Q}_N\}\ge \epsilon\right\}\le& \frac{4\E_{0}^{(n)}C\|S_n-b_0\|}{\epsilon}+\frac{\sqrt{\E^{(N)}_{p}\KL\{\Pi(\cdot\mid b_0)\|\wh Q_N(\cdot\mid b_0)\}}}{\epsilon}\\&\lesssim \frac{1}{\nu_n}\frac{1}{\epsilon} +\frac{\gamma_N}{\epsilon}	\\&\le\frac{2}{\nu_n\epsilon},
	\end{flalign*}where the last line follows since, by the assumptions of the result, $\gamma_N\lesssim\nu_n^{-1}$ so that $\gamma_N\lesssim \epsilon_n$. Hence, the result follows for any $\epsilon\ge\epsilon_n$ such that $\epsilon\nu_n\rightarrow\infty$. 
\end{proof}

\begin{proof}[Proof of Theorem \ref{thm:two}]
	The proof follows a similar argument to Lemma \ref{lem:contmap}. In particular, we start from the upper bound
	\begin{flalign*}
		\wh \Pi(\Theta_{n}^c\mid S_n)
		&\le \frac{\int_{\Theta_{n}^c}\wh q_N(S_n\mid\theta)p(\theta)\dt\theta}{\int_{\Theta_{n}}\wh q_N(S_n\mid\theta)p(\theta)\dt\theta}\\
		&\le\frac{\int_{\Theta_{n}^c}\wh q_N(S_n\mid\theta)p(\theta)\dt\theta}{\int_{\Theta_{n}}\wh q_N(S_n\mid\theta)p(\theta)\dt\theta} \times1\left\{\int_{\Theta_{n}}\wh q_N(S_n\mid\theta)p(\theta)\dt\theta> P(\Theta_{n})c_{\epsilon_n} f_0^{(n)}(S_n)(1-\epsilon_n)\right\}\\&+1\left\{\int_{\Theta_{n}}\wh q_N(S_n\mid\theta)p(\theta)\dt\theta\le c_{\epsilon_n}  P(\Theta_{n})f_0^{(n)}(S_n)(1-\epsilon_n)\right\}
		\\&\le \frac{\int_{\Theta_{n}^c}\wh q_N(S_n\mid\theta)p(\theta)\dt\theta}{ P(\Theta_{n})c_{\epsilon_n} f_0^{(n)}(S_n)(1-\epsilon_n)}+1\left\{\int_{\Theta_{n}}\wh q_N(S_n\mid\theta)p(\theta)\dt\theta\le  P(\Theta_{n})c_{\epsilon_n} f_0^{(n)}(S_n)(1-\epsilon_n)\right\}.
	\end{flalign*}Consider the first term. Recall that $f_0^{(n)}(S)$ is the true density of $S_n$, and take expectations on both sides to obtain
	\begin{flalign*}
		\E_0^{(n)}\wh \Pi(\Theta_{n}^c\mid S_n)&\le	\int_{\mathcal{S}} \frac{\int_{\Theta_{n}^c}\wh q_N(s\mid\theta)p(\theta)\dt\theta}{ P(\Theta_{n})c_{\epsilon_n} f_0^{(n)}(s)(1-\epsilon_n)}f_0^{(n)}(s)\dt s\\&+P_0^{(n)}\left\{\int_{\Theta_n}\wh q_N(S_n\mid\theta)p(\theta)\dt\theta\le  P(\Theta_n)c_{\epsilon_n} f_0^{(n)}(S_n)(1-\epsilon_n)\right\}.%\\&
		%	\le\frac{1}{c_{\epsilon_n} P(\Theta_{n})} \int_{\mathcal{S}} {\int_{\Theta_{n}^c}g_n(S\mid\theta)p(\theta)\dt\theta\dt S}+M_n\epsilon_n
	\end{flalign*}
	Analyzing the first term we see that, under Assumption \ref{ass:prior_concentration},  
	\begin{flalign*}
		\int_{\mathcal{S}} \frac{\int_{\Theta_{n}^c}\wh q_N(s\mid\theta)p(\theta)\dt\theta}{ P(\Theta_{n})c_{\epsilon_n} f_0^{(n)}(s)(1-\epsilon_n)}f_0^{(n)}(s)\dt s&=\frac{1}{c_{\epsilon_n}(1-\epsilon_n)}\frac{P(\Theta_n)}{P(\Theta_n^c)}\\&\asymp M_n^{(\tau-d)}\epsilon_n^{(\tau-d)}\\&\lesssim M_n\epsilon_n.
	\end{flalign*}
	Applying  Lemma \ref{lem:denom_ext} to the second term then yields the stated result. 
\end{proof}

\begin{proof}[Proof of Theorem \ref{thm:bvm-snl}]
	By the triangle inequality, 	
	\begin{flalign*}
		\int |\wh \pi_N(t\mid S_n)-N\{t;0,\Sigma^{-1}\}|\dt t	&\le \int |\pi(t\mid S_n)-N\{t;0,\Sigma^{-1}\}|\dt t+\int |\wh \pi_N(t\mid S_n)-\pi(t\mid S_n)|\dt t. 
	\end{flalign*}
	Let $Z_{n,t}=\int_\Theta g_n(S_n\mid \theta_n+t/\nu_n)\pi(\theta_n+t/\nu_n)\dt t$. To see that the first term is $o_p(1)$, note that the maintained assumption in the result implies that 
	$$
	Z_{n,t}=(2\pi)^{d_\theta/2}p(\theta_n)|\Sigma|^{-1/2}+o_p(1).
	$$Hence, writing 
	\begin{flalign*}
		\int |\pi(t\mid S_n)-N\{t;0,\Sigma^{-1}\}|\dt t=\int \left|\frac{g_n(S_n\mid\theta_n+t/\nu_{n})p(\theta_n+t/\nu_{n})}{Z_{n,t}}-\frac{\exp\{-t^\top \Sigma t\}}{(2\pi)^{d/2}|\Sigma|^{-1/2}}\right|\dt t,
	\end{flalign*}since $Z_{n,t}= (2\pi)^{d_\theta/2}p(\theta_n)|\Sigma|^{-1/2}+o_p(1)$. The first term is then $o_p(1)$ if 
	$$
	\int \left|{g_n(S_n\mid\theta_n+t/\nu_{n})}\frac{p(\theta_n+t/\nu_n)}{p(\theta_n)}-{\exp(-t^\top \Sigma t/2)}\right|\dt t=o_p(1),
	$$which is precisely the maintained condition in the stated result.

	Now, focus on the second term and write
	\begin{flalign*}
		&\widehat \pi_N( \theta_n+t/\nu_{n} \mid S_n)-\pi( \theta_n+t/\nu_{n} \mid S_n)\\&=\frac{\wh{q}_{N}(S_n\mid \theta_n+t/\nu_{n} )p( \theta_n+t/\nu_{n} )}{\int_\Theta \wh{q}_{N}(S_n\mid \theta_n+t/\nu_{n} )p( \theta_n+t/\nu_{n} )\dt t}-\frac{g_n(S_n\mid \theta_n+t/\nu_{n} )p( \theta_n+t/\nu_{n} )}{\int_\Theta g_n(S_n\mid \theta_n+t/\nu_{n} )p( \theta_n+t/\nu_{n} )\dt t}\\&=\frac{\left\{\wh{q}_{N}(S_n\mid  \theta_n+t/\nu_{n} )-g_n(S_n\mid \theta_n+t/\nu_{n} )\right\}p( \theta_n+t/\nu_{n} )}{Z_{N,t}}\\&-g_n(S_n\mid \theta_n+t/\nu_{n} )p( \theta_n+t/\nu_{n} )\left(\frac{1}{Z_{n,t}}-\frac{1}{Z_{N,t}}\right),
	\end{flalign*}where $Z_{N,t}=\int_\Theta \wh{q}_N(S_n\mid \theta_n+t/\nu_n)\pi(\theta_n+t/\nu_n)$. Hence,
	\begin{flalign}
		&\int|\wh \pi_N(\theta_n+t/\nu_{n}\mid S_n)-\pi(\theta_n+t/\nu_{n}\mid S_n)|\dt t\nonumber\\&\le  	\int \frac{\left |\wh{q}_{N}(S_n\mid \theta_n+t/\nu_{n})-g_n(S_n\mid\theta_n+t/\nu_{n})\right| p(\theta_n+t/\nu_{n})}{Z_{N,t}}\dt t+\left|1-\frac{Z_{n,t}}{Z_{N,t}}\right|\nonumber\\&\le 
		2\int_\Theta \frac{\left |\wh{q}_{N}(S_n\mid \theta_n+t/\nu_{n})-g_n(S_n\mid \theta_n+t/\nu_{n})\right| p(\theta_t+t/\nu_{n})}{Z_{N,t}}\dt t\label{eq:new}.
	\end{flalign}
	
	By Lemma \ref{lem:pinsk} and Assumption \ref{ass:sieve3}, 
	\begin{flalign}
		\left|\int_\Theta g_n(S_n\mid\theta)p(\theta)\dt\theta-\int_\Theta \wh q_N(S_n\mid\theta)p(\theta)\dt\theta\right|&\le \int_\Theta |\wh q_N(S_n\mid\theta)-g_n(S_n\mid\theta)|p(\theta)\dt\theta\nonumber\\&\le \sqrt{\KL\{g_n(S_n\mid\theta)p(\theta)\|\wh q_N(S_n\mid\theta)p(\theta)\}}\nonumber\\&\le \gamma_N\nonumber\\&\lesssim \epsilon_n\label{eq:neweqs}
	\end{flalign}where the second to last inequality holds with probability at least $1-\epsilon_n$ under Assumption \ref{ass:sieve3}, and the last holds since $\gamma_N\lesssim \epsilon_n$ by hypothesis. Since $\dt_{\mathrm{TV}}(\cdot,\cdot)$ and $\KL(\cdot\|\cdot)$ are invariant to affine transformations,  applying \eqref{eq:neweqs} to the above implies that, as $n,N\rightarrow\infty$, 
	$
	|Z_{n,t}-Z_{N,t}|\le \epsilon_n
	$ with probability at least $1-\epsilon_n$, and we can conclude that, as $n\rightarrow\infty$,  
	$$
	Z_{N,t}=(2\pi)^{d_\theta/2}p(\theta_n)|\Sigma|^{-1/2}+o_p(1).
	$$ Applying the above into equation \eqref{eq:new}, and again using the invariance of $\KL(\cdot\|\cdot)$
	\begin{flalign*}
		&\int|\wh \pi_N(\theta_n+t/\nu_{n}\mid S_n)-\pi(\theta_n+t/\nu_{n}\mid S_n)|\dt t\nonumber\\&\le  2\frac{[\KL\{g_n(S_n\mid\theta_n+t/\nu_{n}) p(\theta_n+t/\nu_{n})\|\wh q_N(S_n\mid\theta_n+t/\nu_{n})p(\theta_n+t/\nu_{n})\}]^{1/2}}{Z_{n,t}}\\&=2\frac{[\KL\{g_n(S_n\mid\theta) p(\theta)\|\wh q_N(S_n\mid\theta)p(\theta)\}]^{1/2}}{[(2\pi)^{d_\theta/2}p(\theta_n)|\Sigma|^{-1/2}\{1+o_p(1)\}]}\\&\le \frac{2\gamma_N}{\{(2\pi)^{d_\theta/2}p(\theta_n)|\Sigma|^{-1/2}\}}\{1+o_p(1)\}\\&\le \frac{2\epsilon_n}{\{(2\pi)^{d_\theta/2}p(\theta_n)|\Sigma|^{-1/2}\}}\{1+o_p(1)\}.
	\end{flalign*}	
\end{proof}

\end{document}